%% file: main.tex
\documentclass{article}

\usepackage{iclr2027_conference,times}
\iclrfinalcopy

\input{math_commands.tex}

\usepackage{amsmath}
\usepackage{amssymb}
\usepackage{amsthm}
\usepackage{booktabs}
\usepackage{graphicx}
\usepackage{float}
\usepackage{afterpage}
\usepackage{microtype}
\usepackage{multirow}
\usepackage{pifont}
\usepackage{tabularx}
\usepackage{tikz}
\usepackage{xcolor}
\usepackage{colortbl}
\usepackage{xspace}
\usepackage{url}
\usepackage{placeins}
\usepackage{etoolbox}
\usepackage{hyperref}
\hypersetup{hidelinks,
  pdftitle={Waggle: Learning One Anonymous Local Law for Self-Organizing LLM Swarms},
  pdfauthor={Mingxi Zou, Wei Zhu, Zhuo Wang, Langzhang Liang, Zhiwen Tang, Yinghui Xu, Zenglin Xu}}

\AtBeginEnvironment{table}{\setlength{\belowcaptionskip}{4pt}}
\AtBeginEnvironment{table*}{\setlength{\belowcaptionskip}{4pt}}
\makeatletter
\renewcommand\paragraph{\@startsection{paragraph}{4}{\z@}%
  {1.05ex plus .3ex minus .2ex}{-1em}{\normalsize\bf}}
\makeatother

\newcommand{\method}{\textup{Waggle}\xspace}
\newcommand{\scd}{\textsc{SCD}\xspace}
\newcommand{\silo}{\textsc{Silo-Bench}\xspace}
\newcommand{\swarmbench}{\textsc{SwarmBench}\xspace}
\newcommand{\nativemsg}{\textsc{NativeMsg}\xspace}
\newcommand{\fixedswarm}{\textsc{FixedSwarm}\xspace}
\newcommand{\nocomm}{\textsc{NoComm}\xspace}
\newcommand{\gptswarm}{\textsc{GPTSwarm}\xspace}
\newcommand{\agentnet}{\textsc{AgentNet}\xspace}
\newcommand{\agentdistill}{\textsc{AgentDistill}\xspace}
\newcommand{\gdesigner}{\textsc{G-Designer}\xspace}
\newcommand{\argdesigner}{\textup{ARG-Designer}\xspace}
\newcommand{\dmoa}{\textup{DMoA}\xspace}
\newcommand{\agentsnetrp}{\textsc{AgentsNet-RP}\xspace}
\newcommand{\swarmsys}{\textsc{SwarmSys}\xspace}

\newcommand{\tracevec}{\boldsymbol{\phi}}
\newcommand{\view}{\ell}

\definecolor{rankfirstbg}{RGB}{226,224,238}
\definecolor{ranksecondbg}{RGB}{250,234,218}
\definecolor{tableheaderbg}{RGB}{243,246,249}
\definecolor{tablerefbg}{RGB}{244,244,244}
\newcommand{\bestcell}[1]{\cellcolor{rankfirstbg}\bfseries\boldmath #1}
\newcommand{\secondcell}[1]{\cellcolor{ranksecondbg}\underline{#1}}
\newcommand{\bestnum}[1]{\begingroup\setlength{\fboxsep}{1pt}\colorbox{rankfirstbg}{\strut\bfseries #1}\endgroup}
\newcommand{\secondnum}[1]{\begingroup\setlength{\fboxsep}{1pt}\colorbox{ranksecondbg}{\strut\underline{#1}}\endgroup}
\newcommand{\bestsrmsg}[2]{\cellcolor{rankfirstbg}\textbf{#1}/#2}
\newcommand{\secondsrmsg}[2]{\cellcolor{ranksecondbg}\underline{#1}/#2}
\newtheorem{proposition}{Proposition}
\newtheorem{corollary}{Corollary}

\title{\method: Learning One Anonymous Local Law for Self-Organizing LLM Swarms}

\author{%
\textbf{Mingxi Zou}$^{1,2}$ \quad
\textbf{Wei Zhu}$^{1,3}$ \quad
\textbf{Zhuo Wang}$^{2}$ \quad
\textbf{Langzhang Liang}$^{1,2}$ \\
\textbf{Zhiwen Tang}$^{3}$ \quad
\textbf{Yinghui Xu}$^{2,\ast}$ \quad
\textbf{Zenglin Xu}$^{1,2,\ast}$ \\
$^{1}$Shanghai Academy of AI for Science (SAIS) \\
$^{2}$AI$^{3}$ Institute, Fudan University \quad
$^{3}$Yunnan University \\
\texttt{mxzou24@m.fudan.edu.cn} \\
$^{\ast}$Corresponding authors
}

\begin{document}

\maketitle
\lhead{Preprint}

\begin{abstract}
As LLM agents increasingly collaborate on complex tasks, how to organize their
interactions becomes a central design question. Existing multi-agent systems
typically learn or adapt explicit roles, hierarchies, routing policies, or
communication topologies. We shift the learning target to a reusable local law
that can be shared across interchangeable agents and adapt coordination as
populations or interaction conditions change, without redefining a global
organization.
We introduce \textbf{\method}, a shared anonymous policy over bounded local
views that jointly selects task actions, semantic communication, and local
commitment updates. Repeated execution of the same law allows coordination to
form, persist, and reorganize online without explicit roles or global topology.
To learn this law across interchangeable agents and evolving coordination,
we develop \textbf{Swarm-Consistent Distillation (\scd)}, combining
anonymous-orbit consistency with rollout-grounded prediction of the next local
coordination field, with no added inference-time components.
Across diverse coordination settings, the same learned law remains effective
as populations and interaction budgets change, retains over 96\% of
substrate-specific oracle quality, and transfers without retraining; \scd
further improves reorganization after counterevidence. Together, these results
show that LLM-agent organization can emerge and adapt through repeated
execution of a learned local law.
\end{abstract}

\section{Introduction}
\label{sec:intro}
\setlength{\parskip}{4.5pt plus .5pt minus .5pt}

LLM agents can reason, use tools, and revise plans from feedback, enabling
flexible collaborative systems
\citep{yao2023react}.  Multi-agent systems organize these capabilities through
roles, managers, routes, and conversation graphs
\citep{li2023camel,hong2024metagpt,qian2024chatdev,wu2023autogen,
zhuge2024language}.
Recent methods further design task-dependent communication graphs, jointly
generate agent compositions and topologies, or adapt connectivity and agent
activation online
\citep{zhang2025gdesigner,li2026assemblecrew,yang2025agentnet,wu2026dmoa}.
Despite this flexibility, the learned object remains an explicit
organization---the graph, routing policy, or agent composition.

Swarm intelligence suggests a different organizing principle: persistent yet
revisable population-level coordination can emerge from repeated local
reactions to local signals, without requiring a population-wide plan
\citep{bonabeau1999swarm,theraulaz1999stigmergy,seeley2012stop}. Related
decentralized and swarm-inspired learning systems provide computational
precedents for this local-rule view
\citep{huettenrauch2019deep,shaw2022formic,yang2025agentnet,li2025swarmsys}.
Inspired by this local-rule view, we shift the learning target from explicit
organization to the reusable local law that generates and revises it. This
law can be shared across interchangeable agents and reused as populations or
interaction conditions change.

Yet effective local laws require more than individually reasonable actions.
Proposals based on private observations can become
redundant, conflicting, or stale when composed across a population. This
failure is pronounced in \silo: at $N=64$, transient-exchange and fixed-rule
baselines reach only about 15\% success, compared with 58.2\% for our method
(Table~\ref{tab:g3-cross-substrate-core}), and additional interaction alone
does not close the gap (Figure~\ref{fig:horizon-budget-sweep}). The bottleneck
is therefore not merely local competence or message availability, but how a
shared local rule composes proposals into evolving collective state.

We therefore treat proposals as inputs to coordination and learn the shared
policy that determines each local decision. This leads to our central
question: \textbf{can a shared anonymous local law generate, maintain, and
revise organization online through repeated execution?}
Figure~\ref{fig:organization-vs-organizing-law} contrasts learning an explicit
organization with learning the local law that generates it.

\begin{figure*}[t!]
\centering
\includegraphics[width=\linewidth]{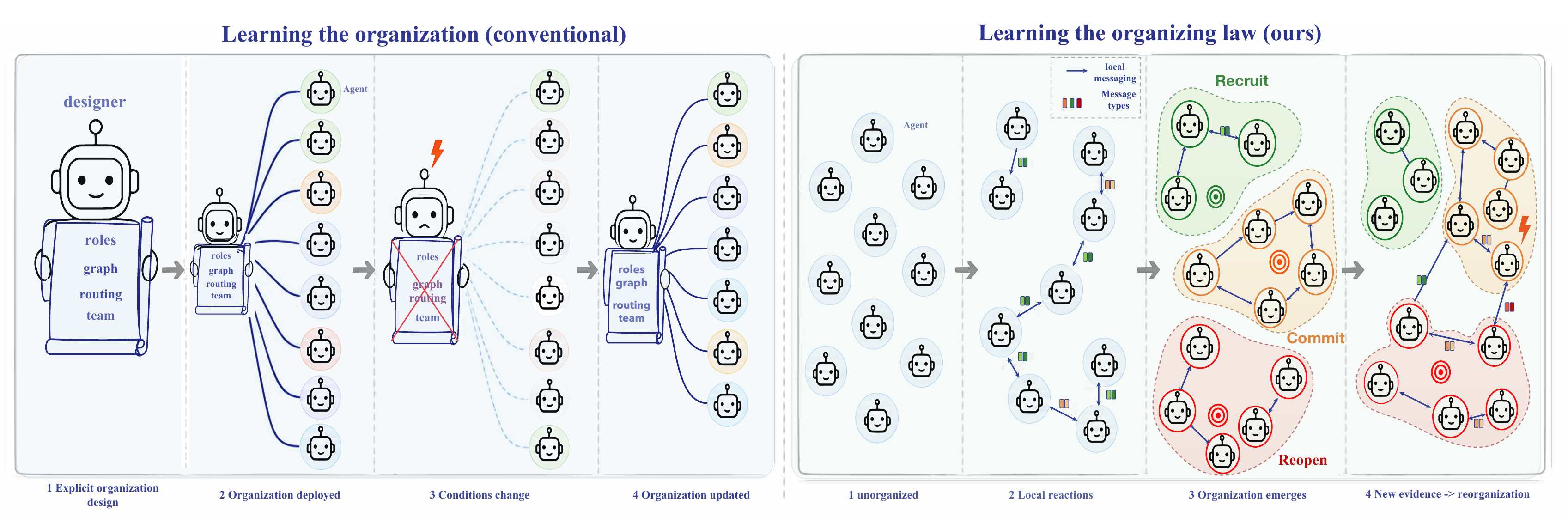}
\caption{\textbf{Learning the organization versus the organizing law.}}
\label{fig:organization-vs-organizing-law}
\end{figure*}

We introduce \method, which instantiates this learning object as one anonymous,
parameter-shared policy over local views.  Each local decision jointly chooses
a task action and a bounded semantic update that can create or revise
persistent local coordination state.
We train the law with \emph{Swarm-Consistent Distillation} (\scd). Beyond
imitating teacher decisions, \scd encourages the policy to remain semantically
consistent when opaque identifiers or neighbor orderings change, while learning
representations that predict how the local coordination state evolves from one
update to the next.  Figure~\ref{fig:overview}
summarizes the resulting training-to-deployment pipeline.  This consistency
carries over to the swarm level: if local views and runtime updates treat
relabelings consistently, renaming agents or local references changes only
the labels, not the resulting coordination behavior.  Across spatial and
distributed-language coordination settings, the
same learned law remains effective as population size and interaction budget
change and adapts more reliably when later evidence contradicts an established
commitment.

\setlength{\parskip}{3.5pt plus .5pt minus .5pt}
Our contributions are:\par
\begin{list}{\labelitemi}{%
    \setlength{\leftmargin}{2em}%
    \setlength{\labelwidth}{0.55em}%
    \setlength{\labelsep}{0.35em}%
    \setlength{\itemindent}{0pt}%
    \setlength{\listparindent}{0pt}%
    \setlength{\topsep}{0pt}%
    \setlength{\itemsep}{3pt}%
    \setlength{\parsep}{0pt}%
    \setlength{\partopsep}{0pt}%
}
    \item \textbf{Conceptually,} we reformulate the learning target of
    LLM-agent organization from explicit organizational structures to the
    local law that generates them: \method learns one shared anonymous local
    reaction law whose repeated execution generates and revises organization
    online.
    \item \textbf{Methodologically,} we develop \emph{Swarm-Consistent
    Distillation} (\scd), which trains this law toward consistent decisions
    across anonymous local views and representations predictive of evolving
    coordination state, with no additional inference-time machinery.
    \item \textbf{Empirically,} we show that one learned law supports spatial
    and language-based coordination, remains effective as swarm conditions
    change, and reorganizes after counterevidence.
\end{list}

\afterpage{%
\begin{figure*}[t!]
\centering
\includegraphics[width=\linewidth]{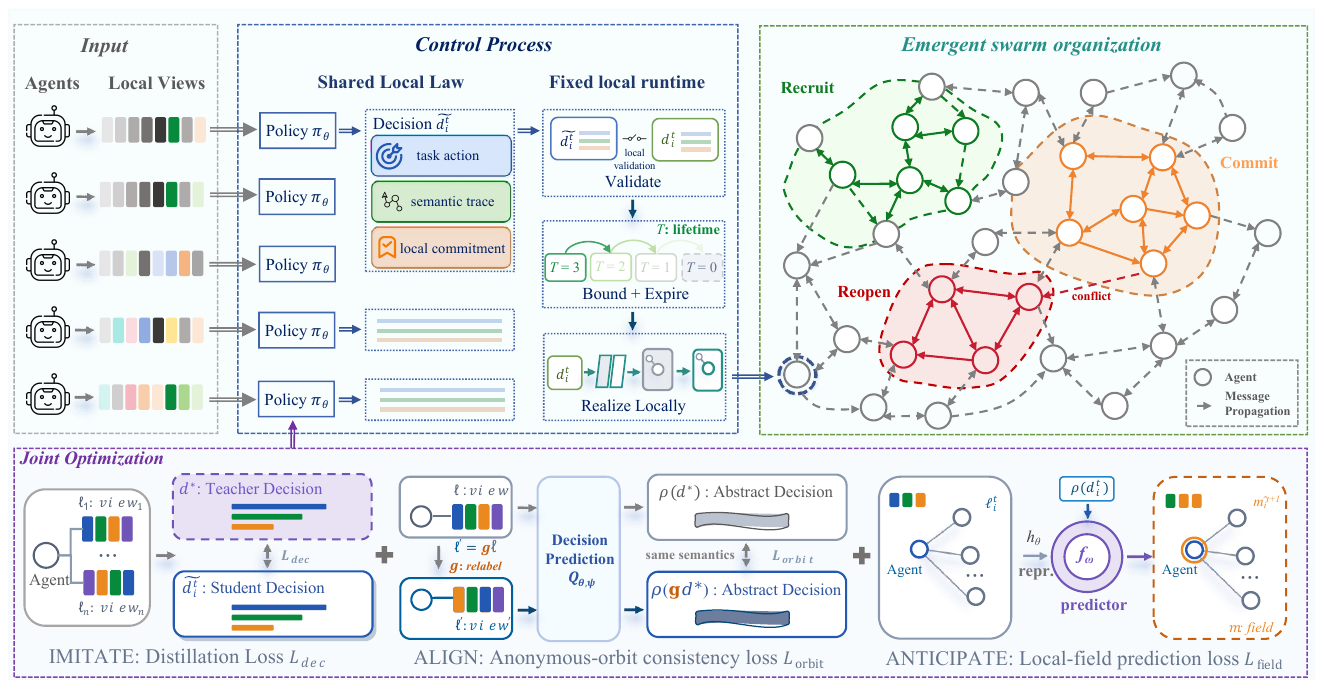}
\caption{\textbf{Overview of \method.}  \scd trains one shared anonymous
local law; each agent deploys it independently through a fixed runtime, and
repeated local reactions form and revise swarm organization.}
\label{fig:overview}
\end{figure*}
}

\section{Related Work}
\label{sec:related}
\setlength{\parskip}{2.5pt plus .5pt minus .5pt}

\paragraph{LLM multi-agent organization.}
LLM multi-agent systems encode organization through roles, identities,
protocols, or programmable interaction patterns
\citep{li2023camel,hong2024metagpt,qian2024chatdev,wu2023autogen}.
DyLAN and AgentVerse adapt team composition and collaboration
\citep{liu2024dynamic,chen2024agentverse}.
Other methods optimize workflows or communication topology, or jointly
generate task-conditioned agent compositions and graphs, as in \argdesigner{}
\citep{zhang2025aflow,zhuge2024language,
zhang2025agentprune,zhang2025gdesigner,li2026assemblecrew}.
Others route or activate specialists through adaptive connectivity,
publish--subscribe structures, or specialized agent classes
\citep{wu2026dmoa,yang2025agentnet,li2026raps,li2025swarmsys}.
LLawCo derives natural-language cooperation principles from failures for
embodied planners, while multi-agent distillation compresses interaction
graphs or dynamics into a student
\citep{zhou2026llawco,chen2024magdi,luo2026agentark}.
Agent Distillation trains smaller agents on teacher-generated interaction
trajectories \citep{kang2025distilling}.
Across these approaches, organization remains represented, adapted, or
compressed through system-level structures or agents; \method instead learns
one shared anonymous local law whose repeated execution generates and revises
organization online.

\paragraph{Learned communication and swarm coordination.}
Learned-communication methods jointly optimize control with message content,
broadcast gates, or recipient selection
\citep{sukhbaatar2016commnet,foerster2016dial,singh2019ic3net,das2019tarmac},
while shared-policy swarm RL uses interchangeable controllers with local,
variable-size observations \citep{huettenrauch2019deep}.
LLM2Swarm explores LLM-based controller synthesis and per-robot reasoning,
whereas GenSwarm generates and deploys task-specific multi-robot code policies
from natural-language instructions
\citep{strobel2024llm2swarm,ji2026genswarm}.
Population protocols repeatedly apply a shared local transition rule among
anonymous finite-state agents, whereas stigmergic systems coordinate through
persistent local state, from ant-colony reinforcement and evaporation to
learned deposition and sensing in ForMIC
\citep{angluin2006computation,theraulaz1999stigmergy,
dorigo1997ant,shaw2022formic}.
Building on shared control, learned communication, local rules, and persistent
state, \method learns the organization-generating local law itself.

\paragraph{Symmetry-consistent learning.}
Anonymous-network theory characterizes when topology and local
information permit symmetry-sensitive tasks such as leader election
\citep{yamashita1996computing}.
Permutation structure can be built in through parameter sharing or set
architectures, or promoted through transformed examples and consistency
objectives; related work also regularizes task-equivalent prompts and models
transition regularities in world and local MARL models
\citep{ravanbakhsh2017equivariance,zaheer2017deep,hounie2023automatic,
zhou2022prompt,park2022symmetric,wu2023models}.
\scd brings these ideas into a deployed anonymous local policy by promoting
consistent decisions across equivalent local views and representations that
track evolving coordination state.

\par
\setlength{\parskip}{3.5pt plus .5pt minus .5pt}
\section{\method: The Deployed Anonymous Local Law}
\label{sec:problem}
\label{sec:method}

\method separates a learned semantic reaction law from a fixed local runtime.
We first define what each anonymous node observes and decides, then show how
repeated decisions generate reversible coordination, and finally state the
boundary between learned semantics and substrate realization.

\subsection{Anonymous local view}

In the primary synchronous runtime, $t$ indexes local-update rounds on a
contact graph $G_t=(V,E_t)$, with all active-node views formed from the same
pre-update state.  Node $i$ has private task
state and evidence $p_i^t$, a set $\mathcal{H}_i^t$ of currently incident
channel handles, and an optional local commitment $c_i^t$.  Handles are opaque,
sampled independently for each node from an episode-local pool, and fixed only
within that episode; they are routing references rather than designated agent
identities.  Every active node receives the same bounded interface,
\begin{equation}
\label{eq:view}
\view_i^t =
\operatorname{Canon}\!\left(
q,\;p_i^t,\;w_i^t,\;
\{\tracevec_{ih}^t\}_{h\in\mathcal{H}_i^t},\;
c_i^t,\;b_i^t
\right),
\end{equation}
where $q$ is the task contract, $w_i^t$ is an adapter-provided candidate native
action derived from private evidence $p_i^t$, and $b_i^t$ is the remaining
local budget.  The private proposal is an input to the joint local decision,
not a separate decision or language-model call.
$\operatorname{Canon}$ standardizes field structure rather than substrate
semantics: adapters place native task contracts, private observations, and
proposals in the common $q$, $p_i^t$, and $w_i^t$ slots, while incident
communication uses a shared bounded trace schema.  Native action names
accompany the same controller instruction (Appendix~\ref{app:interface}).
Serialization exposes neither persistent identities nor a global node ordering.

A policy is \emph{anonymous-local} when all nodes share $\theta$, every social
input arrives through an incident channel, and neither its view nor output
schema contains a designated identity, role, roster, population-size field,
non-neighbor state, or global transcript.  The interface enforces anonymity;
Section~\ref{sec:orbit-learning} trains consistency across equivalent views
obtained by permuting incident records and consistently relabeling opaque local
references.

\subsection{One joint local decision}
\label{sec:joint-decision}

Each node invokes the same controller once per local update,
\begin{equation}
\tilde d_i^t = \mathrm{Decode}_{\theta}(\view_i^t),
\qquad
d_i^t = \mathrm{Valid}_{s}(\tilde d_i^t,\view_i^t)
= (a_i^t,m_i^t,\Delta c_i^t),
\end{equation}
where $s$ denotes the substrate, $a_i^t$ is a task-level action,
$m_i^t$ is an optional bounded semantic
update addressed through incident handles, and $\Delta c_i^t$ is an optional
local commitment update.  The decoder proposes all three components jointly,
and the substrate validator admits only locally legal fields and trace
relations.  The decision type defines the available local
effects, while $\pi_\theta$ learns their state-dependent selection and semantic
content.  All three components are produced by one decode, so acting and
communicating are chosen jointly rather than by separate procedures.

The same decision type is used across substrates.  The validator checks local
legality, while the fixed runtime transports and expires admitted traces and
realizes task-level outputs in native mechanics.  The exact JSON serialization,
field bounds, and prompt appear in Appendix~\ref{app:interface}.

\subsection{From repeated decisions to reversible organization}

For claim $z$, a trace carried by incident handle $h\in\mathcal{H}_i^t$ has
bounded state
\begin{equation}
\tracevec_{ih}^t(z) =
[n_{ih}^t(z),\,s_{ih}^t(z),\,x_{ih}^t(z),\,T_{ih}^t(z)],
\end{equation}
where $n$, $s$, $x$, and $T$ encode novelty, support, conflict, and remaining
lifetime.  The controller conditions its next local decision on these bounded
traces.

Repeated local decisions can produce reversible coordination trajectories.
Depending on its current evidence, a node may write or relay a record, form a
commitment, or revise one after new conflict.  The runtime supplies locally
checkable readiness and finite expiry; the policy chooses the underlying
semantic write and commitment update.  These repeated reactions allow
coordination to form and change without persistent roles or a broadcast
transcript.  Recruitment, commitment formation, and reopening are post hoc
event labels rather than explicit controller modes.

\subsection{Runtime boundary and substrate realization}

The two primary substrates implement different native mechanics behind this
boundary.  \swarmbench realizes records as egocentric actions and contact-local
cue pulses, whereas \silo uses private-shard operations and local-port evidence
frames.  Readiness is similarly local: \swarmbench exposes a fixed
support-and-coherence predicate and \silo exposes native constructibility.
Fresh writes must match private evidence; relays must reference incident
evidence, cannot amplify its trace components, and lose at least one TTL step.
Detailed native realizations, duplicate suppression, expiry, and compiler
behavior appear in Appendices~\ref{app:substrates},
\ref{app:execution-evidence}, and~\ref{app:protocol}.

This contract defines what the local law can observe and change; the remaining
question is how to learn a law that is stable across anonymous views and
sensitive to evolving coordination state.

\section{Learning the Law with Swarm-Consistent Distillation}
\label{sec:learning}

Decision distillation teaches the policy to emit valid local records, but
leaves two swarm-specific ambiguities unresolved.  First, equivalent anonymous
views may induce different semantic decisions when incident order and opaque
names change.  Second, imitation alone does not require the controller
representation to capture how the local coordination field evolves during
collective execution.  \emph{Swarm-Consistent Distillation} (\scd) addresses
these ambiguities with consistency across anonymous views and rollout-grounded
next-field prediction.

\subsection{Decision distillation}

Let $\mathcal{D}_{\mathrm{dec}}=\{(\view_a,d_a^\star)\}_{a=1}^{A}$ denote
anonymous local views paired with schema-valid teacher decisions.  The teacher
observes exactly one local view and returns the joint decision defined in
Section~\ref{sec:joint-decision}.  Using assistant-record likelihood only, the decision
objective is
\begin{equation}
\label{eq:decision-loss}
\mathcal{L}_{\mathrm{dec}}
= \mathbb{E}_{(\view,d^\star)\sim\mathcal{D}_{\mathrm{dec}}}
\big[-\log p_\theta(d^\star\mid\view)\big].
\end{equation}
It teaches the controller to map private evidence and incident traces to a
legal task action, semantic write, and commitment decision.  The token-level
expansion and corpus construction appear in
Appendix~\ref{app:decision-details}.

\subsection{Consistency across anonymous views}
\label{sec:orbit-learning}

For a local view $\view$, let $\mathcal{G}(\view)$ be the family generated by
incident-record permutations and view-local bijections over opaque channel and
claim references.  For each training pair, we sample
$g\sim\mathcal{G}(\view)$ and construct $(g\view,gd^\star)$ by applying $g$
consistently to the view and record, while non-identifier task content remains
unchanged.  Thus the pair encodes one anonymous local situation under two
surface representations.  Let $\rho(d)$ be an identifier-free summary of the
task action, communication act, trace attributes, and commitment; by
construction, $\rho(gd)=\rho(d)$.

Orbit supervision has two complementary roles.  Transformed token supervision
preserves the correct opaque references in the emitted record, while
identifier-free alignment requires the two views to represent the same
semantic decision.  A factorized training head $Q_{\theta,\psi}$ predicts
$\rho(d)$ from the controller representation, giving the compact objective
\begin{equation}
\label{eq:orbit-loss}
\begin{aligned}
\mathcal{L}_{\mathrm{orbit}}
= \mathbb{E}_{(\view,d^\star),g}\Big[&-\log p_\theta(gd^\star\mid g\view)
+\mathcal{L}_{\mathrm{sup}}\!\left(
Q_{\theta,\psi}(\view),Q_{\theta,\psi}(g\view);\rho(d^\star)\right)\\
&+\beta D_{\mathrm{SKL}}\!\left(
Q_{\theta,\psi}(\view),Q_{\theta,\psi}(g\view)\right)\Big].
\end{aligned}
\end{equation}
Here $\mathcal{L}_{\mathrm{sup}}$ is factorized supervised cross-entropy on
both views and $D_{\mathrm{SKL}}$ is symmetric KL divergence.  Supervision
prevents content-free agreement, while transformation-consistent augmentation
preserves executable references \citep{hounie2023automatic,zhou2022prompt};
Appendix~\ref{app:orbit-details} gives the complete factorization.

\paragraph{Exact local equivariance lifts to the swarm level under an
equivariant runtime.}
Let a legal $\Gamma$ have node-permutation component $\sigma$ and consistently
transport every schema-designated opaque reference.  It induces
$g_{\Gamma,i}^t$ on local views and records, including incident-order
permutations, with
$(\Gamma D^t)_{\sigma(i)}=g_{\Gamma,i}^t d_i^t$; below we suppress the
$\Gamma,t$ indices on $g$.  Write $F$ for one synchronous update and let
$\pi_\theta^{\mathrm{dep}}$ denote the complete local decode-and-admit pipeline,
including permitted repair and conservative projection.

\begin{proposition}[Conditional local-to-global relabeling equivariance]
\label{prop:trajectory-equivariance}
Suppose every node uses the same $\pi_\theta^{\mathrm{dep}}$.  Suppose also that
view construction, $\pi_\theta^{\mathrm{dep}}$, and $F$ commute with every legal relabeling:
\[
\resizebox{0.98\linewidth}{!}{$\displaystyle
\operatorname{View}_{\sigma(i)}(\Gamma S)=g_i\operatorname{View}_i(S),\qquad
\pi_\theta^{\mathrm{dep}}(g_i\view)=g_i\pi_\theta^{\mathrm{dep}}(\view),\qquad
F(\Gamma S,\Gamma D;\Gamma\xi)=\Gamma F(S,D;\xi).$}
\]
If $S_\Gamma^0=\Gamma S^0$ and exogenous randomness is coupled by the same
relabeling, then $(S_\Gamma^t,D_\Gamma^t)=(\Gamma S^t,\Gamma D^t)$ for every
$t$.  Consequently every trajectory functional $J$ satisfying
$J(\Gamma\tau)=J(\tau)$ is unchanged.
\end{proposition}

Appendices~\ref{app:synchronous-equivariance}--\ref{app:event-driven-equivariance}
give the proof and its event-driven
extension under equivariant scheduling.

\begin{corollary}[Symmetry preservation]
\label{cor:symmetry-preservation}
Under Proposition~\ref{prop:trajectory-equivariance}, consider a deterministic
execution and a nonidentity legal relabeling $\Gamma$ satisfying
$\Gamma S^0=S^0$ and $\Gamma\xi^t=\xi^t$.  Then
$\Gamma S^t=S^t$ and $\Gamma D^t=D^t$ for every $t$; hence a uniquely selected
node must be fixed by $\Gamma$.
\end{corollary}

Thus anonymity preserves genuine symmetries until task state, topology,
interaction history, or episode-local randomness distinguishes the
alternatives.  \scd promotes the proposition's local premise, while the
\agentsnetrp diagnostic in Appendix~\ref{app:agentsnet-rp} tests the
corollary's symmetry-breaking prediction.

\subsection{Anticipating the local coordination field}
\label{sec:field-learning}

Anonymous consistency constrains how a decision is represented but does not
directly teach how its local coordination field evolves.  After fitting a
decision-supervised warm start, we execute complete training-seed episodes
through the fixed runtime and collect genuine adjacent transitions
$\mathcal{D}_{\mathrm{tr}}=\{(\view_i^t,d_i^t,\view_i^{t+1})\}$, where
$d_i^t$ is the admitted focal decision and $\view_i^{t+1}$ is the next view of
the same node after the synchronous update.  The bounded, identifier-free
target is
\begin{equation}
\label{eq:field-summary}
m(\view)=
[\operatorname{Hist}_{n}(\view),\operatorname{Hist}_{s}(\view),
 \operatorname{Hist}_{x}(\view),\operatorname{Hist}_{T}(\view),
 r(\view),\kappa(\view)],
\end{equation}
where the four histograms summarize incident novelty, support, conflict, and
remaining lifetime, while $r$ and $\kappa$ denote local readiness and commitment.
Writing $h_\theta(\view_i^t)$ for the controller's final current-view hidden
state, a training-only head $f_\omega$ receives this representation together
with the admitted-decision summary $\rho(d_i^t)$ and predicts
$\hat m_i^{t+1}=f_\omega(h_\theta(\view_i^t),\rho(d_i^t))$.  We group
the prediction loss by target family,
\begin{equation}
\label{eq:field-loss}
\mathcal{L}_{\mathrm{field}}
=\mathbb{E}_{\mathcal{D}_{\mathrm{tr}}}
\left[\mathcal{L}_{\mathrm{hist}}(\hat m_i^{t+1},m(\view_i^{t+1}))
+\lambda_{\mathrm{state}}\mathcal{L}_{\mathrm{state}}
(\hat m_i^{t+1},m(\view_i^{t+1}))\right].
\end{equation}
$\mathcal{L}_{\mathrm{hist}}$ averages cross-entropy over the four trace
histograms, whereas $\mathcal{L}_{\mathrm{state}}$ averages over readiness and
commitment.  We use equal group weight, $\lambda_{\mathrm{state}}=1$, so the
four histogram targets do not dominate the two state targets by count alone;
the componentwise expression appears in Appendix~\ref{app:field-details}.

Because neighboring nodes act simultaneously and the substrate and runtime
also determine transport, expiry, and commitment updates, this objective
implements rollout-conditioned next-field prediction rather than an isolated
causal model of the focal decision.  The objective therefore encourages the
shared representation to retain local regularities involving support
accumulation, readiness, commitment,
conflict-triggered reopening, and
expiry---local changes relevant to subsequent reorganization---rather than
reconstruct a global world state
\citep{park2022symmetric,wu2023models}.

\subsection{Joint training and deployment}

The complete objective is
\begin{equation}
\label{eq:scd}
\mathcal{L}_{\mathrm{SCD}}(\theta,\psi,\omega)
:=\mathcal{L}_{\mathrm{dec}}(\theta)
+\lambda_{\mathrm{orbit}}\mathcal{L}_{\mathrm{orbit}}(\theta,\psi)
+\lambda_{\mathrm{field}}\mathcal{L}_{\mathrm{field}}(\theta,\omega).
\end{equation}
Here, $\theta$ parameterizes the shared local law $\pi_\theta$ retained at
deployment, whereas $\psi$ and $\omega$ parameterize the training-only orbit
and field heads, respectively.  Training first obtains a decision-supervised
warm start, uses that policy to collect a fixed corpus of rollout transitions,
and then jointly updates all three parameter sets with interleaved decision
and transition minibatches.  Both auxiliary heads are discarded afterward;
deployment retains only $\pi_\theta$, so SCD adds no inference-time
components.  Appendix~\ref{app:training} specifies the model, data, complete
losses, and optimization schedule.

\section{Experiments}
\label{sec:experimental-design}
\label{sec:results}

We organize the evaluation from capability to mechanism.  Scaling, specialist
retention, and frozen transfer establish what one law can do; structural and
late-counterevidence diagnostics connect learning to behavior, while
orthogonal controls separate policy and runtime.

\subsection{Experimental Setup and Baselines}

Our primary matrix evaluates one jointly trained law at
$N\in\{16,32,64\}$ on two complementary local-interaction substrates, with all
evaluation cases and seeds held out from training.  \swarmbench evaluates
spatial coordination with contact-local communication
\citep{ruan2025swarmbench}; \silo integrates private language shards through
sparse local ports \citep{zhang2026silobench}.

We compare three interface-matched controls: \nocomm suppresses explicit
inter-agent messages \citep{huettenrauch2019deep}; \nativemsg uses transient
task-native messages over matched contacts, without a persistent coordination
field or quorum
\citep{sukhbaatar2016commnet,foerster2016dial,zhang2026silobench}; and
\fixedswarm replaces learning with a shared hand-coded local rule over
finite-lifetime signals \citep{angluin2006computation,shaw2022formic}.
Broader comparisons include \gptswarm, \gdesigner, \argdesigner, \agentnet,
\dmoa, \swarmsys, and Agent Distillation (\agentdistill)
\citep{zhuge2024language,zhang2025gdesigner,li2026assemblecrew,
yang2025agentnet,wu2026dmoa,li2025swarmsys,kang2025distilling}.  All methods share task instances,
populations, evaluation seeds, evaluators, and applicable interaction horizons.
Interface-matched controls also share \method's topology and private proposals;
broader baselines retain their native information-flow and private-computation
interfaces subject to the documented budget settings.
Appendix~\ref{app:baseline-implementations} details benchmark
integration and selection.

Tables~\ref{tab:g3-cross-substrate-core}--\ref{tab:agentsnet-transfer} report
point estimates; Appendix~\ref{app:paired-uncertainty} gives selected paired
95\% intervals, resampling matched evaluation units, plus training seeds for
independently trained controller comparisons.

\subsection{One Law Scales across Population and Interaction Budgets}
\label{sec:horizon-results}

Despite observing only bounded local views and no population size, the same
\method policy remains effective across all tested scales, achieving the
strongest overall profile in the primary cross-substrate matrix
(Table~\ref{tab:g3-cross-substrate-core}). Its advantage persists as populations
grow, including 58.2\% \silo success at $N=64$, while the distillation baseline
does not close the gap under comparable inference cost
(Appendix~\ref{app:inference-accounting}).
Appendices~\ref{app:anchors} (lower scales), \ref{app:population-summaries}
(additional tasks), and~\ref{app:silo-results} (\textsc{Silo} difficulty)
extend these results.

\begin{table}[t]
\centering
\caption{Cross-substrate scaling over five matched seeds. \silo reports
success (\%) / messages per agent; \swarmbench reports raw and task-normalized
scores. Lavender/peach mark first/second; paired intervals appear in
Appendix~\ref{app:paired-uncertainty}.}
\label{tab:g3-cross-substrate-core}
\scriptsize
\setlength{\tabcolsep}{2.2pt}
\renewcommand{\arraystretch}{1.02}
\resizebox{\textwidth}{!}{%
\begin{tabular}{@{}lccc@{\hspace{0.4em}}ccc@{\hspace{0.4em}}ccc@{\hspace{0.4em}}ccc@{\hspace{0.4em}}c@{}}
\toprule
& \multicolumn{6}{c}{\silo} & \multicolumn{6}{c}{\swarmbench}
& \multicolumn{1}{c}{Overall $\downarrow$} \\
\cmidrule(lr){2-7}\cmidrule(lr){8-13}\cmidrule(lr){14-14}
& \multicolumn{3}{c}{Population scaling: SR $\uparrow$ / msg. $\downarrow$}
& \multicolumn{3}{c}{$N=64$ SR by difficulty $\uparrow$}
& \multicolumn{3}{c}{Raw scores $\uparrow$: Pursuit / Foraging / Synchronization}
& \multicolumn{3}{c}{Task-normalized $\uparrow$}
& \shortstack{Avg.\\rank $\downarrow$} \\
\cmidrule(lr){2-4}\cmidrule(lr){5-7}\cmidrule(lr){8-10}\cmidrule(lr){11-13}
Method & $N=16$ & $N=32$ & $N=64$
& Level I & Level II & Level III
& $N=16$ & $N=32$ & $N=64$
& $N=16$ & $N=32$ & $N=64$ & \\
\midrule
\multicolumn{14}{l}{\scriptsize\color{black!62}\textit{Interface-matched controls}} \\[-1pt]
\nocomm & 10.7/0.00 & 9.8/0.00 & 11.7/0.00 & 20.0 & 10.1 & 5.0
& 1.71/\secondnum{1.31}/0.30 & 1.30/0.89/0.10 & 7.00/2.86/0.00
& 0.222 & 0.153 & 0.627 & 9.17 \\
\nativemsg & 74.7/9.00 & 67.0/9.00 & 14.3/9.00 & 25.2 & 13.0 & 4.7
& 1.50/1.03/3.40 & 0.80/0.53/2.00 & 0.99/0.82/0.80
& 0.348 & 0.194 & 0.167 & 8.46 \\
\fixedswarm & 77.6/9.00 & \secondsrmsg{69.7}{9.00} & 15.0/9.00
& 26.6 & 13.9 & 4.5
& 1.80/0.84/4.20 & 2.20/1.16/2.80 & 6.19/2.33/1.50
& 0.389 & 0.366 & 0.615 & 6.38 \\
\specialrule{0.35pt}{1.5pt}{0.6pt}
\multicolumn{14}{l}{\scriptsize\color{black!62}\textit{Broader baselines}} \\[-1pt]
\gptswarm & 12.0/4.10 & 11.3/4.98 & 13.7/4.40 & 23.2 & 12.9 & 5.0
& 1.61/0.96/2.00 & 2.40/1.37/1.70 & \secondnum{8.21}/3.55/1.30
& 0.275 & 0.338 & 0.821 & 7.96 \\
\gdesigner & 14.0/3.80 & 13.6/4.30 & 16.0/4.00 & 27.5 & 15.3 & 5.2
& 1.70/1.02/2.40 & 2.28/1.42/2.00 & 8.01/3.65/1.50
& 0.306 & 0.351 & 0.829 & 7.00 \\
\argdesigner & 19.9/3.23 & 19.1/3.51 & 20.8/3.41 & 34.6 & 22.3 & 6.0
& 1.80/1.08/3.20 & 2.40/\secondnum{1.55}/2.60
& 8.15/\secondnum{3.85}/2.00
& 0.358 & 0.400 & \secondcell{0.880} & 4.71 \\
\agentnet & 25.7/6.61 & 16.3/6.66 & 14.7/6.68 & 24.8 & 14.1 & 5.2
& \secondnum{1.90}/0.99/2.90 & \secondnum{2.60}/1.32/2.20 & 5.70/2.23/1.60
& 0.340 & 0.370 & 0.585 & 6.88 \\
\dmoa & 19.0/3.40 & 18.8/3.70 & 21.5/3.60 & 34.8 & 21.4 & 8.3
& 1.82/1.08/3.30 & 2.33/1.49/2.70 & 8.11/3.80/2.10
& 0.364 & 0.397 & 0.879 & 4.79 \\
\swarmsys & 61.7/9.00 & 44.3/9.00 & 26.7/9.00
& 47.9 & 23.1 & 9.1
& \bestnum{2.00}/1.00/4.00 & 1.11/0.75/3.20
& 7.81/3.59/2.40
& 0.404 & 0.292 & 0.860 & 5.25 \\
\agentdistill & \secondsrmsg{79.0}{8.20} & 69.0/8.35 & \secondsrmsg{43.7}{8.50}
& \secondcell{71.1} & \secondcell{48.7} & \secondcell{11.3}
& 1.75/1.20/\secondnum{4.60}
& 2.20/1.30/\secondnum{3.70}
& 7.35/3.50/\secondnum{3.10}
& \secondcell{0.439} & \secondcell{0.425} & 0.864 & \secondcell{4.29} \\
\midrule
\textbf{\method} & \bestsrmsg{86.3}{7.66} & \bestsrmsg{79.0}{7.72} & \bestsrmsg{58.2}{8.03}
& \bestcell{87.6} & \bestcell{69.9} & \bestcell{17.1}
& \secondnum{1.90}/\bestnum{1.55}/\bestnum{5.80}
& \bestnum{2.64}/\bestnum{1.64}/\bestnum{5.00}
& \bestnum{8.23}/\bestnum{4.33}/\bestnum{4.20}
& \bestcell{0.540} & \bestcell{0.546} & \bestcell{1.041} & \bestcell{1.12} \\
\bottomrule
\end{tabular}
}
\end{table}

We further test robustness to changing deployment conditions along two axes.
Without retraining or adaptation, the frozen law scales to $N=128$, retaining
82.9\% of its matched $N=64$ \silo success while keeping per-agent communication
bounded (Appendix~\ref{app:n128-population-extension}). The interaction-budget
sweep shows a similar pattern: \method reaches strong performance with
substantially fewer interaction rounds than competing methods
(Figure~\ref{fig:horizon-budget-sweep}). Together, these results indicate that
the learned local law is not tied to a fixed population size or interaction
horizon; the same local reactions continue to compose as both change.

\subsection{One Shared Law Retains Specialist Quality and Transfers}
\label{sec:joint-specialist-results}

Having established composition across population and interaction budgets, we
next test two forms of reuse: preserving specialist quality within the
training substrates and transferring the frozen law to a new substrate.  Under
matched in-domain exposure, the joint law slightly exceeds the two specialists
on both substrates.  Even when each oracle specialist receives the entire
joint data budget in its own domain, the shared law retains 96.8\% and 97.6\%
of oracle quality on \swarmbench and \silo, respectively, while using one
adapter and half the aggregate labeled data
(Figure~\ref{fig:horizon-budget-sweep}(d);
Appendix~\ref{app:joint-specialists}).

\begin{figure*}[t]
\centering
\includegraphics[width=\textwidth,trim=0 5pt 0 0,clip]{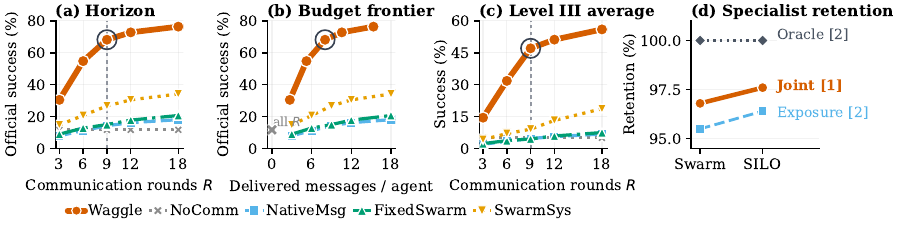}
\vspace{-5pt}
\setlength{\abovecaptionskip}{0pt}
\caption{\textbf{Communication efficiency and specialist retention at $N=64$.}
(a,c) Success across communication rounds; (b) success versus delivered
messages per agent; (d) oracle-relative retention. Outlined markers indicate
the primary $R=9$ setting.}
\label{fig:horizon-budget-sweep}
\vspace{-4pt}
\end{figure*}

On frozen AgentsNet transfer, the law leads four of five tasks, attains the
best strict solved fraction at both sizes and the highest scale retention, and
uses fewer messages than every communicating comparator
(Table~\ref{tab:agentsnet-transfer}).  Episode-local priority helps mainly on
symmetry-obstructed LeaderElection graphs, consistent with
Corollary~\ref{cor:symmetry-preservation}; Appendix~\ref{app:agentsnet-rp}
gives the protocol and symmetry diagnostics.  Complementary private-worker
transfer and paired reuse results appear in Appendices~\ref{app:g5-api-worker}
and~\ref{app:g6}.

\begin{table*}[t]
\centering
\caption{Frozen \agentsnetrp transfer without parameter updates. Lavender/peach
mark first/second; paired 95\% intervals for the main comparisons are reported
in Appendix~\ref{app:paired-uncertainty}.}
\label{tab:agentsnet-transfer}
\fontsize{8}{8.4}\selectfont
\setlength{\tabcolsep}{0.85pt}
\renewcommand{\arraystretch}{1.00}
\begin{tabular}{@{}lccccc@{\hspace{2pt}}ccc@{\hspace{2pt}}c@{\hspace{2pt}}c@{\hspace{2pt}}c@{}}
\toprule
Method
& \multicolumn{5}{c}{Task-wise soft score $\uparrow$}
& \multicolumn{3}{c}{Strict solved fraction $\uparrow$}
& Scale ret. $\uparrow$
& Efficiency & Overall $\downarrow$ \\
\cmidrule(lr){2-6}\cmidrule(lr){7-9}\cmidrule(lr){10-10}\cmidrule(lr){11-11}\cmidrule(lr){12-12}
& Coloring & Consensus & \shortstack{Leader-\\Election} & Matching
& \shortstack{Vertex-\\Cover}
& $N=8$ & $N=16$ & Overall & $16/8$
& \shortstack{Msg./agent\\$\downarrow$} & Avg. rank \\
\midrule
\multicolumn{12}{l}{\scriptsize\color{black!62}\textit{Interface-matched controls}} \\[-2pt]
\nocomm
& 0.10 & 0.88 & 0.05 & 0.04 & 0.08
& 0.28 & 0.18 & 0.23 & 0.64 & 0.0 & 6.00 \\
\nativemsg
& 0.34 & 0.90 & 0.55 & 0.25 & 0.18
& 0.54 & 0.35 & 0.44 & 0.65 & 28.4 & 4.93 \\
\fixedswarm
& 0.44 & \secondcell{0.91} & \secondcell{0.72} & 0.34
& \bestcell{0.35}
& \secondcell{0.64} & \secondcell{0.46} & \secondcell{0.55}
& \secondcell{0.72} & 15.8 & \secondcell{2.29} \\
\specialrule{0.4pt}{0.6pt}{0.6pt}
\multicolumn{12}{l}{\scriptsize\color{black!62}\textit{Shared-controller variants}} \\[-2pt]
\textsc{Zero-shot}
& 0.39 & 0.90 & 0.63 & 0.30 & 0.21
& 0.58 & 0.38 & 0.48 & 0.66 & 14.6 & 4.07 \\
\textsc{Decision only}
& \secondcell{0.47} & \secondcell{0.91} & 0.68 & \secondcell{0.36}
& 0.27
& \secondcell{0.64} & 0.44 & 0.54 & 0.69 & \secondcell{12.6} & 2.57 \\
\specialrule{0.4pt}{0.6pt}{0.6pt}
\textbf{\method}
& \bestcell{0.56} & \bestcell{0.94} & \bestcell{0.80} & \bestcell{0.44}
& \secondcell{0.32}
& \bestcell{0.69} & \bestcell{0.53} & \bestcell{0.61}
& \bestcell{0.77} & \bestcell{11.1} & \bestcell{1.14} \\
\bottomrule
\end{tabular}
\end{table*}

\subsection{SCD Learns Anonymous Stability and Local Dynamics}
\label{sec:scd-ablation-results}

The four-arm factorial reveals a clear division of labor between the two \scd
terms while held-out decision NLL remains effectively unchanged.  Orbit supervision increases deployed
anonymous-view agreement by 22.3 percentage points, whereas field supervision
reduces frozen-probe loss by 34.3\% on eventful transitions.  Complete \scd
combines both effects (Table~\ref{tab:scd-objective-attribution}).  The
training-head columns verify that each auxiliary objective is optimized, while
the deployment columns measure what remains in the decoder and frozen controller
representation after those heads are removed
(Appendix~\ref{app:scd-primary-factorial}).

These diagnostics connect the two \scd objectives to distinct deployed
properties. Orbit supervision improves decision consistency across equivalent
anonymous views, while field supervision makes evolving local coordination
state more accessible in the frozen representation. The former is the local
property underlying Proposition~\ref{prop:trajectory-equivariance}; the latter
captures dynamics relevant to later reorganization. The next experiment asks
whether these structural changes translate into better reorganization after
counterevidence.

\begin{table}[t]
\centering
\caption{Four-arm \scd factorial over five paired training seeds. Values are
mean $\pm$ s.d.; exact equivariance entries are held-out agreement rates.
Paired \scd effects are reported in Appendix~\ref{app:scd-primary-factorial}.}
\label{tab:scd-objective-attribution}
\newcommand{\pmcell}[2]{\shortstack{$#1$\\[0pt]
  {\fontsize{6}{6.2}\selectfont$\pm #2$}}}
\newcommand{\bestpm}[2]{\shortstack{$\mathbf{#1}$\\[0pt]
  {\normalfont\fontsize{6}{6.2}\selectfont$\pm #2$}}}
\newcommand{\ratecell}[1]{\shortstack{$#1$\\[0pt]
  {\fontsize{5.8}{6}\selectfont\strut}}}
\newcommand{\bestrate}[1]{\shortstack{$\mathbf{#1}$\\[0pt]
  {\fontsize{5.8}{6}\selectfont\strut}}}
\fontsize{6.8}{7.35}\selectfont
\setlength{\tabcolsep}{0.83pt}
\renewcommand{\arraystretch}{1.08}
\begin{tabular}{@{}l*{12}{c}@{}}
\toprule
Objective
& \multicolumn{1}{c}{Imitation}
& \multicolumn{3}{c}{Training objective}
& \multicolumn{5}{c}{Deployed structure}
& \multicolumn{3}{c}{Closed-loop reorganization} \\
\cmidrule(lr){2-2}\cmidrule(lr){3-5}\cmidrule(lr){6-10}\cmidrule(lr){11-13}
& NLL $\downarrow$
& \shortstack{Orbit-head acc.\\(\%) $\uparrow$}
& \multicolumn{2}{c}{\shortstack{Train-head\\loss $\downarrow$}}
& \shortstack{Summary\\agree. (\%) $\uparrow$}
& \multicolumn{2}{c}{\shortstack{Exact record equiv.\\(\%) $\uparrow$}}
& \multicolumn{2}{c}{\shortstack{Frozen-probe\\loss (norm.) $\downarrow$}}
& \shortstack{Recovery\\(\%) $\uparrow$}
& \shortstack{Stale\\AUC $\downarrow$}
& \shortstack{\textsc{Silo}\\msg./agent $\downarrow$} \\
& & & All & Event & & Decoded & Admitted & All & Event & & & \\
\midrule
\textsc{Decision only}
& \pmcell{0.0191}{0.0000}
& \pmcell{24.6}{8.7}
& \pmcell{3.643}{0.258}
& \pmcell{3.611}{0.385}
& \pmcell{72.4}{3.8}
& \ratecell{61.2}
& \ratecell{62.0}
& \pmcell{1.000}{0.000}
& \pmcell{1.000}{0.000}
& \pmcell{80.0}{10.0}
& \pmcell{22.768}{1.55}
& \pmcell{7.157}{0.18} \\
\addlinespace[1pt]
$+$ Orbit
& \pmcell{0.0192}{0.0001}
& \pmcell{96.6}{0.4}
& \pmcell{3.512}{0.301}
& \pmcell{3.398}{0.288}
& \pmcell{94.7}{0.8}
& \ratecell{88.6}
& \ratecell{89.4}
& \pmcell{0.968}{0.030}
& \pmcell{0.944}{0.040}
& \pmcell{82.0}{9.4}
& \pmcell{22.2}{1.45}
& \pmcell{7.05}{0.17} \\
\addlinespace[1pt]
$+$ Field
& \pmcell{0.0191}{0.0000}
& \pmcell{23.5}{4.8}
& \pmcell{0.152}{0.006}
& \pmcell{1.289}{0.037}
& \pmcell{71.6}{3.1}
& \ratecell{60.1}
& \ratecell{60.9}
& \pmcell{0.724}{0.035}
& \pmcell{0.657}{0.040}
& \pmcell{86.0}{8.0}
& \pmcell{20.0}{1.25}
& \pmcell{6.86}{0.16} \\
\midrule
\rowcolor{rankfirstbg}
\textbf{\method}
& \pmcell{0.0192}{0.0001}
& \bestpm{97.0}{0.2}
& \bestpm{0.147}{0.006}
& \bestpm{1.248}{0.019}
& \bestpm{95.2}{0.5}
& \bestrate{90.3}
& \bestrate{91.1}
& \bestpm{0.681}{0.030}
& \bestpm{0.608}{0.035}
& \bestpm{88.5}{7.2}
& \bestpm{19.237}{1.10}
& \bestpm{6.737}{0.14} \\
\bottomrule
\end{tabular}
\end{table}

\begin{figure}[t]
\centering
\includegraphics[width=\linewidth,trim=0 5pt 0 0,clip]{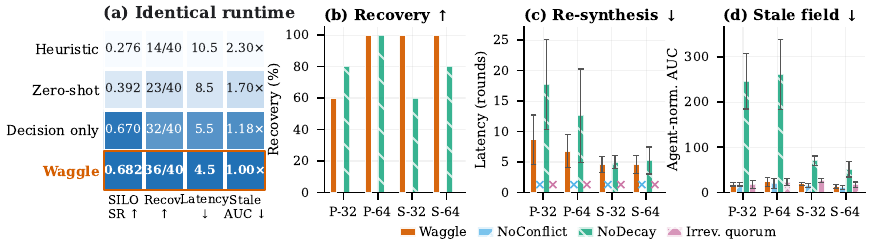}
\vspace{-5pt}
\setlength{\abovecaptionskip}{0pt}
\caption{\textbf{Decision-source and runtime-component attribution.} (a)
Decision laws under a shared runtime. (b--d) Runtime ablations on \swarmbench
Pursuit (P) and a \silo late-conflict task (S). Error bars show $\pm 1$ s.d.;
crosses denote undefined metrics when no settings recover.}
\label{fig:g4-mechanism-attribution}
\vspace{-4pt}
\end{figure}

\subsection{SCD Improves Reorganization after Counterevidence}
\label{sec:scd-behavior-results}

The late-counterevidence intervention carries the same four-arm factorial
into closed loop.  With the runtime fixed, the two single-term arms yield
intermediate behavioral gains, while complete \scd has the strongest joint
profile.  Recovery, requiring conflict-triggered reopening followed by
evidence-supported replacement synthesis, rises from 80.0\% to 88.5\%, while stale-field burden
and post-conflict traffic both fall (Table~\ref{tab:scd-objective-attribution};
paired effects in Appendix~\ref{app:scd-primary-factorial}).
Stale-field AUC integrates obsolete semantic-trace mass after conflict, so lower
values indicate faster clearance.  A representative trajectory shows verified
conflict dissolving the old commitment, synthesis building a replacement, and
both activities subsiding after the new commitment stabilizes
(Appendix~\ref{app:g3-late-conflict}).

The same structural and behavioral pattern replicates with
Llama-3.1-8B-Instruct under the same runtime
(Appendix~\ref{app:llama-backbone-replication}).  This motivates the
attribution question addressed next: whether the gain comes from learned
semantic decisions or from the fixed runtime affordances that realize them.

\subsection{Learned Decisions and Runtime Mechanics Play Complementary Roles}
\label{sec:same-runtime-results}

Figure~\ref{fig:g4-mechanism-attribution} applies two orthogonal
interventions: panel (a) fixes the interface and runtime while varying the
decision source, whereas panels (b--d) fix the \method controller and remove
conflict handling, reversibility, or decay; Appendix~\ref{app:g4-mechanism-extension}
extends the component test.  Under this fixed runtime, \method achieves the
strongest joint profile across task quality and late-conflict reorganization.
\textsc{Decision only} remains close on conventional quality but adapts less
effectively after counterevidence, while the zero-shot and heuristic laws trail
further.  These differences are therefore attributable to the decision source;
Appendix~\ref{app:same-runtime-comparison} reports the full comparison.

Across the instrumented evaluation matrix, nearly every first decode is
executable as emitted, and the no-communication fallback is never invoked.
Runtime repair can only remove invalid content or reuse the node's private
proposal; it cannot originate social content
(Appendix~\ref{app:held-out-decisions}; per-mode results in
Appendix~\ref{app:controller-results}).

Removing conflict handling or commitment reversibility collapses recovery;
removing decay preserves most recovery but delays re-synthesis and stale-state
clearance (Appendix~\ref{app:runtime-component-intervention}).
Nearby runtime settings preserve performance while
shifting the communication--responsiveness trade-off
(Appendix~\ref{app:runtime-sensitivity}; execution perturbations:
Appendix~\ref{app:execution-perturbations}).  Together, these interventions
separate learned semantics from the bounded runtime that realizes them.

\FloatBarrier
\section{Conclusion and Limitations}
\label{sec:discussion}

\method shows that adaptive LLM-agent organization can be learned as a
reusable anonymous local law, without representing organization as an explicit
system-level structure.  Across the tested substrates and scales, repeated
local reactions form, sustain, and revise coordination, while \scd improves
anonymous stability and reorganization after counterevidence.  More broadly,
these results suggest a complementary design principle for multi-agent
systems: organization can arise repeatedly from a learned local rule instead
of being specified directly.

The present study deliberately isolates this organizing law behind a bounded
canonical interface and a fixed runtime.  Deterministic substrate adapters
expose private state, locally checkable evidence, and legal actions, while
validation, bounded transport, expiry, and native realization remain fixed
runtime mechanisms and are not learned end to end.  The theoretical guarantee
requires equivariant local decisions, view construction, and runtime updates;
the empirical evidence establishes reuse across the tested local-interaction
substrates and transfer settings.
Appendices~\ref{app:responsible} and~\ref{app:future-directions} detail these
scope boundaries and extensions.

\clearpage

\section*{Reproducibility Statement}

Appendices~\ref{app:substrates}--\ref{app:interface} document the substrate
adapters, runtime validation, and canonical local interface, while
Appendix~\ref{app:trajectory-equivariance} gives the assumptions and proof of
conditional relabeling equivariance.  Training-corpus construction, data
splits, optimization settings, and checkpoint selection are detailed in
Appendices~\ref{app:training} and~\ref{app:scd-ablation}.
Appendix~\ref{app:protocol} specifies benchmark protocols, baseline
implementations, and budget settings.  Transfer configurations appear in
Appendices~\ref{app:agentsnet-rp} and~\ref{app:g5-api-worker}.
Training and model selection exclude formal evaluation cases and seeds.
The \scd comparison uses five paired training seeds with shared teacher
records and optimization budgets, with matched-evaluation and paired-bootstrap
procedures described in Appendix~\ref{app:scd-primary-factorial}.

\section*{Ethics Statement}

We study coordination in benchmark environments using teacher-generated
training records.  The experiments contain no human-subject data.
Anonymous handles remove identity cues from the controller interface but are
not a privacy guarantee.  In open-world deployments, unreliable observations
or malicious messages could propagate through local interactions and
influence collective actions.  Bounded transport, finite trace lifetimes,
and reversible commitments constrain coordination dynamics but do not
establish a complete security boundary.  Applications involving consequential
external actions should authenticate evidence provenance, sandbox tools,
protect sensitive inputs, and retain appropriate human approval.
Appendix~\ref{app:responsible} discusses these deployment assumptions and
limitations.

\section*{AI Use Statement}

GPT-4o was used in the experimental teacher and private-worker roles described
in Section~\ref{sec:learning} and Appendix~\ref{app:g5-api-worker}.
Generative AI was used for language polishing, including grammar, clarity, and
stylistic editing; to assist with code development and debugging; and to support
brainstorming of figure concepts and visual design ideas. All AI-assisted text,
code, and analyses were critically reviewed and verified against the underlying
sources, data, and implementations. The authors take full responsibility for
the final manuscript.

\bibliographystyle{iclr2027_conference}
\bibliography{references}

\appendix

\clearpage
\section*{Appendix Contents}
\begingroup
\small
\setlength{\parskip}{0pt}
\setcounter{tocdepth}{2}
\makeatletter
\renewcommand{\l@section}[2]{%
  \addpenalty{\@secpenalty}\addvspace{3pt}%
  {\bfseries\@dottedtocline{1}{0em}{2.4em}{#1}{#2}}}
\renewcommand{\l@subsection}{\@dottedtocline{2}{2.4em}{3.2em}}
\@starttoc{apc}
\makeatother
\endgroup
\clearpage
\let\waggleOriginalAddcontentsline\addcontentsline
\renewcommand{\addcontentsline}[3]{%
  \waggleOriginalAddcontentsline{#1}{#2}{#3}%
  \ifstrequal{#1}{toc}{%
    \ifstrequal{#2}{section}
      {\waggleOriginalAddcontentsline{apc}{#2}{#3}}
      {\ifstrequal{#2}{subsection}
        {\waggleOriginalAddcontentsline{apc}{#2}{#3}}{}}%
  }{}%
}

\section{Lower-Scale Anchor Results}
\label{app:anchors}

These lower-scale evaluations locate the frozen controller before the main
population expansion.  They use ten held-out seeds and are retained as
context; the \swarmbench values provide the Pursuit/Foraging $N=8$ anchors,
with Synchronization normalized by its matched 6.40 reference.  None of these
episodes is reused in a complete-matrix comparison.  Compact summaries appear in
Tables~\ref{tab:swarm-main} and~\ref{tab:silo-main}; the population and task
breakdowns below retain the detailed anchor evidence.

\begin{table}[h]
\centering
\caption{Lower-scale \swarmbench anchor at $N=8$.}
\label{tab:swarm-main}
\scriptsize
\setlength{\tabcolsep}{3.2pt}
\begin{tabular}{lrrrr}
\toprule
& \multicolumn{2}{c}{Pursuit} & \multicolumn{2}{c}{Foraging} \\
\cmidrule(lr){2-3}\cmidrule(lr){4-5}
Method & Score $\uparrow$ & Msg. $\downarrow$
       & Score $\uparrow$ & Msg. $\downarrow$ \\
\midrule
\nativemsg & 3.70 $\pm$ 3.33 & 1742.5 $\pm$ 531.9
          & 2.90 $\pm$ 2.18 & 341.2 $\pm$ 325.4 \\
\fixedswarm  & 6.10 $\pm$ 4.09 & 509.8 $\pm$ 148.5
          & 3.70 $\pm$ 1.42 & 80.5 $\pm$ 73.5 \\
\method   & \textbf{6.20 $\pm$ 7.76} & \textbf{430.3 $\pm$ 69.8}
          & \textbf{3.80 $\pm$ 2.39} & \textbf{57.0 $\pm$ 16.1} \\
\bottomrule
\end{tabular}
\end{table}

At the \swarmbench anchor, \method improves mean score by 0.10 over \fixedswarm
on both tasks while reducing messages by 15.6\% on Pursuit and 29.2\% on
Foraging.

\begin{table}[h]
\centering
\caption{Lower-scale six-task \silo anchor at $N\in\{4,8\}$.}
\label{tab:silo-main}
\small
\setlength{\tabcolsep}{7pt}
\begin{tabular}{lrrr}
\toprule
Method & Success (\%) $\uparrow$ & Msg. $\downarrow$ & kB $\downarrow$ \\
\midrule
\nativemsg & \textbf{100.00} & 54.00 & 49.21 \\
\fixedswarm  & \textbf{100.00} & 54.00 & 50.04 \\
\method   & 97.92 & \textbf{41.63} & \textbf{34.47} \\
\bottomrule
\end{tabular}
\end{table}

\begin{table}[h]
\centering
\caption{Population breakdown of the lower-scale \silo anchor.}
\label{tab:silo-scale}
\small
\setlength{\tabcolsep}{5pt}
\begin{tabular}{crrrr}
\toprule
$N$ & Success (\%) & Msg. & kB & \nativemsg/\fixedswarm kB \\
\midrule
4 & 100.00 & 27.27 & 16.68 & 21.85 / 22.42 \\
8 & 95.83  & 55.98 & 52.25 & 76.58 / 77.66 \\
\bottomrule
\end{tabular}
\end{table}

Across this lower-scale evaluation anchor, \method reaches 97.92\% success
with 22.9\% fewer messages than either matched baseline and 30.0\% fewer bytes
than \nativemsg.  The primary $N\in\{16,32,64\}$ cross-method \silo
evaluation is reported in Table~\ref{tab:g3-cross-substrate-core};
Appendix~\ref{app:silo-results} breaks \method down by population and
difficulty.

\section{Detailed Native Realizations}
\label{app:substrates}

The same learned record is presented to both substrate adapters.
Table~\ref{tab:substrates-detail} separates the policy's semantic choices from
their fixed native realization.  Adapters construct private proposals, expose
locally checkable facts, and compile admitted records; they do not originate a
claim, communication intent, commitment, or controller mode.

\begin{table}[H]
\centering
\caption{Learned semantic decisions and fixed native realization in the two
substrates.}
\label{tab:substrates-detail}
\scriptsize
\setlength{\tabcolsep}{3.5pt}
\begin{tabularx}{\linewidth}{p{0.13\linewidth}p{0.27\linewidth}XX}
\toprule
Element & Learned local law & \swarmbench fixed realization
& \silo fixed realization \\
\midrule
Task
& Chooses a legal task action or waits
& Egocentric local operation
& Private-shard operation or response \\
Trace
& Chooses content, incident target, and whether to relay
& Contact-local cue with bounded per-hop relay
& Local-port evidence frame with finite-lifetime relay \\
Commit and reopen
& Chooses whether to synthesize, commit, challenge, or revise
& Support/coherence readiness; contradictory cue
& Native constructibility; incompatible partial evidence \\
Safeguards
& Makes no schema, provenance, duplication, or expiry decision
& Validates provenance, suppresses repeats, bounds relay, and expires cues
& Validates provenance, suppresses repeats, bounds relay, and expires frames \\
\bottomrule
\end{tabularx}
\end{table}

\FloatBarrier
\section{Population, Controller, and Runtime Evidence}
\label{app:execution-evidence}

\subsection{Population summaries and quality--communication trade-offs}
\label{app:population-summaries}

The task-normalized \swarmbench aggregate averages Pursuit, Foraging, and
Synchronization after division by the frozen $N=8$ \method anchors, 6.20,
3.80, and 6.40, respectively.  Table~\ref{tab:g3-cross-substrate-core}
reports this aggregate alongside the raw scores.  Flocking and Transport
complete the five-task \swarmbench evaluation; their full population-scale
results appear in Table~\ref{tab:swarmbench-flocking-transport}.
Figure~\ref{fig:g3-success-cost} retains the Pursuit/Foraging telemetry
diagnostic, whereas Figure~\ref{fig:g3-scale-transfer} uses all three tasks.

\begin{table}[h]
\centering
\caption{Raw Flocking and Transport scores in the five-task \swarmbench
evaluation.  Bold/underline mark first/second within each task--population
column, with lavender/peach shading.}
\label{tab:swarmbench-flocking-transport}
\scriptsize
\setlength{\tabcolsep}{4.0pt}
\renewcommand{\arraystretch}{1.06}
\begin{tabular}{@{}lcccccc@{}}
\toprule
\rowcolor{tableheaderbg}
& \multicolumn{3}{c}{Flocking $\uparrow$}
& \multicolumn{3}{c}{Transport $\uparrow$} \\
\cmidrule(lr){2-4}\cmidrule(lr){5-7}
\rowcolor{tableheaderbg}
Method & $N=16$ & $N=32$ & $N=64$ & $N=16$ & $N=32$ & $N=64$ \\
\midrule
\multicolumn{7}{l}{\scriptsize\color{black!62}\textit{Interface-matched controls}} \\[-1pt]
\nocomm      & 5.40 & 5.00 & 4.30 & 0.00 & 0.00 & 0.00 \\
\nativemsg   & 6.00 & 5.70 & 5.00 & 0.25 & 0.20 & 0.10 \\
\fixedswarm & \secondcell{7.60} & 7.20 & 6.50 & 0.45 & 0.55 & 0.45 \\
\specialrule{0.35pt}{1.5pt}{0.6pt}
\multicolumn{7}{l}{\scriptsize\color{black!62}\textit{Organizational baselines}} \\[-1pt]
\gptswarm    & 5.80 & 6.10 & 5.70 & 0.20 & 0.25 & 0.20 \\
\gdesigner   & 6.20 & 6.50 & 6.10 & 0.30 & 0.35 & 0.30 \\
\argdesigner & 6.70 & 7.00 & 6.70 & 0.45 & 0.55 & 0.50 \\
\agentnet    & 6.60 & 6.80 & 6.30 & 0.40 & 0.50 & 0.45 \\
\dmoa        & 6.80 & 7.10 & 6.80 & 0.50 & 0.60 & 0.55 \\
\swarmsys    & 7.20 & \secondcell{7.50} & \secondcell{7.00}
              & \secondcell{0.65} & \secondcell{0.80} & \secondcell{0.75} \\
\agentdistill & 7.00 & 7.10 & 6.40 & 0.55 & 0.65 & 0.55 \\
\midrule
\textbf{\method}
& \bestcell{7.80} & \bestcell{8.10} & \bestcell{7.70}
& \bestcell{1.20} & \bestcell{1.55} & \bestcell{1.70} \\
\bottomrule
\end{tabular}
\end{table}

\begin{figure}[h]
\centering
\begin{minipage}[t]{0.487\linewidth}
\vspace{0pt}
\centering
\includegraphics[width=\linewidth]{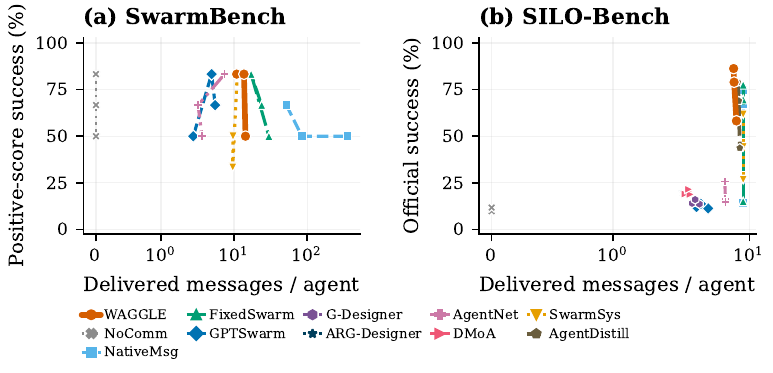}
\caption{\textbf{Quality--communication trade-off over five matched evaluation
seeds.} The shared legend lists all eleven methods: the \swarmbench
Pursuit/Foraging subset shows the seven with both positive-score success and
delivered-message telemetry, whereas \silo shows all eleven.}
\label{fig:g3-success-cost}
\end{minipage}
\hfill
\begin{minipage}[t]{0.487\linewidth}
\vspace{0pt}
\centering
\includegraphics[width=\linewidth]{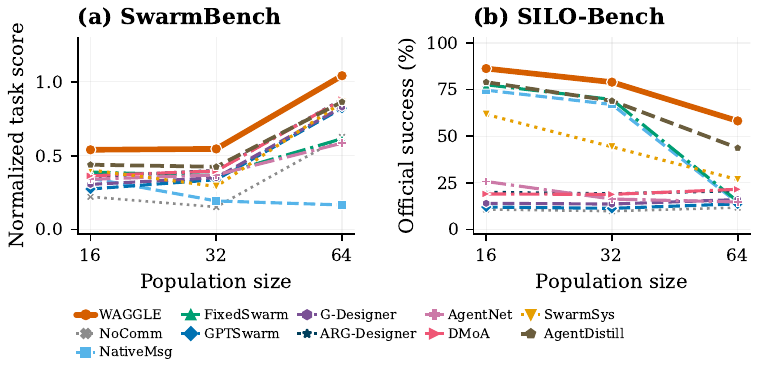}
\caption{\textbf{Cross-method population scaling.}
\swarmbench scores use task-specific frozen $N=8$ \method references;
\silo reports official success.}
\label{fig:g3-scale-transfer}
\end{minipage}
\end{figure}

\FloatBarrier
\subsection{Population extension beyond the primary matrix}
\label{app:n128-population-extension}

We test deployment-time population extrapolation without retraining or
checkpoint selection.  The joint \method controller remains frozen, and we
compare it with \fixedswarm, the interface-matched fixed-rule control, and
\swarmsys, a broader \silo organizational baseline.  The
evaluation uses \silo and \swarmbench Pursuit and Foraging at
$N\in\{64,128\}$ with the same five held-out seeds, yielding 960
episodes.  The $N=64$ values are newly evaluated five-seed anchors for this
extension.

\begin{table}[h]
\centering
\caption{Frozen-controller population extrapolation from $N=64$ to $N=128$;
scale retention is $100\,\mathrm{SR}_{128}/\mathrm{SR}_{64}$.}
\label{tab:n128-population-extension}
\footnotesize
\setlength{\tabcolsep}{4.2pt}
\renewcommand{\arraystretch}{1.08}
\begin{tabular}{lcccccc}
\toprule
\rowcolor{tableheaderbg}
& \multicolumn{4}{c}{\silo} & \multicolumn{2}{c}{\swarmbench} \\
\cmidrule(lr){2-5}\cmidrule(lr){6-7}
\rowcolor{tableheaderbg}
Method
& \shortstack{$N=64$\\SR (\%) $\uparrow$}
& \shortstack{$N=128$\\SR (\%) $\uparrow$}
& \shortstack{Scale\\retention $\uparrow$}
& \shortstack{$N=128$\\msg./agent $\downarrow$}
& \shortstack{$N=64$\\P/F task-norm. $\uparrow$}
& \shortstack{$N=128$\\P/F task-norm. $\uparrow$} \\
\midrule
\textbf{\method}
& \textbf{68.5} & \textbf{56.8} & \textbf{82.9\%}
& \textbf{8.25} & \textbf{1.230} & \textbf{1.380} \\
\fixedswarm
& 15.4 & 8.6 & 55.8\% & 9.00 & 0.810 & 0.860 \\
\swarmsys
& 27.1 & 17.8 & 65.7\% & 9.00 & 1.100 & 1.150 \\
\bottomrule
\end{tabular}
\end{table}

The frozen law retains 82.9\% of its matched $N=64$ \silo success at
$N=128$, compared with 55.8\% for \fixedswarm and 65.7\% for \swarmsys,
while using 8.25 delivered messages per agent.  Its Pursuit/Foraging
task-normalized \swarmbench score also increases from 1.230 to 1.380.  The scaling advantage
in the primary matrix therefore persists when the deployment population
doubles beyond $N=64$.

\subsection{Held-out local decisions}
\label{app:held-out-decisions}

We evaluate the frozen controller on 800 held-out anonymous local views.  Every
generation is valid JSON, satisfies the schema, and is executable.  Derived
six-mode macro-F1 is 100\%, including Challenge and Synthesize, while exact
full-record match is 74.75\%.  Many non-exact records differ only in a bounded
deposit value or an equivalent local choice without changing the operational
mode.  This is a record-format and derived-category fidelity diagnostic;
end-to-end coordination is evaluated on full benchmark episodes in Section~\ref{sec:experimental-design}.

\begin{table}[h]
\centering
\caption{Frozen-controller validity and derived-mode classification on 800
held-out local views.}
\label{tab:controller}
\small
\begin{tabular}{lrrrr}
\toprule
& JSON valid & Schema valid & Executable & Mode macro-F1 \\
\midrule
Qwen3-4B + LoRA & 100.00 & 100.00 & 100.00 & 100.00 \\
\bottomrule
\end{tabular}

\vspace{2pt}
\footnotesize Decision exact match: 74.75\%.
\end{table}
Across the 72 \method cells in the core matrix, 85,129 of 85,344 first
decodes were executable (99.75\%).  Of the 215 remaining first decodes, one
was envelope-normalized, 207 were repaired by one validator-code-conditioned
regeneration, and seven were conservatively projected; the no-communication
fallback was never invoked.  Wrapper operations are limited to deleting invalid
content or reusing the local private proposal; social content remains
controller-authored.

\subsection{Mechanism activity on \swarmbench}

Entries are per-episode mean $\pm$ sample s.d. over the ten $N=8$ anchor
seeds.

\begin{table}[h]
\centering
\caption{Controller-selected and runtime-filtered writes on \swarmbench at
$N=8$.}
\label{tab:mechanism-activity}
\footnotesize
\setlength{\tabcolsep}{4.0pt}
\begin{tabularx}{\linewidth}{lXXX}
\toprule
Task & Overall writes & Local suppression & Context safeguards \\
\midrule
Pursuit
& Candidate $420.6\pm133.5$; executed $139.9\pm17.1$
& Repeats blocked $147.8\pm110.8$; refreshes $30.2\pm20.7$
& Unverified challenges blocked $11.8\pm7.6$; route relays silenced $0.0$ \\
Foraging
& Candidate $93.8\pm47.7$; executed $24.7\pm8.3$
& Repeats blocked $27.6\pm35.4$; refreshes $0.0$
& Unverified challenges blocked $0.0$; route relays silenced $26.6\pm16.2$ \\
\bottomrule
\end{tabularx}
\end{table}

Table~\ref{tab:mechanism-activity} shows that one local law selects different
swarm primitives from event context.  Pursuit keeps 1,399 of 4,206 candidate
writes (33.3\%), blocks 1,478 spatial repeats among 3,259 deposit candidates
(45.4\%), and executes 302 controlled refreshes as target evidence moves.
Every emitted inhibitory challenge is backed by a verified local conflict
(52/52), while 118 of 170 unverified challenge candidates are blocked.
Foraging keeps 247 of 938 candidate writes (26.3\%), blocks 276 spatial
repeats among 446 deposit candidates (61.9\%), and silences 266 of 492
route-active relay candidates (54.1\%).  Thus Pursuit activates refreshed
evidence and verified conflict, whereas Foraging activates event-once evidence
and private-route protection.

\section{Canonical Local Interface}
\label{app:interface}

\subsection{Controller view}

The serialized controller view contains exactly six components: a task
contract, private state, benchmark-native private proposal, incident fields,
local commitment, and remaining-budget bin.  The historical JSON key
\texttt{worker\_proposal} denotes that native private proposal; the final
system does not invoke a second worker LLM.  Recursive schema validation
rejects identity, roster, population-size, non-neighbor, global-state, and
global-transcript keys.

For the reported interfaces, \swarmbench uses a $5\!\times\!5$ egocentric window
and \silo uses degree-three ports; the resulting trace list is bounded by those
benchmark contracts.

Each incident trace contains:
\begin{quote}
\small\verb|{channel, claim, novelty, support, conflict, ttl, direction?}|
\end{quote}
Channel handles are sampled from an episode-local pool, randomized per node, and
fixed only within that episode.  Incident records are seed-controlled shuffled
at each local update.  During \scd, incident permutations and opaque
handle/claim relabelings are applied as paired, consistency-checked anonymous
orbits to both the view and its teacher record.  Equation~\ref{eq:orbit-loss}
promotes semantic consistency across these paired anonymous views.

\subsection{Decision schema}

The controller emits, and the validator accepts, the following canonical schema:
{\small
\begin{verbatim}
{
  "task_action": string,
  "response": string | null,
  "deposits": [
    {channel, claim, novelty, support, conflict, ttl, direction?}
  ],
  "commit": {claim, confidence_bin} | null,
  "execution_intent": "PRIVATE" | "FIELD" | "WAIT"  // optional; omission = AUTO
}
\end{verbatim}
}
Numeric trace components use five bins (0--4); TTL uses 1--8; confidence uses
0--4.  Canonical serialization includes \texttt{execution\_intent} only for a
non-\texttt{AUTO} request; omission denotes the fixed default compiler rule.
The validator derives the mode after checking consistency with the local view.  In particular, a relay must match incident evidence, a challenge
must contain positive conflict, and an abstaining wait cannot write or commit.

For analysis only, the validator assigns the first matching label in the
precedence order in Table~\ref{tab:modes}; labels are not controller outputs.

\begin{table}[h]
\centering
\caption{Precedence for assigning derived analysis modes.}
\label{tab:modes}
\small
\setlength{\tabcolsep}{4pt}
\begin{tabular}{lp{6.0cm}}
\toprule
Mode & First matching condition \\
\midrule
Challenge & Valid positive-conflict write \\
Synthesize & Commitment or response \\
Deposit & Fresh private-evidence write \\
Relay & Write matching an incident trace \\
Explore & Non-wait execution field with no higher label \\
Abstain & Legal wait with no write or commitment \\
\bottomrule
\end{tabular}
\end{table}

\subsection{Anonymous controller prompt}

All nodes use the same instruction.  Benchmark-specific action names and the
serialized local view are appended mechanically.

\begin{quote}
\small
You are one anonymous member of a changing local network.  Use only the
private proposal, private evidence, current commitment, and semantic traces on
incident channels.  Return one canonical JSON decision.  Continue useful
private work unless local evidence justifies a deposit, relay, challenge, or
commitment.  Treat repeated support as local evidence rather than global
consensus.  Fresh verified conflict may reopen a commitment.  Do not infer an
identity, role, roster, population size, or global transcript.
\end{quote}

\section{Proof of Conditional Local-to-Global Relabeling Equivariance}
\label{app:trajectory-equivariance}

\subsection{Synchronous execution}
\label{app:synchronous-equivariance}

We make the actions in Proposition~\ref{prop:trajectory-equivariance} explicit.
The relabeling $\Gamma$ has node-bijection component $\sigma$ and transports
the graph and all node-indexed state, consistently renaming every
schema-designated opaque reference, including reference-bearing action
arguments; it does not change topology, population size, reference-free task
semantics, directions, or trace values.  Its induced $g_{\Gamma,i}^t$ may also
permute incident-record presentation, and we define
$(\Gamma D^t)_{\sigma(i)}=g_{\Gamma,i}^t d_i^t$.  The map $\pi_\theta^{\mathrm{dep}}$
includes decoding, validation, repair, and conservative projection, so its
equivariance concerns the complete admitted record.  Node- or edge-indexed
random draws are carried to their $\Gamma$-images, whereas label-free draws are
shared; denote this coupling by $\xi_\Gamma^t=\Gamma\xi^t$.  All components of
$D^t$ are computed from the same pre-update state $S^t$ before $F$ is applied.

\begin{proof}[Proof of Proposition~\ref{prop:trajectory-equivariance}]
The initial states obey $S_\Gamma^0=\Gamma S^0$ by construction.  Suppose
$S_\Gamma^t=\Gamma S^t$.  For every node $i$, view and local-map equivariance
give
\begin{align*}
(d_\Gamma^t)_{\sigma(i)}
&=\pi_\theta^{\mathrm{dep}}\!\left(
\operatorname{View}_{\sigma(i)}(S_\Gamma^t)\right) \\
&=\pi_\theta^{\mathrm{dep}}\!\left(
g_{\Gamma,i}^t\operatorname{View}_{i}(S^t)\right)
=g_{\Gamma,i}^t d_i^t
=(\Gamma D^t)_{\sigma(i)}.
\end{align*}
Thus $D_\Gamma^t=\Gamma D^t$.  Applying the coupled synchronous update yields
\begin{align*}
S_\Gamma^{t+1}
&=F(S_\Gamma^t,D_\Gamma^t;\xi_\Gamma^t) \\
&=F(\Gamma S^t,\Gamma D^t;\xi_\Gamma^t)
=\Gamma F(S^t,D^t;\xi^t)
=\Gamma S^{t+1}.
\end{align*}
Induction proves both identities for every round.  Therefore the two
trajectories satisfy $\tau_\Gamma=\Gamma\tau$, and every $J$ with
$J(\Gamma\tau)=J(\tau)$ obeys
$J(\tau_\Gamma)=J(\Gamma\tau)=J(\tau)$.
\end{proof}

\begin{proof}[Proof of Corollary~\ref{cor:symmetry-preservation}]
The original and relabeled deterministic executions receive the same initial
state and exogenous inputs, so they coincide.  Proposition~\ref{prop:trajectory-equivariance}
also identifies the relabeled trajectory with the $\Gamma$-image of the
original.  Hence every state and joint decision is $\Gamma$-invariant.  A
unique selected node must therefore be fixed by $\Gamma$.
\end{proof}

With independently resampled draws, the conclusion holds in distribution
when their conditional law is equivariant under the corresponding state and
history relabeling.  Stochastic decoding analogously requires an equivariant
kernel for the complete admitted record.

\subsection{Event-driven extension}
\label{app:event-driven-equivariance}

The local-to-global argument also applies to event-driven execution.
We formulate an asynchronous runtime contract here; the primary benchmark
executions and training transitions retain their synchronous construction.
Appendix~\ref{app:execution-perturbations} separately tests execution perturbations.
Let the complete execution state be
\[
X^k=(S^k,M^k,P^k,T^k,R^k),
\qquad
X^{k+1}=\mathcal U(X^k,e_k;\xi_k).
\]
Here $S^k$ contains task and local coordination state, $M^k$ contains buffered
and in-flight messages, $P^k$ contains pending calls and their saved input
snapshots, $T^k$ contains local clocks and expiry timers, and $R^k$ contains
active membership, the contact graph, node attributes, and routing state.
An event $e_k$ starts or completes a local call, delivers or drops a message,
advances a timer or expires a trace, or changes membership or routing;
$\xi_k$ contains the associated runtime randomness.
The index $k$ orders events for analysis, not through a clock or coordinator
available to the agents.  A pending call retains its invocation snapshot even
when other events change the current local view before it completes.

\paragraph{Relabeling and completion-time admission.}
A legal $\Gamma$ consistently renames a universe of possible node
incarnations, including future arrivals, and all schema-designated opaque
references.  It transports message endpoints, pending-call descriptors and
snapshots, timer owners, node attributes, and membership/routing records.
Reference-free task content and numerical attributes are unchanged in value
and move with their associated nodes or records.
Thus the active sets obey $V_\Gamma^k=\sigma(V^k)$ in paired executions,
although $V^k$ may vary with $k$.
In the asynchronous contract, a completed output is rechecked against the
current local state before taking effect: targets must remain valid, required
evidence must remain admissible, and the native action must remain legal.
This completion-time admission, including cancellation or rejection after
departure or expiry, must commute with $\Gamma$.  Cancelled or rejected calls
are represented by a common symbol $\bot$ fixed by relabeling.

\begin{proposition}[Conditional event-driven relabeling equivariance]
\label{prop:event-driven-equivariance}
Suppose all calls use the same equivariant local decode-and-admit map
$\pi_\theta^{\mathrm{dep}}$ on their saved views.  Suppose view construction
and snapshot storage, event legality, completion-time admission, and every
non-policy runtime operation commute with each legal relabeling.
For any legal event history, couple the two executions by
\[
X_\Gamma^0=\Gamma X^0,\qquad
e_{\Gamma,k}=\Gamma e_k,\qquad
\xi_{\Gamma,k}=\Gamma\xi_k.
\]
Then $X_\Gamma^k=\Gamma X^k$ for every event index $k$, and corresponding
completed calls produce consistently relabeled admitted records, including
$\bot$.  Consequently, any relabeling-invariant functional of the event
trajectory has the same value in the two executions.
\end{proposition}

\begin{proof}
At call initiation, equivariant view construction and storage produce
corresponding saved snapshots.  At completion, local-map equivariance gives
corresponding outputs from those snapshots, and current-state admission
preserves this relation for both accepted and rejected records.
Delivery, dropping, timer updates, and membership/routing changes commute
with relabeling by assumption.  Composing the operations for each event
therefore gives
\[
\mathcal U(\Gamma X,\Gamma e;\Gamma\xi)
=\Gamma\mathcal U(X,e;\xi).
\]
The initial states are paired by hypothesis.  If $X_\Gamma^k=\Gamma X^k$,
equivariant event legality makes the paired event admissible, and
\begin{align*}
X_\Gamma^{k+1}
&=\mathcal U(X_\Gamma^k,e_{\Gamma,k};\xi_{\Gamma,k})\\
&=\mathcal U(\Gamma X^k,\Gamma e_k;\Gamma\xi_k)\\
&=\Gamma\mathcal U(X^k,e_k;\xi_k)
=\Gamma X^{k+1}.
\end{align*}
Induction proves the state and completed-record identities on every finite
event prefix, hence on the event trajectory $\tau_\Gamma=\Gamma\tau$.
Relabeling-invariant trajectory functionals are unchanged.
\end{proof}

\paragraph{Random event generation.}
The proposition pairs particular event histories.  To obtain a
distributional statement for the same stochastic execution mechanism, let
$H_k$ denote the complete history through $X^k$ and let
$\mathcal K(\cdot\mid H_k)$ be the conditional kernel of the next
marked event $\zeta_k=(e_k,\xi_k)$.
The joint kernel includes state-dependent timing and runtime randomness,
so an equivariant event order alone is insufficient if the remaining random
choices depend on arbitrary labels.
Require, for every measurable set $B$ of legal marked events,
\[
\mathcal K(\Gamma B\mid\Gamma H_k)=\mathcal K(B\mid H_k).
\]
Terminated histories may be padded by an absorbing, label-free event.
Couple successive marked events by relabeling.  The kernel condition and
Proposition~\ref{prop:event-driven-equivariance} inductively give equal
relabeling-mapped prefix laws, and therefore
\[
\mathcal L(\tau\mid\Gamma X^0)
=\Gamma_*\mathcal L(\tau\mid X^0),
\]
where $\Gamma_*$ denotes pushforward of the trajectory law.
For stochastic decoding, require the complete-record kernel to be equivariant
as well and include its sampled outcomes in the marks; this supplies the
corresponding local-output coupling in the same argument.
Node speeds and capabilities need not be identical: their attributes travel
with their nodes, and the event kernel respects that transport.

\paragraph{Delays, loss, churn, and symmetry.}
Variable latency changes completion or delivery events while pending
snapshots remain in $P^k$.  Message loss is an explicit drop or the absence of
delivery in the paired history.  Churn relabels arrivals and their
initialization, departures, and the affected pending messages, calls, and
routes together.  These are two namings of the same event process, not a
comparison between different schedules or different membership histories.
No delay bound or scheduling fairness is needed for this structural
equality at matched event indices.  It does not imply schedule independence,
task convergence, or successful completion when information fails to arrive.
The symmetry-preservation argument of
Corollary~\ref{cor:symmetry-preservation} applies pathwise only when the
realized schedule and other exogenous inputs also preserve the symmetry.
An equivariant random scheduler can preserve symmetry in distribution while
an individual realization breaks it by activating one of two symmetric nodes
first.

\section{Training Details}
\label{app:training}

\subsection{Decision loss, corpus, and split construction}
\label{app:decision-details}

For a serialized teacher record $d=(d_1,\ldots,d_{|d|})$, the sequence
negative log-likelihood used in Equation~\ref{eq:decision-loss} expands over
assistant tokens as
\begin{equation}
\label{eq:token-loss}
\ell_{\mathrm{tok}}(\theta;\view,d)
=-\sum_{r=1}^{|d|}\log p_\theta(d_r\mid\view,d_{<r}).
\end{equation}

The decision corpus contains 6,400 synthetic and 512 train-only native local
views.  The resulting 6,912-view training split allocates 3,456 views to each
substrate.  Validation and test each contain 800 synthetic views.  The 3,456
\silo training views are task-stratified across all three benchmark levels; the
six-task lower-scale anchor in Table~\ref{tab:silo-main} is evaluation-only and
does not define the training task set.  Each synthetic curriculum bundle contains one
independently instantiated example from each of eight local-situation classes:
\emph{private-only progression} (a legal private proposal without usable incident
evidence or a commitment event), \emph{unresolved/insufficient coordination
evidence} (traces lack sufficient support or coherence for commitment, with no
verified conflict or useful new relay), \emph{fresh private evidence} (new,
locally verifiable evidence absent from the incident field), \emph{admissible
incident relay} (useful, unexpired evidence relayable without increasing trace
components), \emph{pre-commit verified conflict} (verified counterevidence to an
uncommitted candidate), \emph{commitment contradiction/reopening} (verified
counterevidence incompatible with an existing commitment),
\emph{readiness-triggered synthesis} (local readiness or constructibility
satisfied without active conflict), and \emph{safe abstention/wait} (no
warranted social update or commitment, or a required legal wait).
These are input-situation classes, distinct from the six validator-derived
mode labels.
Bundles, rather than individual examples, are assigned to splits before teacher
labeling.  This grouping controls scenario overlap across splits but does not
define a temporal trajectory; adjacent scenarios from these bundles never
contribute to $\mathcal{L}_{\mathrm{field}}$.

Teacher calls use GPT-4o at temperature 0.2.  Only canonical JSON records that
pass batch-level schema validation enter $\mathcal{D}_{\mathrm{dec}}$.  Native
decision views and all sequential rollouts use collection seeds disjoint from
formal evaluation.

\subsection{Anonymous-orbit construction}
\label{app:orbit-details}

For each decision example and epoch, we sample a view-local bijection over
opaque channel and claim identifiers together with an independent
incident-record permutation.  The mapping is applied consistently to every
schema-designated identifier occurrence in the private proposal and evidence,
incident traces, local-event claim fields, commitment, and structured teacher
record; free-form task content and native action tokens are unchanged.  We
retain the original example and its transformed partner in the same minibatch.  A pair is admitted
only when schema validation and all local freshness, relay, challenge, and
target-matching relations agree with those of the source pair.

The identifier-free decision summary contains 12 factorized labels: task-action
class, execution intent, derived semantic act, communication act, claim source,
channel scope, novelty, support, conflict, TTL, commitment action, and
commitment confidence.  Every inapplicable factor has an explicit null class;
let $\mathcal{A}$ index these factors.  Orbit heads are linear projections of
the final local-view hidden state.  They provide training gradients only and
are absent from the deployed checkpoint.

Writing $Q_{\theta,\psi}^{(k)}$ for factor $k$, the compact orbit objective in
Equation~\ref{eq:orbit-loss} expands as
\begin{align}
\label{eq:orbit-loss-components}
\mathcal{L}_{\mathrm{orbit}}
=\mathbb{E}_{(\view,d^\star),g}
\Bigg[\ell_{\mathrm{tok}}(\theta;g\view,gd^\star)
+\frac{1}{|\mathcal{A}|}\sum_{k\in\mathcal{A}}
\Big[&\operatorname{CE}\!\left(
Q_{\theta,\psi}^{(k)}(\view),\rho_k(d^\star)\right)
\nonumber\\[-2pt]
+&\operatorname{CE}\!\left(
Q_{\theta,\psi}^{(k)}(g\view),\rho_k(d^\star)\right)
\nonumber\\[-2pt]
+&\beta D_{\mathrm{SKL}}\!\left(
Q_{\theta,\psi}^{(k)}(\view),
Q_{\theta,\psi}^{(k)}(g\view)\right)\Big]\Bigg].
\end{align}
The first term reuses the transformed teacher record; it requires no additional
teacher call.  The two cross-entropies constitute
$\mathcal{L}_{\mathrm{sup}}$ in the compact notation, while the last term
aligns the corresponding factor distributions, with
$D_{\mathrm{SKL}}(p,q)=\tfrac12D_{\mathrm{KL}}(p\Vert q)
+\tfrac12D_{\mathrm{KL}}(q\Vert p)$.

\subsection{Sequential transition collection}
\label{app:field-details}

After fitting the common behavior warm start with
$\mathcal{L}_{\mathrm{dec}}$ only, we execute complete episodes with the fixed
benchmark adapters and runtime on training-only seeds, covering Pursuit,
Foraging, and \silo.  Immediately
before an active node decodes, we store $\view_i^t$; after validation we store
the admitted $d_i^t$; and after the synchronous batch is realized we store the
next view $\view_i^{t+1}$ for the same node.  Terminal records without a
well-defined next local update are masked.  The field target is computed
deterministically from $\view_i^{t+1}$ using Equation~\ref{eq:field-summary}.
No synthetic curriculum adjacency contributes to this corpus.

$\operatorname{Hist}_{n}$, $\operatorname{Hist}_{s}$, and
$\operatorname{Hist}_{x}$ are six-way normalized histograms with one null
class and bins 0--4; $\operatorname{Hist}_{T}$ has one null class and TTL bins
0--8.  A view with no incident traces places all histogram mass on the null
class.  Readiness has two classes (not-ready/ready), and commitment has six
classes (empty or confidence 0--4), so the target width is independent of
population size and incident ordering and invariant to opaque aliases.

The predictor concatenates $h_\theta(\view_i^t)$, the final current-view hidden
state, with learned 32-dimensional embeddings of the factors in
$\rho(d_i^t)$, applies a 512-dimensional GELU layer with dropout 0.05, and uses
one linear output head per target.  With $\mathcal{F}=\{n,s,x,T\}$ and
$\mathcal{B}=\{r,\kappa\}$, the grouped loss in
Equation~\ref{eq:field-loss} expands as
\begin{align}
\label{eq:field-loss-components}
\mathcal{L}_{\mathrm{field}}
=\mathbb{E}_{(\view_i^t,d_i^t,\view_i^{t+1})\sim\mathcal{D}_{\mathrm{tr}}}
\Bigg[&\frac14\sum_{k\in\mathcal{F}}
\operatorname{CE}\!\left(
 f_{\omega}^{(k)}(h_\theta(\view_i^t),\rho(d_i^t)),
 \operatorname{Hist}_{k}(\view_i^{t+1})\right)\nonumber\\
+&\frac12\sum_{k\in\mathcal{B}}
\operatorname{CE}\!\left(
 f_{\omega}^{(k)}(h_\theta(\view_i^t),\rho(d_i^t)),
 m_k(\view_i^{t+1})\right)\Bigg].
\end{align}
Histogram terms use soft-label cross-entropy and state terms use categorical
cross-entropy.  The factors $1/4$ and $1/2$ average within the histogram and
state families, respectively, giving the two groups equal total weight
($\lambda_{\mathrm{state}}=1$) rather than allowing the four histogram targets
to dominate by count.  Targets are extracted deterministically from the next
view and require no teacher call.

We stratify the transition sampler into ordinary and eventful updates.  The
latter includes verified conflict, challenge, commitment, reopening, and expiry
transitions.  This makes local-field prediction sensitive to reorganization
rather than being dominated by unchanged background fields.

\subsection{Optimization}

The student is Qwen3-4B-Instruct-2507 \citep{yang2025qwen3} with a rank-32
LoRA adapter \citep{hu2022lora}.  Stage one optimizes
$\mathcal{L}_{\mathrm{dec}}$ and supplies the shared
behavior warm start used by every factorial arm.  Stage two interleaves
decision and transition minibatches and optimizes Equation~\ref{eq:scd}.  The two auxiliary heads share the adapted backbone but are removed after
training.  All factorial arms use minimum validation decision loss for
checkpoint selection.  Deployed equivariance diagnostics and frozen-probe field
loss are computed only after selection, on data excluded from controller
training and checkpoint choice; formal benchmark episodes remain untouched.

\begin{table}[t]
\centering
\caption{\scd configuration shared across matched training arms.}
\label{tab:training}
\small
\begin{tabular}{ll}
\toprule
Component & Value \\
\midrule
Teacher / student & GPT-4o (temperature 0.2) / Qwen3-4B-Instruct-2507 \\
Initial policy & Common semantic-record warm start \\
Decision-task coverage & \shortstack[l]{\swarmbench Pursuit/Foraging;\\\silo} \\
Adaptation & LoRA rank 32, alpha 64, dropout 0.05 \\
Target modules & q/k/v/o and gate/up/down projections \\
Decision train / validation & 6,912 / 800 local views \\
Transition train / validation & 14,256 / 4,752 adjacent triplets \\
Eventful transition fraction & 0.50 \\
Stage-two budget & 648 optimizer updates ($\approx1.5$ source-data passes) \\
$\lambda_{\mathrm{orbit}}$ / $\lambda_{\mathrm{field}}$ & 0.05 / 0.05 \\
$\lambda_{\mathrm{state}}$ & 1.0 (equal histogram/state group weight) \\
Orbit consistency coefficient $\beta$ & 0.10 \\
Auxiliary warm-up / ramp & 8 / 32 updates \\
Optimizer & AdamW; weight decay 0.01 \\
Adapter / auxiliary learning rate & $2\!\times\!10^{-6}$ / $5\!\times\!10^{-4}$ \\
Effective global batch & 16 records (32 source/transformed sequences) \\
Precision / maximum context & bfloat16 / 16,384 tokens \\
Checkpoint selection & Minimum validation decision loss \\
Deployment & Greedy decoding; auxiliary heads removed \\
Hardware & Two NVIDIA H800 GPUs \\
\bottomrule
\end{tabular}
\end{table}

\section{SCD Training and Replication Protocols}
\label{app:scd-ablation}

\subsection{Primary Factorial and Downstream Protocol}
\label{app:scd-primary-factorial}

The factorial comparison uses five paired training seeds.  Within a seed, all four cells share the
same semantic-record warm start, source records, transformed pairs, adjacent
transition corpus, deterministic minibatch schedule, token budget, and 648
optimizer updates.  Each update uses the same source and transition record
indices; orbit transforms follow the same deterministic seed schedule.  Disabled losses are computed but
multiplied by zero before backpropagation; no arm replaces an auxiliary loss
with additional decision supervision.  Checkpoints are selected solely by
validation decision loss, after which structural diagnostics are evaluated on
untouched held-out data.

\paragraph{Deployed equivariance diagnostics.}
For each held-out local view $\view$, we sample a non-identity anonymous
transformation $g$ whenever the orbit is nontrivial.  The final greedy decoder
(not the training-time orbit head) independently emits $d_\theta(\view)$ and
$d_\theta(g\view)$.  After inverse-mapping the transformed output, summary
agreement compares their identifier-free decision summaries.  Exact
decoded-record equivariance instead requires equality of the complete canonical
records; exact admitted-record equivariance repeats the test after the same
validator-admission map is applied.  The three held-out rates test increasingly
strict agreement under incident permutations and fresh opaque-handle
bijections.

\paragraph{Frozen-probe field loss.}
We delete the training-time field predictor and freeze the selected controller
backbone for every factorial cell.  For each frozen checkpoint, we then fit a
fresh linear probe with the same architecture, probe-training transitions,
preprocessing, initialization schedule, optimizer, and update budget.  No
probe is reused across controllers.  Evaluation uses an independent held-out
transition split, with the eventful subset defined by verified conflict,
challenge, commitment, reopening, or expiry.  We normalize all-transition and
eventful losses separately by the corresponding \textsc{Decision only} value.
This capacity-matched protocol tests field information accessible in the frozen
deployed representation without rewarding a controller for having trained its
original auxiliary head.

Table~\ref{tab:scd-objective-attribution} consolidates the objective,
deployment, and closed-loop diagnostics. Except for the two exact-equivariance
rates, entries report mean $\pm$ sample s.d. over five paired training seeds.
Train-head losses are raw optimization diagnostics; because those heads are
removed at deployment, only deployed agreement rates and fresh frozen probes
provide deployment-facing structural evidence.

The conventional-quality matrix compares the two endpoint arms,
\textsc{Decision only} and \method, and contains
1,020 episodes: five training seeds and two arms evaluated on \silo at
$N=64$ and on Pursuit and Foraging at $N\in\{32,64\}$, with matched held-out
cases in every task--population cell.  Table~\ref{tab:scd-downstream-cells}
gives the resulting aggregates; its confidence bounds use 10,000 hierarchical
paired-bootstrap resamples, first over training seeds and then over matched
held-out cases within each task--population stratum.  The separate
late-counterevidence matrix evaluates all four factorial arms on Pursuit and
\silo III-26 at $N\in\{32,64\}$ over ten matched held-out cases.  Across five
training seeds, this yields 800 episodes, or 200 matched settings per arm.
Recovery is the fraction of these settings in which verified counterevidence
reopens the old commitment and is followed by synthesis of a replacement
supported by the new evidence.  These settings provide the final three columns
of Table~\ref{tab:scd-objective-attribution}.  All
downstream runs remove the auxiliary heads and use identical validators,
transports, and native compilers.

\begin{table}[!htbp]
\centering
\caption{Matched conventional task quality under the same runtime; $\Delta$
denotes \method minus \textsc{Decision only}.}
\label{tab:scd-downstream-cells}
\small
\setlength{\tabcolsep}{6pt}
\resizebox{0.90\textwidth}{!}{%
\begin{tabular}{lrrrrr}
\toprule
Benchmark & \textsc{Decision only} quality & \method quality & $\Delta$ quality
& 95\% lower bound & \method/\textsc{Decision only} \\
\midrule
\silo & $0.675$ & $\mathbf{0.711}$ & $\mathbf{+0.036}$ & $0.007$ & $1.053$ \\
\swarmbench & $4.09$ & $\mathbf{4.43}$ & $\mathbf{+0.34}$ & $0.06$ & $1.083$ \\
\bottomrule
\end{tabular}
}
\end{table}
\FloatBarrier

\paragraph{Paired deployment and reorganization effects.}
Table~\ref{tab:scd-paired-effects} reports complete \scd relative to
\textsc{Decision only}, separating uncertainty in the paired differences
from the per-arm sample s.d. in Table~\ref{tab:scd-objective-attribution}.
The structural endpoints use their respective held-out diagnostic sets;
the behavioral endpoints use the late-counterevidence settings described above.
The paired 95\% intervals use 10,000 hierarchical bootstrap resamples:
first resample the five paired training seeds, then matched held-out
views or transitions for structural endpoints, or matched cases within
task--population strata for behavioral endpoints.  Each resample preserves
method pairing and recomputes the contrast using the original metric
definitions and task weights.

\begin{table}[H]
\centering
\caption{Paired deployment and reorganization effects of complete \scd.
$\Delta$ is \method minus \textsc{Decision only}; agreement and recovery
differences are in percentage points (pp).  Endpoint definitions follow
Table~\ref{tab:scd-objective-attribution}.}
\label{tab:scd-paired-effects}
\fontsize{8}{9.2}\selectfont
\setlength{\tabcolsep}{4pt}
\renewcommand{\arraystretch}{1.10}
\begin{tabular}{@{}lrrrr@{}}
\toprule
\rowcolor{tableheaderbg}
Endpoint & \textsc{Decision only} & \method & Paired $\Delta$ & Paired 95\% CI \\
\midrule
Summary agreement $\uparrow$ & 72.4\% & \textbf{95.2\%} & $+22.8$ pp & $[+18.7,+26.6]$ \\
Eventful frozen-probe loss $\downarrow$ & 1.000 & \textbf{0.608} & $-0.392$ & $[-0.435,-0.350]$ \\
Recovery $\uparrow$ & 80.0\% & \textbf{88.5\%} & $+8.5$ pp & $[+2.0,+15.0]$ \\
Stale-field AUC $\downarrow$ & 22.768 & \textbf{19.237} & $-3.531$ & $[-4.85,-2.15]$ \\
Post-conflict msg./agent $\downarrow$ & 7.157 & \textbf{6.737} & $-0.420$ & $[-0.61,-0.23]$ \\
\bottomrule
\end{tabular}
\end{table}

\subsection{Replication with a Second Student Backbone}
\label{app:llama-backbone-replication}

We repeat the endpoint comparison with Llama-3.1-8B-Instruct
\citep{grattafiori2024llama3}, reusing the same GPT-4o teacher records,
canonical interface, LoRA configuration, \scd coefficients, 648-update
stage-two budget, checkpoint rule, and greedy deployment protocol.  Only the
tokenizer and chat template follow the Llama-native format; memory-preserving
microbatching keeps the effective batch unchanged.  From a common
decision-distillation warm start, we train \textsc{Decision only} and complete
\scd with five paired training seeds.

Structural evaluation uses the same 800 held-out local views and frozen-probe
protocol.  Behavioral evaluation reuses the late-counterevidence intervention
on Pursuit and \silo III-26 at $N\in\{32,64\}$ across the five paired training
seeds.  The conventional-quality
evaluation reports \silo success at $N=64$ over matched held-out
cases.  Table~\ref{tab:llama-backbone-replication}
shows that both the learned structure and its behavioral consequence replicate
with the second student family: orbit agreement rises by 17.0 points, eventful
probe loss falls by 34.0\%, and recovery rises by 7.5 points while stale-field
burden and post-conflict traffic both decrease.

\begin{table}[h!]
\centering
\caption{Replication of \scd with Llama-3.1-8B-Instruct; eventful field loss is
normalized to \textsc{Decision only}.}
\label{tab:llama-backbone-replication}
\small
\setlength{\tabcolsep}{4.5pt}
\resizebox{\textwidth}{!}{%
\begin{tabular}{lrrrrrr}
\toprule
Objective
& \shortstack{Deployed orbit\\agreement (\%) $\uparrow$}
& \shortstack{Eventful frozen-probe\\field loss $\downarrow$}
& \shortstack{\silo $N=64$\\success $\uparrow$}
& \shortstack{Late-conflict\\recovery $\uparrow$}
& \shortstack{Stale-field\\AUC $\downarrow$}
& \shortstack{Post-conflict\\messages/agent $\downarrow$} \\
\midrule
\textsc{Decision only}
& $77.0$ & $1.000$ & $0.660$ & $80.0\%$ & $23.0$ & $7.25$ \\
\textbf{\method}
& \bestcell{$94.0$} & \bestcell{$0.660$} & \bestcell{$0.673$}
& \bestcell{$87.5\%$} & \bestcell{$20.1$} & \bestcell{$6.95$} \\
\midrule
Change
& $+17.0$ pp & $-34.0\%$ & $+0.013$ & $+7.5$ pp & $-12.6\%$ & $-4.1\%$ \\
\bottomrule
\end{tabular}
}
\end{table}
\FloatBarrier

\section{Benchmark Protocol}
\label{app:protocol}

\subsection{\swarmbench}

We evaluate all five released \swarmbench task families.  The main matrix
reports Pursuit, Foraging, and Synchronization at $N\in\{16,32,64\}$ over five
matched seeds; Flocking and Transport are reported under the same protocol in
Appendix~\ref{app:population-summaries}.  The lower-scale anchor uses $N=8$ and
ten held-out seeds.
The adapter gives all active agents a local snapshot from the same pre-update
environment state in one batch each round and invokes the resident controller.
Egocentric observations generate a private
movement proposal.  Incident edge signals are projected into binned traces with
freshly randomized local channel handles.  The runtime deterministically resolves the concrete native move from the
validated record and local state, and emits at most one compact outbound trace
per local decision.  Admitted writes are staged with native actions and become
incident evidence only after the subsequent environment update.

The fixed candidate-ready gate requires two or more incident channels carrying
the same claim, total support at least eight, zero conflict, and
support-weighted directional coherence at least 0.6.  These are fixed adapter
settings held constant across the formal matrix, not learned controller
parameters.  The gate counts incident channels rather than distinct evidence
origins: it is deliberately a two-channel readiness test, not a distinct-origin
or independent-source test.

The formal \swarmbench system combines:
\begin{center}
\small
Qwen3-4B LoRA + \method shared local controller\\
+ local-fingerprint duplicate suppression + per-hop-capped relay\\
+ fixed native action compilation + verified-conflict inhibition.
\end{center}
Pursuit refreshes moving spatial evidence after twelve rounds; Foraging treats
each stationary spatial event as event-once evidence.

\subsection{\silo}

The primary \silo matrix evaluates 30 official tasks, with ten tasks in each
of Levels I--III.  Every method is run at $N\in\{16,32,64\}$ with five matched
evaluation seeds, yielding 450
task--population--seed episodes per method for
Table~\ref{tab:g3-cross-substrate-core}.  At each population size, official
success and receiver-delivered messages are macro-averaged over the complete
task--seed set.  Every episode has nine local communication rounds.  A directed
frame jointly carries one compact trace, proposal intent, and one task
snapshot.  The official case generator, task solver, partial-correctness
criterion, and success metric are retained.  Neither controller nor adapter
receives a population roster or global transcript.  \silo does not invoke the
spatial transport layer: the shared learned controller law writes through
finite-lifetime native incident transport.

The communication-horizon sweep fixes $N=64$ and evaluates the same task suite on
eight held-out seeds at each
$R\in\{3,6,9,12,18\}$.  Each horizon therefore contains 240 matched
task--seed units per method; task instances and seeds are held fixed across
horizons so that the curves vary only the available interaction budget.

The six-task subset I-04/I-09, II-11/II-16, and III-23/III-26 is used only for
the independent lower-scale evaluation anchor at $N\in\{4,8\}$ with ten
held-out seeds (Table~\ref{tab:silo-main}); it defines neither the training corpus
nor the primary matrix.
Focused intervention experiments state their narrower task scope separately.

\subsection{Baseline implementations}
\label{app:baseline-implementations}

All methods use the same formal task instances, population conditions,
evaluation seeds, benchmark-facing legal task-action spaces, evaluators, and, where
applicable, interaction horizons.  The interface-matched controls additionally share
\method's local topology and private proposals.  \nocomm suppresses explicit
outbound inter-agent messages; \nativemsg exchanges transient task-native messages
over the same local contacts, without a separate persistent semantic field,
quorum, or commitment mechanism; and \fixedswarm replaces the learned law with a
shared hand-coded local rule using finite-lifetime local signals and fixed support,
inhibition, quorum, reopening, decay, and expiry.

Broader baselines instead retain their native information-flow and
private-computation interfaces subject to the configurations below.
A benchmark-native private proposal is reused
only when doing so leaves the baseline unchanged; otherwise the method receives
the same underlying private task observation and legal action space.  Adapters
map these inputs into the native organizational loop and map task-facing
decisions back to benchmark-native actions or final answers; they do not replace
the baseline's graph construction, agent composition, routing, activation, or
role mechanism with \method's runtime.  Consequently, native differences in model calls remain part of each
method, while receiver-delivered messages measure communication rather than
total inference cost; model-call and token telemetry is recorded separately.
No method receives auxiliary calls or evaluation-answer access outside this
recorded execution.

We use a pinned release from the authors when it supports the target substrate;
otherwise Table~\ref{tab:baseline-implementations} marks a clean-room
reimplementation based on the paper and released artifacts; exact commits and
integration changes are recorded in the artifact.  Published defaults are
used when transferable.  Other substrate-specific choices use a disjoint
development split and one selection metric per benchmark family, never formal
evaluation seeds; they are not tuned per seed or population unless the original
method explicitly conditions on population size.  Model parameters,
hyperparameters, checkpoints, and selection rules are fixed before evaluation.
Algorithm-specified organizational state may still change within an episode and
is reset between episodes.

\paragraph{Coverage of organizational adaptation.}
The broader comparisons span optimized information-flow graphs (\gptswarm),
task-conditioned topology (\gdesigner), joint crew and topology generation
(\argdesigner), within-execution graph adaptation (\agentnet), step-wise
activation (\dmoa), and adaptive profile-based allocation (\swarmsys).
These complement the interface-matched controls and the distillation baseline.
Table~\ref{tab:baseline-implementations} distinguishes development-time
optimization, per-task generation, and within-execution adaptation, and
records the mechanisms retained under benchmark integration.  In particular,
our \gptswarm topology is selected on development cases and frozen for formal
evaluation, whereas \agentnet retains within-execution graph updates.

Related frameworks outside the quantitative matrix include DyLAN, with
contribution-based team selection and dynamic communication
\citep{liu2024dynamic}, and AgentVerse, with feedback-driven adjustment of
expert composition \citep{chen2024agentverse}.  CAMEL emphasizes role-prompted
dialogue \citep{li2023camel}.  OpenAI Swarm provides an agent-and-handoff
orchestration interface \citep{openai2024swarm}; a quantitative baseline would
require specified task-level prompts, tools, and handoff rules.

\paragraph{Agent Distillation (\agentdistill).}
We use the authors' official Agent Distillation implementation
\citep{kang2025distilling} in a budget-matched configuration.
Qwen3-4B-Instruct-2507 is retrained from the same base checkpoint on GPT-4o
reason--action--observation trajectories from the training splits of
\swarmbench Pursuit/Foraging and \silo, with no Flocking or Transport
training.  We retain the published trajectory SFT loss, masking observation
tokens, and train rank-64 LoRA adapters on all linear layers for two epochs
with batch size eight and learning rate $2\times10^{-4}$.  One resulting
checkpoint is shared by all $N$ agents.  No formal-evaluation seeds are used
for tuning or selection.

The primary comparison disables first-thought prefix (FTP) and self-consistent
action generation (SAG), using one greedy candidate and one student call per
active-agent update.  External retrieval and code-tool calls are both zero.
Agents exchange transient benchmark-native messages without \method's
persistent field, quorum, commitment, or reopening operations.
Communication is counted as receiver-delivered messages per agent;
Appendix~\ref{app:inference-accounting} reports deployment calls and relative
token counts separately, excluding offline teacher generation and training.

\begin{table}[!t]
\centering
\caption{Implementations of broader baselines in the primary matrix.}
\label{tab:baseline-implementations}
\footnotesize
\setlength{\tabcolsep}{3.2pt}
\renewcommand{\arraystretch}{1.10}
\begin{tabularx}{\textwidth}{@{}
>{\raggedright\arraybackslash}p{0.15\textwidth}
>{\raggedright\arraybackslash}p{0.16\textwidth}
>{\raggedright\arraybackslash}p{0.22\textwidth}
>{\raggedright\arraybackslash}p{0.20\textwidth}
>{\raggedright\arraybackslash}X@{}}
\toprule
\rowcolor{tableheaderbg}
Method & Native object & Benchmark integration
& Retained mechanism & Implementation / tuning \\
\midrule
\gptswarm
& Computational DAG of operation nodes and information-flow edges.
& Map proposal-producing agents or worker operations to DAG nodes; route over
legal contacts; translate benchmark actions or final answers at the boundary.
& Edge-probability optimization and topological execution.
& Clean-room; select a population-matched topology on disjoint development
cases, then freeze it for formal evaluation. \\
\gdesigner
& Task-conditioned communication graph over agents and a virtual task node.
& Map task information to node features; decode a topology over legal contacts;
translate exchanged content and benchmark actions or final answers.
& Graph encoder/decoder, virtual task node, and learned topology generation.
& Pinned released code; published defaults or disjoint-development selection;
frozen parameters with per-task topology decoding. \\
\argdesigner
& Autoregressively generated agent composition and communication topology.
& Map task information and the available agent pool to the graph generator;
execute the generated collaboration graph within the benchmark-legal interaction space.
& Task-conditioned autoregressive generation of crew composition, roles, and
communication links.
& Pinned released code; published defaults; frozen generator with per-task
graph generation. \\
\agentnet
& Decentralized, dynamically evolving agent DAG.
& Feed benchmark observations and outcomes into the native execution loop;
routing evolves within the legal interaction space.
& Forward/Split/Execute routing, local memory, and outcome-based edge updates.
& Clean-room; fixed defaults across populations; within-execution graph
updates retained and reset per episode. \\
\dmoa
& Sparse step-wise activation over an expert-agent pool.
& Give the router benchmark-valid context and history; selected agents return
benchmark actions or final answers through the benchmark boundary.
& Context/history-conditioned recurrent routing and sparse activation.
& Clean-room; published defaults or disjoint-development selection; frozen
router with per-step activation retained. \\
\swarmsys
& Explorer--Worker--Validator population with adaptive profiles.
& Map benchmark events into the native role loop; profile matching allocates
work; translate benchmark actions or final answers.
& Specialized roles, profile--event matching, and validation-driven
reinforcement.
& Clean-room; fixed defaults across populations; online matching and
validation feedback retained and reset per episode. \\
\agentdistill
& Distilled reason--action--observation policy.
& Shared checkpoint across $N$ agents; local observations, native task actions,
and transient inter-agent messages.
& Trajectory SFT with observation masking; one greedy call per active-agent update.
& Authors' official code; train-only retraining; budget-matched configuration
with FTP/SAG and external tools off; no evaluation tuning. \\
\bottomrule
\end{tabularx}
\end{table}

\FloatBarrier
\subsection{Paired uncertainty}
\label{app:paired-uncertainty}

Unless a caption states otherwise, reported paired confidence intervals use
10,000 bootstrap resamples while keeping the compared methods coupled within a
matched unit.  Fixed-controller benchmark contrasts resample matched
task--seed units within benchmark strata.  Comparisons spanning independently
trained controllers instead use the hierarchical procedure specified in
Appendix~\ref{app:scd-primary-factorial}: training seeds are resampled first and matched
environment seeds second within task--population strata.
Key paired effects are $+0.07$ (95\% CI $[+0.02,+0.12]$) for \method versus
\textsc{Decision only} on AgentsNet strict solved fraction, $+0.06$
($[+0.01,+0.11]$) versus \fixedswarm on the same endpoint, and $-0.017$
($[-0.029,-0.005]$) for Joint minus the oracle specialist on \silo success at
$N=64$.

\begin{table}[H]
\centering
\caption{Paired uncertainty for the primary scaling contrasts.
The comparator has the strongest non-\method point estimate for each
endpoint and population in Table~\ref{tab:g3-cross-substrate-core}.
$\Delta$ is \method minus comparator; \silo differences and intervals are
in percentage points (pp).}
\label{tab:primary-scaling-paired}
\fontsize{8}{9.2}\selectfont
\setlength{\tabcolsep}{3.2pt}
\renewcommand{\arraystretch}{1.08}
\begin{tabular}{@{}lclrrrr@{}}
\toprule
\rowcolor{tableheaderbg}
Endpoint & $N$ & Comparator & Score & \method & $\Delta$ & 95\% CI \\
\midrule
\silo SR (\%) & 16 & \agentdistill & 79.0 & 86.3 & $+7.3$ & $[+3.1,+11.5]$ \\
& 32 & \fixedswarm & 69.7 & 79.0 & $+9.3$ & $[+4.7,+13.8]$ \\
& 64 & \agentdistill & 43.7 & 58.2 & $+14.5$ & $[+8.2,+20.5]$ \\
\midrule
\swarmbench norm. & 16 & \agentdistill & 0.439 & 0.540 & $+0.101$ & $[-0.01,+0.21]$ \\
& 32 & \agentdistill & 0.425 & 0.546 & $+0.121$ & $[+0.01,+0.23]$ \\
& 64 & \argdesigner & 0.880 & 1.041 & $+0.161$ & $[+0.05,+0.28]$ \\
\bottomrule
\end{tabular}
\end{table}

Table~\ref{tab:primary-scaling-paired} retains the primary task weights and
frozen normalization anchors.  At each population, \silo contributes
$30\times5=150$ matched task--seed episodes and \swarmbench contributes
$3\times5=15$ task--seed results, with five seeds within each task.
The 10,000 paired-bootstrap resamples preserve the method pairing and task
weights.  These fixed-checkpoint contrasts concern
evaluation uncertainty rather than variation across training runs.
All three \silo intervals lie above zero; the \swarmbench interval at $N=16$
includes zero.

In Table~\ref{tab:g3-cross-substrate-core}, Avg. rank averages ranks over
12 equally weighted endpoints: three population-specific \silo success rates
and nine raw \swarmbench task scores, with ties averaged.
SR/msg. pairs are ranked by SR; bold/underline duplicate the first/second-place
colors. Rankings indicate point-estimate order, not statistical significance.

\subsection{Inference accounting}
\label{app:inference-accounting}

Table~\ref{tab:deployment-inference} separates deployment inference from
receiver-delivered communication on \silo at $N=64$.  Calls are reported as
absolute episode averages; input and output token totals are each normalized
to Agent Distillation.  These deployment figures exclude offline teacher
generation and training and are not measurements of FLOPs or wall-clock time.

\begin{table}[H]
\centering
\caption{Deployment inference and communication on \silo, $N=64$.
Token ratios use Agent Distillation as the reference ($1.00\times$);
LLM calls are absolute counts per episode.}
\label{tab:deployment-inference}
\fontsize{8}{9.2}\selectfont
\setlength{\tabcolsep}{4pt}
\renewcommand{\arraystretch}{1.10}
\begin{tabular}{@{}lrrrrr@{}}
\toprule
\rowcolor{tableheaderbg}
Method & SR (\%) $\uparrow$ & \shortstack{LLM calls /\\episode $\downarrow$}
& \shortstack{Input tokens\\(relative) $\downarrow$}
& \shortstack{Output tokens\\(relative) $\downarrow$} & Msg./agent $\downarrow$ \\
\midrule
Agent Distillation & 43.7 & 576 & $1.00\times$ & $1.00\times$ & 8.50 \\
\dmoa & 21.5 & 806 & $1.47\times$ & $1.34\times$ & 3.60 \\
\swarmsys & 26.7 & 864 & $1.62\times$ & $1.48\times$ & 9.00 \\
\method & 58.2 & 577 & $1.13\times$ & $1.09\times$ & 8.03 \\
\bottomrule
\end{tabular}
\end{table}

\method uses nearly the same number of calls as Agent Distillation, with
13\% more input tokens and 9\% more output tokens, while improving success
by 14.5 percentage points.  \dmoa delivers fewer messages but requires more
calls and tokens.  Communication savings therefore do not imply proportional
inference savings.  For \method, $A$ active-agent updates require
$C=A+C_{\mathrm{regen}}$ controller calls, where
$0\leq C_{\mathrm{regen}}\leq A$ under the one-regeneration limit
(Appendix~\ref{app:held-out-decisions}).  A no-message decision still requires
inference, and bounded local input does not make total population compute
independent of $N$.

\FloatBarrier
\section{Frozen Transfer to Randomized-Anonymous \textsc{AgentsNet}}
\label{app:agentsnet-rp}

This experiment asks whether a controller learned only from spatial and
distributed-language substrates transfers, without parameter updates, to an
unseen graph-structured coordination substrate.  We construct
\agentsnetrp from the released AgentsNet benchmark
\citep{grotschla2025agentsnet}.  The controlled adaptation retains the graph
instances, five task definitions, paper-specified round budgets, and official
task evaluators, while replacing persistent node names with opaque incident
handles and supplying private episode-local randomness for symmetry breaking.
This setting tests transfer of the learned law through a new deterministic
adapter.  AgentsNet is distinct from the singular \agentnet system baseline
used in the primary comparison matrix.

\paragraph{Protocol.}
Formal evaluation covers \textsc{Coloring}, \textsc{Consensus},
\textsc{LeaderElection}, \textsc{Matching}, and \textsc{VertexCover} on the
official \textsc{SmallWorld}, \textsc{ScaleFree}, and \textsc{Delaunay} graph
families.  We use the released $N\in\{8,16\}$ instances, with three graphs for
each size--topology pair, yielding
\begin{equation}
5\ \text{tasks}\times3\ \text{families}\times2\ \text{sizes}
\times3\ \text{graphs}=90
\end{equation}
matched settings per method.  Released $N=4$ graphs are used only to validate
schema legality, task-action mapping, and propagation of the priority encoding;
they are excluded from LoRA training, \scd transition collection, loss-weight
selection, checkpoint selection, and performance-driven prompt or threshold
tuning.

We follow the communication horizons specified in AgentsNet.  Coloring,
Matching, and VertexCover use five synchronous rounds at $N=8$ and six at
$N=16$; LeaderElection and Consensus use $2D+1$ rounds for graph diameter $D$.
Each node accumulates non-private information only through incident messages,
as in the benchmark's randomized-LOCAL-inspired protocol.

\paragraph{Episode-local randomized anonymity.}
At the start of each episode, node $i$ receives
\begin{equation}
u_i\sim\operatorname{Uniform}\!\left(\{0,\ldots,2^{64}-1\}\right),
\end{equation}
with collisions resampled so that priorities are unique within that episode.
Each node initially sees only its own $u_i$.  Priorities are fixed within an
episode, resampled between episodes, never exposed as a roster, and available
to another node only after ordinary incident-local propagation.  All methods
share the same graph and priority realization in each matched setting.  The
task contract states only that smaller priorities may serve as episode-local
tie breakers.  The runtime validates and executes the controller's legal task
action; it neither propagates priorities automatically nor selects a color,
partner, coordinator, leader, or consensus value.  Priority therefore provides
private randomness rather than persistent identity or an assigned role.

The same contract supports stable retreat from adjacent coloring conflicts,
partner contention in matching, and redundant coordinators in vertex cover;
LeaderElection and Consensus may propagate the minimum observed priority as a
common anchor.  These uses are controller decisions rather than runtime rules.

\paragraph{Frozen transfer and comparisons.}
All model parameters, validator behavior, and decoding parameters are frozen
before any AgentsNet instance is accessed.  No AgentsNet teacher
record, rollout transition, or evaluation answer contributes to training or
checkpoint selection, and the training-only \scd heads are absent at
deployment.  A deterministic adapter exposes the task contract, legal
task-action vocabulary, opaque incident handles, received traces, a
remaining-round bin, and the focal node's private priority through the
canonical local interface.  It maps validated decisions to the official answer
format without choosing the solution.

We compare frozen \method with two shared-controller variants: Qwen3-4B
zero-shot applies the unadapted base checkpoint to the same local prompt and
decision schema without distillation, while \textsc{Decision only} removes
orbit and field supervision.  The interface-matched controls are \nativemsg,
which sends transient task-native messages without a persistent field;
\fixedswarm, which applies fixed support, inhibition, quorum, and decay; and
\nocomm.  Every method shares graphs, priorities, contracts, round budgets,
legal actions, and final evaluators.  The primary endpoint is official strict
solved fraction, macro-averaged across tasks.  We also report task-wise scores
using the official soft evaluators and receiver-delivered messages per agent.  Paired uncertainty resamples
matched task--graph settings while stratifying by graph size and topology.

In Table~\ref{tab:agentsnet-transfer}, scale retention is the $N{=}16$ to
$N{=}8$ strict-solved ratio. Avg. rank averages ranks over five task scores and
two size-specific strict-solved endpoints, with equal weights and averaged
ties. Bold/underline repeat the color ranking; traffic ranks exclude \nocomm.

\begin{table}[t]
\centering
\caption{LeaderElection with and without episode-local priority; all else fixed.}
\label{tab:agentsnet-priority}
\small
\setlength{\tabcolsep}{4pt}
\begin{tabular}{lccc}
\toprule
\rowcolor{tableheaderbg}
Variant & $N=8$ & $N=16$ & Overall \\
\midrule
\method without priority & 0.61 & 0.39 & 0.50 \\
\method with priority & \bestcell{0.89} & \bestcell{0.71} & \bestcell{0.80} \\
\bottomrule
\end{tabular}
\end{table}

\paragraph{Results.}
Table~\ref{tab:agentsnet-transfer} shows that frozen \method leads on four of
five task-specific soft scores and on the strict Overall endpoint; \fixedswarm
is strongest only on VertexCover.  Within the shared-controller family,
overall solved fraction rises from 0.48 zero-shot to 0.54 after decision
distillation and 0.61 with complete \scd; \fixedswarm reaches 0.55.  The
corresponding matched effects are approximately
$+0.07$ (95\% CI $[+0.02,+0.12]$) against Decision only and $+0.06$
(95\% CI $[+0.01,+0.11]$) against \fixedswarm.  The \method--Decision gap is positive at both scales and
widens from 0.05 at $N=8$ to 0.09 at $N=16$.  Relative to \nativemsg, \method reduces
delivered communication by $1-11.1/28.4=60.9\%$; \nocomm confirms that private
priority alone does not solve these graph-level tasks.

Table~\ref{tab:agentsnet-priority} isolates the randomized-anonymity contract
on the task that explicitly requires symmetry breaking.  Episode-local
priority raises LeaderElection from 0.50 to 0.80 overall, without changing the
learned controller or runtime.  Consistent with
Corollary~\ref{cor:symmetry-preservation}, the gain concentrates on
symmetry-obstructed settings ($+0.49$) and is much smaller when the initial
attributed graph already distinguishes a candidate ($+0.06$;
Figure~\ref{fig:agentsnet-symmetry-priority}).  Taken together, the experiment
shows that one frozen local law
transfers through a deterministic adapter to a new graph distribution and
action vocabulary, with the advantage of complete \scd persisting at the
larger graph size.

\begin{figure}[H]
\centering
\resizebox{0.27\linewidth}{!}{\input{figures/agentsnet_symmetry_priority.tex}}
\caption{\textbf{LeaderElection with and without episode-local priority,
stratified by initial symmetry.}  Symmetry-obstructed settings have no
singleton candidate orbit after priority removal and opaque-name quotienting
(10/18); the remaining eight are already asymmetric.  Bars show official soft
score.}
\label{fig:agentsnet-symmetry-priority}
\end{figure}
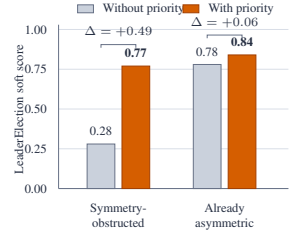

\FloatBarrier

\section{GPT-4o Private-Worker Transfer}
\label{app:g5-api-worker}

Table~\ref{tab:g5-api-worker-transfer} compares the quality--compute operating
points of the comparison methods under a common full-refresh
schedule.  The comparison keeps the private
worker and refresh schedule fixed across methods, isolating the operating
point of coordination rather than a schedule change.  A separate within-\method
comparison in Table~\ref{tab:g5-api-worker-reuse} removes 41.1--43.4\% of
logical GPT-4o input through private-state reuse while preserving or improving
mean quality.

\begin{table}[h]
\centering
\caption{GPT-4o full-refresh transfer at $N=32$: mean native quality
$\uparrow$ / logical input kTokens $\downarrow$ / delivered messages per agent
$\downarrow$.}
\label{tab:g5-api-worker-transfer}
\scriptsize
\setlength{\tabcolsep}{3.2pt}
\begin{tabular}{lccc}
\toprule
\rowcolor{tableheaderbg}
Method & Pursuit & \silo II-16 & \silo III-26 \\
\midrule
\textbf{\method (full refresh)}
& 4.33 / 1867.0 / 22.77
& 1.000 / 406.5 / 7.62
& 0.794 / 742.6 / 7.12 \\
\nativemsg
& 1.40 / 1856.3 / 167.88
& 1.000 / 210.3 / 9.00
& 1.000 / 411.8 / 9.00 \\
\fixedswarm
& 4.60 / 1860.0 / 29.36
& 1.000 / 210.3 / 9.00
& 1.000 / 411.8 / 9.00 \\
\gdesigner
& 3.40 / 1720.0 / 10.50
& 1.000 / 285.0 / 6.10
& 0.900 / 505.0 / 5.90 \\
\dmoa
& 3.65 / 1480.0 / 8.90
& 1.000 / 235.0 / 5.40
& 0.950 / 425.0 / 5.20 \\
\agentnet
& 3.00 / 1859.9 / 8.38
& --- & --- \\
\swarmsys
& --- & 0.388 / 136.1 / 9.00
& 0.475 / 268.5 / 9.00 \\
\bottomrule
\end{tabular}
\end{table}

\begin{table}[h]
\centering
\caption{Private-state reuse within \method at $N=32$: mean native quality /
logical input kTokens / delivered messages per agent.}
\label{tab:g5-api-worker-reuse}
\scriptsize
\setlength{\tabcolsep}{4pt}
\begin{tabular}{lccc}
\toprule
\rowcolor{tableheaderbg}
Task & Full refresh & Matched reuse & Input reduction \\
\midrule
\silo II-16 & 1.000 / 406.5 / 7.62 & 1.000 / \textbf{239.3} / 7.58 & 41.1\% \\
\silo III-26 & 0.794 / 742.6 / 7.12 & 0.957 / \textbf{420.6} / 7.27 & 43.4\% \\
\bottomrule
\end{tabular}
\end{table}

Figure~\ref{fig:g5-api-worker-tradeoff} shows the full-refresh reference
trade-off.  Marker area encodes receiver-delivered messages per agent, so a
point farther up and left with a smaller marker is preferable.  Logical input
tokens count every executed private-worker request independently of cache use.

\begin{figure}[h]
\centering
\includegraphics[width=\linewidth]{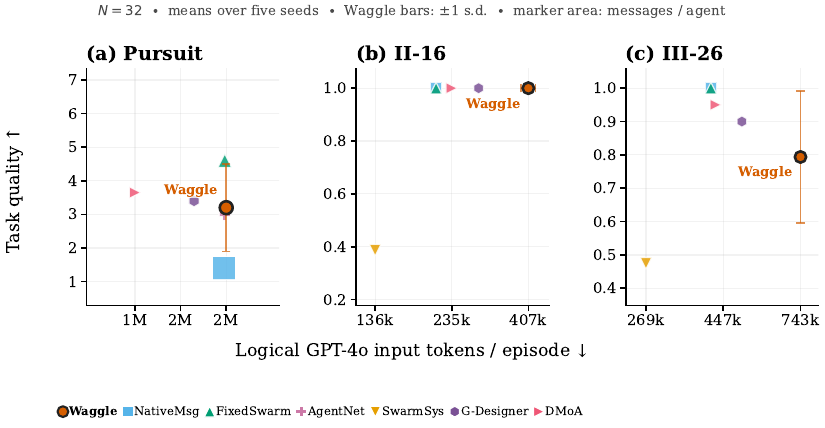}
\caption{\textbf{Full-refresh GPT-4o quality--logical-input trade-off at
$N=32$.}  Means use five matched seeds; marker area encodes delivered messages
per agent, and \method error bars show $\pm1$ sample s.d.}
\label{fig:g5-api-worker-tradeoff}
\end{figure}

\section{Focused Paired Evidence}
\label{app:g6}

\subsection{Matched private-state reuse}

The matched schedule retains the initial private solve and every
communication-conditioned packet update, but reuses each agent's latest
private state for redundant controller refreshes.  The same-seed comparison is
reported in Table~\ref{tab:g5-api-worker-reuse}; it separates this
within-\method schedule change from the common full-refresh baseline
comparison in Table~\ref{tab:g5-api-worker-transfer}.

\subsection{Eight-seed paired evaluation}

The focused paired analysis uses eight held-out seeds for the principal
contrasts.  The complete eleven-method five-seed breadth matrix remains in
Table~\ref{tab:g3-cross-substrate-core};
Figure~\ref{fig:g6-focused} concentrates the added evidence on population
behavior and private-compute efficiency.

\begin{figure}[h]
\centering
\includegraphics[width=\linewidth]{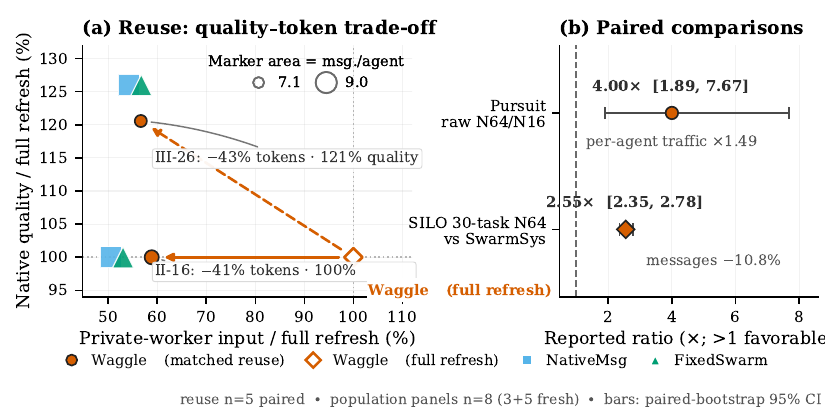}
\caption{\textbf{Focused paired evidence.}  (a) Matched private-state reuse
versus full refresh; marker area encodes delivered messages per agent.  (b)
Eight-seed paired ratios with paired-bootstrap 95\% intervals.}
\label{fig:g6-focused}
\end{figure}

\section{Joint Law versus Substrate-Specific Specialists}
\label{app:joint-specialists}

All controllers in this comparison start from Qwen3-4B-Instruct-2507 and use
the same rank-32 LoRA architecture, target modules, dropout, GPT-4o teacher,
anonymous local prompt, canonical decision schema, complete \scd objective,
optimizer, update count, checkpoint-selection rule, greedy decoding,
validator, and fixed benchmark runtime.  Thus the comparison changes the
allocation of decision-labeled local views across substrates, rather than the
controller interface or the runtime that realizes records.

Formal evaluation fixes $N=64$ and uses eight held-out seeds, yielding 16
matched \swarmbench Pursuit/Foraging task--seed units (two tasks $\times$
eight seeds) and 240
matched \silo task--seed units per deployment.
For the deliberately favorable oracle comparison, we pre-specify a practical
tolerance of $-0.030$ for the Joint-minus-oracle \silo success difference and
a minimum retention ratio of $0.950$ for the Joint-to-oracle \swarmbench score.
The paired \silo interval uses 10,000 bootstrap resamples over the matched
task--seed units; the \swarmbench contrast is reported descriptively because it
contains only 16 units.

Table~\ref{tab:joint-specialist-main} reports the scores, retention ratios, and
deployment counts summarized in Figure~\ref{fig:horizon-budget-sweep}(d).

\begin{table}[!htbp]
\centering
\caption{Joint and substrate-specific \scd deployments at $N=64$; retention is
relative to the oracle pair.}
\label{tab:joint-specialist-main}
\footnotesize
\setlength{\tabcolsep}{7.0pt}
\renewcommand{\arraystretch}{1.10}
\resizebox{0.90\textwidth}{!}{%
\begin{tabular}{lccc}
\toprule
\rowcolor{tableheaderbg}
Deployment
& \shortstack{\swarmbench P/F score /\\oracle retention $\uparrow$}
& \shortstack{\silo SR /\\oracle retention $\uparrow$}
& \shortstack{Deployed LoRA\\adapters $\downarrow$} \\
\midrule
\textbf{Joint \method}
& $\mathbf{1.133 / 96.8\%}$ & $\mathbf{0.682 / 97.6\%}$ & $\mathbf{1}$ \\
Exposure-matched specialist pair
& $1.118 / 95.5\%$ & $0.674 / 96.4\%$ & $2$ \\
\rowcolor{tablerefbg}
Oracle specialist pair (6,912 per specialist)
& $1.171 / 100.0\%$ & $0.699 / 100.0\%$ & $2$ \\
\bottomrule
\end{tabular}}

\end{table}

Exposure-matched specialists receive the joint controller's in-domain views;
oracle specialists each receive the full 6,912-view joint budget.

\begin{table}[!htbp]
\centering
\caption{Training-data allocation for joint and specialist controllers.}
\label{tab:joint-specialist-training}
\scriptsize
\setlength{\tabcolsep}{5.2pt}
\renewcommand{\arraystretch}{1.10}
\resizebox{\textwidth}{!}{%
\begin{tabular}{lccccc}
\toprule
Training regime
& \shortstack{Swarm unique\\views}
& \shortstack{Silo unique\\views}
& \shortstack{Aggregate unique\\views}
& \shortstack{Optimizer\\updates}
& \shortstack{Deployed LoRA\\adapters} \\
\midrule
Joint \method
& $3{,}456$ & $3{,}456$ & $6{,}912$ & Matched & $1$ \\
Exposure-matched specialist pair
& $3{,}456$ & $3{,}456$ & $6{,}912$ & Matched$^{\dagger}$ & $2$ \\
Oracle specialist pair (6,912 per specialist)
& $6{,}912$ & $6{,}912$ & $13{,}824$ & Matched & $2$ \\
\bottomrule
\end{tabular}}

\vspace{2pt}
\parbox{0.98\textwidth}{\footnotesize\raggedright\textit{Note.}
$^{\dagger}$ Each exposure-matched corpus is sampled with replacement to
match the joint controller's optimizer updates; no additional unique examples
are introduced.  Every \silo in-domain corpus uses task-stratified sampling.}
\end{table}

\begin{table}[!htbp]
\centering
\caption{Joint-minus-specialist effects at $N=64$; the paired 95\% interval
applies only to the \silo contrast.}
\label{tab:joint-specialist-effects}
\scriptsize
\setlength{\tabcolsep}{4.0pt}
\resizebox{0.94\textwidth}{!}{%
\begin{tabular}{lcccl}
\toprule
Comparator
& \shortstack{$\Delta$ \swarmbench P/F\\normalized score}
& \shortstack{$\Delta$ \silo\\SR}
& \shortstack{\silo paired\\95\% CI}
& Reading \\
\midrule
Exposure-matched specialist pair
& $+0.015$
& $+0.008$
& ---
& Joint is higher on both \\
Oracle specialist pair
& $-0.038$
& $-0.017$
& $[-0.029,-0.005]$
& $96.8\%$ / $97.6\%$ retained \\
\bottomrule
\end{tabular}}
\end{table}

\begin{table}[!htbp]
\centering
\caption{Late-conflict reorganization for joint and specialist controllers
(40 matched settings).}
\label{tab:joint-specialist-recovery}
\scriptsize
\setlength{\tabcolsep}{2.8pt}
\resizebox{0.86\linewidth}{!}{%
\begin{tabular}{lcccc}
\toprule
Deployment
& \shortstack{Recovery\\$\uparrow$}
& \shortstack{Re-synthesis\\latency $\downarrow$}
& \shortstack{Relative stale\\AUC $\downarrow$}
& \shortstack{Post-conflict\\messages/agent} \\
\midrule
\textbf{Joint \method} & $36/40$ & $4.5$ & $1.00\times$ & $8.2$ \\
Exposure-matched specialist pair & $35/40$ & $4.8$ & $1.05\times$ & $8.3$ \\
Oracle specialist pair & $37/40$ & $4.3$ & $0.96\times$ & $8.1$ \\
\bottomrule
\end{tabular}}

\vspace{2pt}
\parbox{0.86\linewidth}{\footnotesize\raggedright\textit{Note.} Relative
stale-field AUC is normalized by Joint \method.}
\end{table}
\FloatBarrier

\section{Same-Runtime Decision-Law Ablation}
\label{app:same-runtime}

\subsection{Decision-source comparison under a shared runtime}
\label{app:same-runtime-comparison}

For the controller ablation, the complete fixed runtime is
\[
\mathcal{R}=\{\textsc{View},\textsc{Validator},\textsc{Readiness},
\textsc{Conflict},\textsc{Transport},\textsc{Expiry},\textsc{Compiler}\}.
\]
Every decision source receives the same anonymous local view and native private
proposal, emits at most one candidate semantic record per update, and is
subject to the same admissibility checks, duplicate suppression, expiry, and
native compilation.  Learned variants use deterministic decoding with at most
one validator-code-conditioned regeneration.  The runtime may reject or
realize a record, but never originates a claim, commitment, challenge, or
communication intent.

The fixed heuristic has access to every runtime-supported operation, but uses
no learned context-dependent rule: it challenges and revokes an incompatible
commitment after verified conflict; otherwise it commits to a ready
highest-supported candidate, relays the admissible trace with highest support,
deposits fresh private evidence, or takes the native private proposal in that
order.  Zero-shot Qwen3-4B receives the identical local prompt and decision
schema without distillation.  Private-only suppresses every social record;
the random control samples only from locally admissible records.  GPT-4o sees
the same one-view interface and serves solely as a local teacher reference.

\begin{table}[!htbp]
\centering
\caption{Decision-source ablation under an identical interface and runtime;
only the semantic-record source changes.}
\label{tab:same-runtime-full}
\scriptsize
\setlength{\tabcolsep}{3.6pt}
\renewcommand{\arraystretch}{1.10}
\resizebox{\textwidth}{!}{%
\begin{tabular}{lccccc}
\toprule
Decision source
& \shortstack{\silo\\success $\uparrow$}
& \shortstack{\silo\\messages/agent $\downarrow$}
& \shortstack{\swarmbench P/F $N=64$\\normalized score $\uparrow$}
& \shortstack{Late-conflict\\recovery $\uparrow$}
& \shortstack{Re-synthesis\\latency $\downarrow$} \\
\midrule
GPT-4o local teacher
& $0.736\;[0.70,0.77]$ & $8.45$ & $1.18\;[1.08,1.28]$
& $38/40\;(95.0\%)$ & $4.0$ \\
\textbf{\method}
& $0.682\;[0.65,0.72]$ & $8.03$ & $1.13\;[1.05,1.22]$
& $36/40\;(90.0\%)$ & $4.5$ \\
Qwen3-4B zero-shot
& $0.392\;[0.33,0.45]$ & $8.56$ & $0.97\;[0.87,1.06]$
& $23/40\;(57.5\%)$ & $8.5$ \\
Fixed local heuristic
& $0.276\;[0.22,0.34]$ & $8.82$ & $0.90\;[0.80,1.00]$
& $14/40\;(35.0\%)$ & $10.5$ \\
Private-only
& $0.112\;[0.08,0.15]$ & $0.00$ & $0.94\;[0.84,1.03]$
& $1/40\;(2.5\%)$ & --- \\
Random legal record
& $0.093\;[0.06,0.12]$ & $6.00$ & $0.68\;[0.55,0.80]$
& $1/40\;(2.5\%)$ & --- \\
\bottomrule
\end{tabular}}
\vspace{2pt}
\parbox{0.98\textwidth}{\footnotesize\raggedright\textit{Note.} Brackets are
paired-bootstrap 95\% confidence intervals; late-conflict metrics use 40
matched settings, and latency is computed on recovered settings.}
\end{table}

\begin{table}[!htbp]
\centering
\caption{Late-counterevidence outcomes for alternative decision sources under
the same runtime; process metrics are computed on recovered settings.}
\label{tab:same-runtime-mechanism}
\scriptsize
\setlength{\tabcolsep}{2.8pt}
\resizebox{0.86\linewidth}{!}{%
\begin{tabular}{lcccc}
\toprule
Decision source
& \shortstack{Recovery\\$\uparrow$}
& \shortstack{Re-synthesis\\latency $\downarrow$}
& \shortstack{Relative stale\\field AUC $\downarrow$}
& \shortstack{Post-conflict\\messages/agent} \\
\midrule
GPT-4o local teacher & $38/40$ & $4.0$ & $0.90\times$ & $8.5$ \\
\textbf{\method} & $36/40$ & $4.5$ & $1.00\times$ & $8.2$ \\
Qwen3-4B zero-shot & $23/40$ & $8.5$ & $1.70\times$ & $8.8$ \\
Fixed local heuristic & $14/40$ & $10.5$ & $2.30\times$ & $8.9$ \\
Private-only & $1/40$ & --- & --- & $0.0$ \\
\bottomrule
\end{tabular}}

\vspace{2pt}
\parbox{0.86\linewidth}{\footnotesize\raggedright\textit{Note.} Relative
stale-field AUC is normalized by \method.  The random legal-record control is
omitted because it reaches the endpoint in only 1/40 settings, leaving the
trajectory summaries undefined.}
\end{table}
\FloatBarrier

The endpoint differences in Tables~\ref{tab:same-runtime-full}
and~\ref{tab:same-runtime-mechanism} coincide with a
distinct post-conflict trajectory: \method clears stale fields and
re-synthesizes a replacement commitment in roughly half the latency of
zero-shot or heuristic control.  This source-level comparison complements the
component interventions in Figure~\ref{fig:g4-mechanism-attribution}
and Appendix~\ref{app:g4-mechanism-extension}: the former holds all
runtime machinery fixed, whereas the latter removes individual parts of the
reorganization loop.

\section{\silo Scaling by Difficulty}
\label{app:silo-results}

\begin{table}[h]
\centering
\caption{\method success (\%) by population and \silo difficulty level.}
\label{tab:silo-waggle-scale-level}
\small
\setlength{\tabcolsep}{8pt}
\begin{tabular}{ccccc}
\toprule
Population & Level I & Level II & Level III & Overall \\
\midrule
$N=16$ & 94.5 & 87.8 & 76.6 & 86.3 \\
$N=32$ & 91.8 & 80.7 & 64.5 & 79.0 \\
$N=64$ & 87.6 & 69.9 & 17.1 & 58.2 \\
\bottomrule
\end{tabular}
\end{table}

Table~\ref{tab:g3-cross-substrate-core} gives the cross-method comparison at
$N=64$; this table tracks how \method's difficulty profile changes with
population size.

\section{Held-Out Controller Results}
\label{app:controller-results}

\begin{table}[h]
\centering
\caption{Per-mode classification on 800 held-out local views.}
\label{tab:per-mode}
\small
\begin{tabular}{lrrr}
\toprule
Derived mode & Support & Predicted & F1 (\%) \\
\midrule
Explore    & 200 & 200 & 100.00 \\
Deposit    & 100 & 100 & 100.00 \\
Relay      & 100 & 100 & 100.00 \\
Challenge  & 200 & 200 & 100.00 \\
Synthesize & 100 & 100 & 100.00 \\
Abstain    & 100 & 100 & 100.00 \\
\bottomrule
\end{tabular}
\end{table}

All 800 predictions are valid JSON, canonical under the contract, schema-valid,
and executable.  There are no invalid predictions and no safe-fallback
decisions.  Full-record exact match is 598/800 (74.75\%).

\section{Mechanism Measurements}
\label{app:mechanisms}

\subsection{Runtime-component intervention}
\label{app:runtime-component-intervention}

The formal trace stores both controller intent and executed transport.  We
report the following quantities without assigning persistent roles:

\begin{itemize}
    \item \textbf{Recruitment events}: executed Deposit or Relay writes carrying
    novelty/support, plus spatial actions guided by a model-selected relay.
    \item \textbf{Readiness-to-commit events}: transitions from no commitment to a
    controller-selected commitment after candidate readiness.
    \item \textbf{Commitment reopenings}: Challenge decisions that revoke an
    incompatible local commitment after verified counterevidence.
    \item \textbf{Decay}: reduction from initial to final trace mass and the
    number of expiry advances.
    \item \textbf{Explore actions}: validated decisions assigned the Explore
    mode in Table~\ref{tab:modes}, which execute a benchmark-native private
    action without a social write or commitment.
\end{itemize}

On \swarmbench, the spatial transport layer additionally records candidate,
executed, and suppressed writes by mode; duplicate-fingerprint, missing local
trace, exhausted-lifetime, and zero-mass suppressions are
separated.  These counters distinguish a learned communication decision from
a transport invariant and verify that communication savings do not come from
invalid generations.

\begin{table}[h]
\centering
\caption{Runtime-component interventions over 40 matched late-conflict
settings.  The complete runtime recovers 36/40; $\Delta$Rec. is \method minus
each ablation.}
\label{tab:g4-paired-effects}
\footnotesize
\setlength{\tabcolsep}{3.5pt}
\begin{tabularx}{\linewidth}{lccX}
\toprule
Variant & Recovery & $\Delta$Rec. (pp) & Pooled \method signature \\
\midrule
\textsc{NoConflict} & 0/40 & +90.0 & No explicit recovery in any group \\
\textsc{NoDecay} & 33/40 & +7.5 & 3.89 rounds faster; 86.0\% less stale AUC \\
\textsc{Irrev. Quorum} & 0/40 & +90.0 & No explicit recovery in any group \\
\bottomrule
\end{tabularx}
\end{table}

\subsection{Representative late-conflict trajectory}
\label{app:g3-late-conflict}

Figure~\ref{fig:g3-late-conflict} makes the population-level
$A\!\rightarrow\!\varnothing\!\rightarrow\!B$ loop visible.  Before the
intervention, $A$ dominates.  At relative round zero, a controlled target
relocation injects counterevidence through native egocentric views and the
committed share falls, exposing an opening phase.  Challenge peaks at $+1$;
$B$ first exceeds $A$ and Synthesize reactivates at $+2$.  Field mass peaks at
$+3$ and combined Challenge/Synthesize activity at $+4$--$+5$, after which both
actions subside while $B$ remains dominant and total commitment rebuilds.

\begin{figure}[H]
\centering
\includegraphics[width=\linewidth]{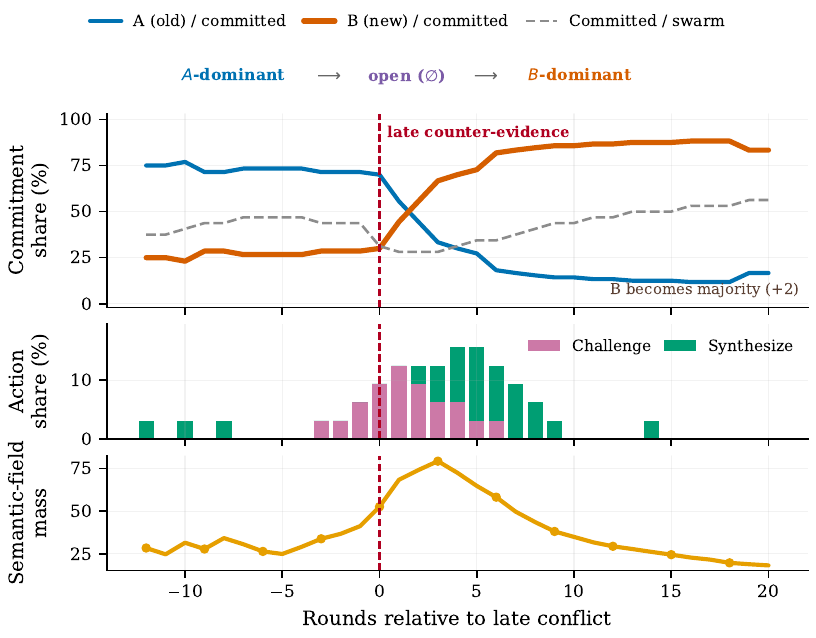}
\caption{\textbf{Representative late-conflict reorganization in \swarmbench
Pursuit ($N=32$).}  A target relocation at round 0 enters only through native
local views; panels show commitment shares, Challenge/Synthesize activity, and
semantic-field mass.}
\label{fig:g3-late-conflict}
\end{figure}
\FloatBarrier

\subsection{Runtime-component task extension}
\label{app:g4-mechanism-extension}

\begin{figure}[H]
\centering
\includegraphics[width=0.94\linewidth]{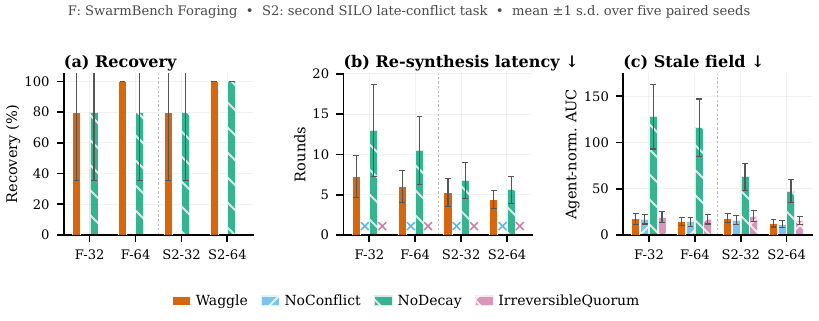}
\caption{\textbf{Runtime-component extension.}  \swarmbench Foraging (F) and
a second \silo late-conflict task (S2) are evaluated at
$N\in\{32,64\}$ over 20 matched settings; latency and stale-field AUC show mean
$\pm1$ s.d.}
\label{fig:g4-mechanism-extension}
\end{figure}
\FloatBarrier

\subsection{Runtime-Parameter Sensitivity under a Frozen Controller}
\label{app:runtime-sensitivity}

We characterize the spatial runtime around its default operating point while
keeping the \method checkpoint, decoder, validator, transport, compiler,
benchmark instances, and evaluation seeds fixed.  All primary results use the
default configuration; the sensitivity analysis varies one factor at a time.
We vary the support threshold
$\tau_s\in\{6,\mathbf{8},10\}$, the coherence threshold
$\tau_c\in\{0.5,\mathbf{0.6},0.7\}$, and the Pursuit refresh interval
$R_{\mathrm{ref}}\in\{8,\mathbf{12},16\}$, holding the two-channel and
zero-conflict requirements fixed.  The sweep is specific to the numerical
\swarmbench readiness gates; \silo readiness follows native constructibility.

Each of the seven configurations is evaluated on Pursuit and Foraging at
$N\in\{32,64\}$ with five paired held-out seeds, and on the Pursuit
late-counterevidence protocol at both populations with the same five seeds.
This gives $7\times(20+10)=210$ frozen-controller runs.  Conventional task
quality and receiver-delivered messages are normalized within each
task--population cell by the default configuration before macro-averaging.

\begin{figure}[H]
\centering
\includegraphics[width=0.98\linewidth]{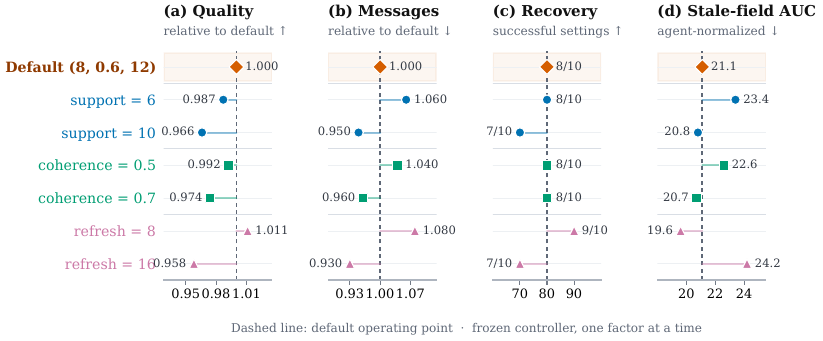}
\caption{\textbf{Runtime-parameter sensitivity with a frozen controller.}
Seven one-factor-at-a-time configurations vary the support threshold,
coherence threshold, or refresh interval.  Quality and messages are normalized
to the default within each task--population cell; dashed lines mark the
default.  Recovery and stale-field AUC use ten matched Pursuit settings.}
\label{fig:runtime-sensitivity}
\end{figure}

Across the sweep, normalized task quality remains between 0.958 and 1.011, and
recovery counts range from 7/10 to 9/10.  Readiness-gate perturbations change
traffic by at most 6\%.  The faster-refresh setting records 9/10 recovery and
lower stale-field AUC at higher traffic; the slower-refresh setting records
7/10 recovery and higher stale-field AUC at lower traffic.  Nearby settings
preserve performance while shifting the communication--responsiveness
trade-off.
\FloatBarrier

\subsection{Delays, message loss, and temporary churn}
\label{app:execution-perturbations}

We reuse the late-counterevidence matrix of \swarmbench Pursuit and
\silo III-26 at $N\in\{32,64\}$, with ten matched held-out cases per
task--population cell (40 settings).  The controller and task instances remain
fixed across regimes.  For each delayed message, the additional delay is
sampled from $\mathrm{DiscreteUniform}\{0,1\}$ or
$\mathrm{DiscreteUniform}\{0,1,2\}$.  One delay unit is the interval between
consecutive nominal local updates in the synchronous reference execution;
zero means delivery at the normally available next update.  Loss regimes
drop messages at the specified rate.  For temporary churn, 10\% of active
nodes are selected uniformly at counterevidence injection and suspended for
two nominal update intervals.  Their task actions, controller calls, and
outbound communication pause; incoming messages are dropped, not buffered.
Nodes then rejoin with their private state and local commitment preserved.

Recovery follows Appendix~\ref{app:scd-primary-factorial}.
Re-synthesis latency counts completed agent-local updates from counterevidence
injection to the first valid synthesis of a replacement commitment, divided
by the reference population $N$.  One round-equivalent thus corresponds to
$N$ completed local updates; this measures update count, not API wall-clock
latency.  As in the runtime-component analysis, latency is summarized over
recovered settings.  Table~\ref{tab:execution-perturbations} reports task quality
and stale-field AUC relative to the synchronous reference, and distinguishes
attempted from receiver-delivered messages per agent.

\begin{table}[H]
\centering
\caption{Execution perturbations over 40 matched late-counterevidence settings.
Quality and stale AUC are relative to the synchronous reference; latency is
in round-equivalents.  Message columns report attempted/delivered counts per agent.}
\label{tab:execution-perturbations}
\fontsize{8}{9.2}\selectfont
\setlength{\tabcolsep}{3.2pt}
\renewcommand{\arraystretch}{1.10}
\begin{tabular}{@{}lrrrrr@{}}
\toprule
\rowcolor{tableheaderbg}
Execution regime & \shortstack{Quality /\\sync $\uparrow$}
& Recovery $\uparrow$ & \shortstack{Re-synthesis latency\\(round-equiv.) $\downarrow$}
& \shortstack{Relative\\stale AUC $\downarrow$}
& \shortstack{Msg./agent\\attempted/delivered} \\
\midrule
Synchronous & 1.000 & 36/40 (90.0\%) & 4.5 & $1.00\times$ & 8.2 / 8.2 \\
Delay $\{0,1\}$ & 0.986 & 35/40 (87.5\%) & 5.0 & $1.07\times$ & 8.3 / 8.3 \\
Delay $\{0,1,2\}$ & 0.969 & 34/40 (85.0\%) & 5.5 & $1.15\times$ & 8.4 / 8.4 \\
10\% message drop & 0.951 & 33/40 (82.5\%) & 5.8 & $1.20\times$ & 8.4 / 7.6 \\
20\% message drop & 0.905 & 30/40 (75.0\%) & 6.7 & $1.34\times$ & 8.5 / 6.8 \\
10\% temporary churn & 0.925 & 32/40 (80.0\%) & 6.1 & $1.27\times$ & 8.6 / 8.1 \\
Delay $\{0,1,2\}$ + 10\% drop & 0.895 & 31/40 (77.5\%) & 6.5 & $1.33\times$ & 8.6 / 7.8 \\
\bottomrule
\end{tabular}
\end{table}

Delays and moderate loss preserve most synchronous task quality, while
recovery slows and stale-field burden rises.  Stronger loss and combined
perturbations degrade performance further.  These tests characterize finite
delivery delays and temporary unavailability; they do not establish robustness
to arbitrary schedules, permanent information loss, or adversarial agents.
\FloatBarrier

\section{Limitations and Deployment Boundaries}
\label{app:responsible}

\paragraph{Interface and substrate scope.}
\method assumes a deterministic adapter that exposes bounded private task
state, locally checkable incident evidence, and legal task actions through the
shared anonymous interface (Section~\ref{sec:method};
Appendix~\ref{app:interface}).  We evaluate this contract on spatial and
distributed-language substrates, through frozen transfer to graph-structured
tasks with a new action vocabulary, and under population extrapolation beyond
the primary matrix (Tables~\ref{tab:g3-cross-substrate-core}
and~\ref{tab:agentsnet-transfer};
Appendix~\ref{app:n128-population-extension}).  These results support reuse
across the tested local-interaction settings; environments in which the
observation interface or transport mechanism must itself be learned remain
outside the present scope.

\paragraph{Anonymity and identity-dependent coordination.}
The present interface excludes persistent identity-indexed reputation and
assigned global roles, so it does not maintain reliability histories for named
agents across episodes.  Task-relevant capability differences could instead
enter through bounded local descriptors while arbitrary identities remain
hidden (Appendix~\ref{app:future-directions}).

\paragraph{Learned semantics and fixed runtime.}
\method learns state-dependent task, trace, and commitment decisions, while
the runtime supplies validation, bounded transport, expiry, and native
realization (Section~\ref{sec:method}).  This leaves low-level coordination
affordances in the deployment contract rather than learning them end to end.
Identical-runtime decision-source comparisons, controller-fixed component
interventions, and nearby spatial-runtime sweeps isolate the two sides of this
boundary (Section~\ref{sec:same-runtime-results};
Appendices~\ref{app:same-runtime-comparison},
\ref{app:runtime-component-intervention}, and~\ref{app:runtime-sensitivity}).
Across the instrumented evaluation matrix, 99.75\% of first decodes are
executable as emitted
and the no-communication fallback is never invoked.  Remaining cases require
envelope normalization, at most one validator-code-conditioned controller
regeneration, or conservative projection; wrapper operations cannot originate
social content
(Appendix~\ref{app:held-out-decisions}).  Thus the runtime constrains and
realizes controller decisions rather than supplying their social semantics.

\paragraph{Conditional symmetry guarantee.}
Proposition~\ref{prop:trajectory-equivariance} states the synchronous case;
Appendix~\ref{app:event-driven-equivariance} extends the structural guarantee
to paired event-driven executions, with a distributional result under an
equivariant marked-event kernel.  Both require equivariant view construction,
local decisions, and runtime operations, including completion-time admission
in the asynchronous contract.  \scd promotes the local premise but does not
enforce it architecturally.  We measure the learned approximation through
held-out summary-, decoded-record-, and admitted-record agreement and the
symmetry-breaking diagnostic (Table~\ref{tab:scd-objective-attribution};
Figure~\ref{fig:agentsnet-symmetry-priority}).
Our primary benchmark executions and next-field training targets use
synchronous updates.  Appendix~\ref{app:execution-perturbations} measures
degradation under finite delivery delays, message loss, and temporary churn.
More general scheduling and membership changes can reduce evidence freshness,
alter temporal reachability, or remove task-relevant information, affecting
both progress and feasibility beyond these tested regimes.
Structural relabeling consistency is distinct from equal task performance
across execution regimes, schedule independence, or task convergence.

\paragraph{Supervision and transfer.}
The law is trained from teacher-labeled single local views and rollout
transitions on the training substrates, so its coverage is bounded by those
local situations (Section~\ref{sec:learning}; Appendix~\ref{app:training}).
Training collection and formal evaluation use disjoint cases and seeds; the
factorial arms share teacher records and optimization budgets, and the same
structural and behavioral pattern replicates with a second student backbone
(Appendices~\ref{app:scd-primary-factorial}
and~\ref{app:llama-backbone-replication}).  In \agentsnetrp, the controller is
frozen before any benchmark instance is accessed, and no AgentsNet teacher
record, rollout transition, or evaluation answer enters training or model
selection (Appendix~\ref{app:agentsnet-rp}).  Continual adaptation in
open-ended domains is not evaluated here.

\paragraph{Open-world deployment.}
The evaluation assumes faithful benchmark observations and non-adversarial
agents.  Local admission checks record legality and consistency with available
evidence, not the truthfulness of arbitrary external observations.
In particular, the spatial readiness gate counts incident channels rather than
independent evidence origins (Appendix~\ref{app:protocol}).  Duplicate
suppression and finite trace lifetime limit repetition and persistence;
they do not authenticate sources or establish Byzantine fault tolerance.
In open tool environments, untrusted observations could imitate a
verified-conflict marker or evade canonicalization.  Deployments with
consequential external actions should authenticate evidence provenance,
sandbox tools, and retain appropriate human approval.  The experiments
contain no human-subject data.

\section{Future Directions}
\label{app:future-directions}

\paragraph{Adaptive interfaces and broader substrates.}
The local interface in \method separates substrate mechanics from the law that
generates organization (Section~\ref{sec:method};
Appendix~\ref{app:interface}).  Cross-substrate evaluation and frozen
\agentsnetrp transfer provide a starting point for extending this separation
to multimodal observations, tool-using agents, embodied interaction,
partially specified action spaces, and environments whose local interface
changes over time (Tables~\ref{tab:g3-cross-substrate-core}
and~\ref{tab:agentsnet-transfer}).  Parts of the adapter could also be learned
or continually refined while retaining provenance, locality, boundedness, and
legal actuation.  The central design question is how much interface plasticity
can be introduced without relocating organizational semantics from the shared
law into the substrate layer.

\paragraph{Asynchronous and changing populations.}
Appendix~\ref{app:event-driven-equivariance} establishes conditional
relabeling consistency for event-driven execution, and
Appendix~\ref{app:execution-perturbations} tests bounded delays, message loss,
and temporary churn.  Further evaluation could examine irregular clocks,
longer delays, ongoing membership turnover, and asynchronous model-call
completion.  It should distinguish matched call budgets from matched time
deadlines and account for whether departing nodes remove information needed
to solve the task.

\paragraph{Capability-conditioned local laws.}
Parameter sharing constrains the coordination rule rather than requiring
identical task-relevant capabilities.  A heterogeneous extension could expose
bounded local descriptors of available tools, action constraints, and context
or compute budgets while retaining one shared controller and the joint decision
schema.  The adapter would validate actions against the focal node's
capabilities.  Budget-dependent truncation and summarization would remain local
and satisfy the same equivariance requirements, including insensitivity to
arbitrary handle ordering.  Capability attributes would travel with their nodes
under relabeling, while information about other nodes would arrive through
incident communication.  Under equivariant view construction, the complete
deployed map, and runtime operations, the local-to-global argument
(Proposition~\ref{prop:trajectory-equivariance};
Appendix~\ref{app:event-driven-equivariance}) applies to this augmented state.
Anonymity removes dependence on arbitrary names while preserving genuine
capability differences as a source of task-relevant asymmetry, without assigning
organizational roles.  This extension retains a shared coordination policy;
mixtures of independently parameterized controllers require separate analysis.
Generalization of the current checkpoint to unseen capability mixtures remains
to be evaluated.

\paragraph{More autonomous and continual learning.}
\scd currently combines teacher-labeled local decisions with
rollout-grounded transition supervision (Section~\ref{sec:learning};
Appendix~\ref{app:training}).  Disjoint evaluation and frozen target-substrate
transfer establish a clean basis for reducing external supervision further
(Appendices~\ref{app:decision-details}
and~\ref{app:agentsnet-rp}).  Candidate signals include self-supervised local
prediction, preference or outcome feedback, self-play, offline or online
reinforcement learning, and locally generated counterfactual trajectories.
Continual updates introduce a sharper problem: new local experience should
improve the law without erasing anonymous-view consistency or destabilizing
previously useful coordination reactions.  Local replay, constrained policy
updates, and decentralized parameter adaptation offer possible routes to
combining plasticity with the invariances promoted by \scd.

\paragraph{Learned transport and multi-timescale coordination state.}
The current runtime fixes validity, transport, expiry, and native realization
so that semantic decisions can be attributed to the policy.  Controller-fixed
component interventions and runtime sensitivity sweeps provide an empirical
base for constrained co-design of these mechanics
(Section~\ref{sec:same-runtime-results};
Appendices~\ref{app:runtime-component-intervention}
and~\ref{app:runtime-sensitivity}).  Future systems could learn when to
allocate communication bandwidth, how long locally useful information should
persist, which traces to compress or forget, and how readiness thresholds
should respond to local conditions.  Longer tasks may also benefit from
several timescales of bounded coordination state.  Such adaptation should
preserve hard local constraints on provenance, bounded communication, legal
actuation, and reversibility while learning the operating policy around them.

\paragraph{Theory beyond exact relabeling equivariance.}
Propositions~\ref{prop:trajectory-equivariance}
and~\ref{prop:event-driven-equivariance} provide a structural basis for a
quantitative theory of learned local laws.  Approximate extensions could
bound finite-horizon trajectory deviation under uniform local error controls
over reachable histories.  Beyond symmetry, open questions include
conditions for convergence or recurrent organization under asynchronous
updates, communication complexity, recovery time after counterevidence,
bounds on stale coordination mass, and the relation between local memory
horizon and population-scale stability.  Progress and convergence results
would additionally need conditions on event scheduling, information
reachability, and the local reaction dynamics.  These questions connect
learned LLM swarms to population protocols, distributed algorithms, and
stochastic interacting systems.

\paragraph{Robustness, trust, and partially aligned participants.}
The bounded interface makes trust assumptions explicit and creates a natural
place to attach stronger safeguards.  Open systems may contain corrupted
observations, misleading messages, prompt injection, compromised tools, or
participants whose objectives are only partially aligned.  Research
directions include authenticated provenance, robust local evidence
aggregation, privacy-preserving coordination, and trust mechanisms that do
not depend on stable global identities.  Formal fault models could measure
graceful degradation under Byzantine, unavailable, or colluding nodes, while
mixed human--agent and mixed-incentive populations would test whether local
reversibility remains effective when disagreement is strategic rather than
incidental.  These extensions build on the deployment boundary stated in
Appendix~\ref{app:responsible}.

\paragraph{Longer-horizon evaluation and measurement of emergent organization.}
The current evaluation separates population and interaction scaling, frozen
transfer, anonymous stability, and reorganization after counterevidence
(Sections~\ref{sec:horizon-results}--\ref{sec:same-runtime-results}).  Future
benchmarks could combine these stresses through long-lived objectives,
multiple competing commitments, changing goals, node churn, heterogeneous
capabilities, and repeated environmental shocks.  Richer diagnostics could
measure formation time, turnover of coordination state, complete recovery
trajectories, communication concentration, and the diversity and persistence
of emergent organizational patterns.  Causal interventions on local evidence
or topology could further distinguish communication patterns that merely
correlate with success from organization that mediates collective behavior.
Broader backbone families, larger deployments, and system-level measurements
of latency, memory, energy, and model-call cost would complement the
algorithmic evidence reported here.

\noindent\textbf{Across these directions, the
broader question is how far adaptive organization can emerge from repeated
execution of a reusable local law, and when, if ever, explicit system-level
structure provides additional value.}

\end{document}

%% file: math_commands.tex
\usepackage{amsmath,amsfonts,bm}

\def\eqref#1{equation~\ref{#1}}

\def\1{\bm{1}}

\DeclareMathAlphabet{\mathsfit}{\encodingdefault}{\sfdefault}{m}{sl}
\SetMathAlphabet{\mathsfit}{bold}{\encodingdefault}{\sfdefault}{bx}{n}



%% file: figures/agentsnet_symmetry_priority.tex
\begin{tikzpicture}[x=1.55cm,y=3.10cm,font=\scriptsize]
  \definecolor{prioritygrid}{HTML}{D9DEE7}
  \definecolor{prioritygray}{HTML}{CAD1DC}
  \definecolor{priorityedge}{HTML}{667085}
  \definecolor{priorityorange}{HTML}{D55E00}
  \definecolor{priorityorangeedge}{HTML}{A33F00}
  \definecolor{prioritytext}{HTML}{273043}

  \foreach \y/\label in {0/0.00,0.25/0.25,0.5/0.50,0.75/0.75,1/1.00}{
    \draw[prioritygrid,line width=0.35pt] (0.32,\y) -- (3.22,\y);
    \node[anchor=east,text=prioritytext] at (0.25,\y) {\label};
  }
  \draw[priorityedge,line width=0.55pt] (0.32,0) -- (3.22,0);
  \draw[priorityedge,line width=0.55pt] (0.32,0) -- (0.32,1.03);

  \filldraw[fill=prioritygray,draw=priorityedge,line width=0.55pt]
    (0.68,0) rectangle (1.03,0.28);
  \filldraw[fill=priorityorange,draw=priorityorangeedge,line width=0.55pt]
    (1.12,0) rectangle (1.47,0.77);
  \node[anchor=south,text=prioritytext] at (0.855,0.295) {0.28};
  \node[anchor=south,text=prioritytext,font=\scriptsize\bfseries]
    at (1.295,0.785) {0.77};
  \draw[priorityedge,line width=0.45pt]
    (0.855,0.885) -- (0.855,0.905) -- (1.295,0.905) -- (1.295,0.885);
  \node[anchor=south,text=prioritytext] at (1.075,0.915) {$\Delta=+0.49$};

  \filldraw[fill=prioritygray,draw=priorityedge,line width=0.55pt]
    (2.02,0) rectangle (2.37,0.78);
  \filldraw[fill=priorityorange,draw=priorityorangeedge,line width=0.55pt]
    (2.46,0) rectangle (2.81,0.84);
  \node[anchor=south,text=prioritytext] at (2.195,0.795) {0.78};
  \node[anchor=south,text=prioritytext,font=\scriptsize\bfseries]
    at (2.635,0.855) {0.84};
  \draw[priorityedge,line width=0.45pt]
    (2.195,0.945) -- (2.195,0.965) -- (2.635,0.965) -- (2.635,0.945);
  \node[anchor=south,text=prioritytext] at (2.415,0.975) {$\Delta=+0.06$};

  \node[align=center,anchor=north,text=prioritytext] at (1.075,-0.055)
    {Symmetry-\\obstructed};
  \node[align=center,anchor=north,text=prioritytext] at (2.415,-0.055)
    {Already\\asymmetric};
  \node[rotate=90,text=prioritytext] at (-0.18,0.50)
    {LeaderElection soft score};

  \filldraw[fill=prioritygray,draw=priorityedge,line width=0.5pt]
    (0.56,1.105) rectangle (0.76,1.155);
  \node[anchor=west,text=prioritytext] at (0.80,1.13) {Without priority};
  \filldraw[fill=priorityorange,draw=priorityorangeedge,line width=0.5pt]
    (1.89,1.105) rectangle (2.09,1.155);
  \node[anchor=west,text=prioritytext] at (2.13,1.13) {With priority};
\end{tikzpicture}

%% file: main.bbl
\begin{thebibliography}{47}
\providecommand{\natexlab}[1]{#1}
\providecommand{\url}[1]{\texttt{#1}}
\expandafter\ifx\csname urlstyle\endcsname\relax
  \providecommand{\doi}[1]{doi: #1}\else
  \providecommand{\doi}{doi: \begingroup \urlstyle{rm}\Url}\fi

\bibitem[Angluin et~al.(2006)Angluin, Aspnes, Diamadi, Fischer, and
  Peralta]{angluin2006computation}
Dana Angluin, James Aspnes, Zo{\"e} Diamadi, Michael~J. Fischer, and Ren{\'e}
  Peralta.
\newblock Computation in networks of passively mobile finite-state sensors.
\newblock \emph{Distributed Computing}, 18\penalty0 (4):\penalty0 235--253,
  2006.
\newblock \doi{10.1007/s00446-005-0138-3}.
\newblock URL \url{https://doi.org/10.1007/s00446-005-0138-3}.

\bibitem[Bonabeau et~al.(1999)Bonabeau, Dorigo, and
  Theraulaz]{bonabeau1999swarm}
Eric Bonabeau, Marco Dorigo, and Guy Theraulaz.
\newblock \emph{Swarm Intelligence: From Natural to Artificial Systems}.
\newblock Oxford University Press, 1999.
\newblock \doi{10.1093/oso/9780195131581.001.0001}.
\newblock URL \url{https://doi.org/10.1093/oso/9780195131581.001.0001}.

\bibitem[Chen et~al.(2024{\natexlab{a}})Chen, Saha, Stengel-Eskin, and
  Bansal]{chen2024magdi}
Justin Chen, Swarnadeep Saha, Elias Stengel-Eskin, and Mohit Bansal.
\newblock {MAGD}i: Structured distillation of multi-agent interaction graphs
  improves reasoning in smaller language models.
\newblock In \emph{Proceedings of the 41st International Conference on Machine
  Learning}, volume 235 of \emph{Proceedings of Machine Learning Research},
  pp.\  7220--7235. PMLR, 2024{\natexlab{a}}.
\newblock URL \url{https://proceedings.mlr.press/v235/chen24ah.html}.

\bibitem[Chen et~al.(2024{\natexlab{b}})Chen, Su, Zuo, Yang, Yuan, Chan, Yu,
  Lu, Hung, Qian, Qin, Cong, Xie, Liu, Sun, and Zhou]{chen2024agentverse}
Weize Chen, Yusheng Su, Jingwei Zuo, Cheng Yang, Chenfei Yuan, Chi-Min Chan,
  Heyang Yu, Yaxi Lu, Yi-Hsin Hung, Chen Qian, Yujia Qin, Xin Cong, Ruobing
  Xie, Zhiyuan Liu, Maosong Sun, and Jie Zhou.
\newblock {AgentVerse}: Facilitating multi-agent collaboration and exploring
  emergent behaviors.
\newblock In \emph{International Conference on Learning Representations},
  2024{\natexlab{b}}.
\newblock URL
  \url{https://proceedings.iclr.cc/paper_files/paper/2024/hash/578e65cdee35d00c708d4c64bce32971-Abstract-Conference.html}.

\bibitem[Das et~al.(2019)Das, Gervet, Romoff, Batra, Parikh, Rabbat, and
  Pineau]{das2019tarmac}
Abhishek Das, Th{\'e}ophile Gervet, Joshua Romoff, Dhruv Batra, Devi Parikh,
  Mike Rabbat, and Joelle Pineau.
\newblock {T}ar{MAC}: Targeted multi-agent communication.
\newblock In \emph{Proceedings of the 36th International Conference on Machine
  Learning}, volume~97 of \emph{Proceedings of Machine Learning Research}, pp.\
   1538--1546. PMLR, 2019.
\newblock URL \url{https://proceedings.mlr.press/v97/das19a.html}.

\bibitem[Dorigo \& Gambardella(1997)Dorigo and Gambardella]{dorigo1997ant}
Marco Dorigo and Luca~Maria Gambardella.
\newblock Ant colony system: A cooperative learning approach to the traveling
  salesman problem.
\newblock \emph{IEEE Transactions on Evolutionary Computation}, 1\penalty0
  (1):\penalty0 53--66, 1997.
\newblock \doi{10.1109/4235.585892}.
\newblock URL \url{https://doi.org/10.1109/4235.585892}.

\bibitem[Foerster et~al.(2016)Foerster, Assael, de~Freitas, and
  Whiteson]{foerster2016dial}
Jakob Foerster, Ioannis~Alexandros Assael, Nando de~Freitas, and Shimon
  Whiteson.
\newblock Learning to communicate with deep multi-agent reinforcement learning.
\newblock In \emph{Advances in Neural Information Processing Systems},
  volume~29, 2016.
\newblock URL
  \url{https://proceedings.neurips.cc/paper/2016/hash/c7635bfd99248a2cdef8249ef7bfbef4-Abstract.html}.

\bibitem[Grattafiori et~al.(2024)Grattafiori, Dubey, Jauhri, Pandey, Kadian,
  et~al.]{grattafiori2024llama3}
Aaron Grattafiori, Abhimanyu Dubey, Abhinav Jauhri, Abhinav Pandey, Abhishek
  Kadian, et~al.
\newblock The {Llama 3} herd of models.
\newblock \emph{arXiv preprint arXiv:2407.21783}, 2024.
\newblock \doi{10.48550/arXiv.2407.21783}.
\newblock URL \url{https://arxiv.org/abs/2407.21783}.

\bibitem[Gr{\"o}tschla et~al.(2025)Gr{\"o}tschla, M{\"u}ller, T{\"o}nshoff,
  Galkin, and Perozzi]{grotschla2025agentsnet}
Florian Gr{\"o}tschla, Luis M{\"u}ller, Jan T{\"o}nshoff, Mikhail Galkin, and
  Bryan Perozzi.
\newblock {AgentsNet}: Coordination and collaborative reasoning in multi-agent
  {LLM}s.
\newblock \emph{arXiv preprint arXiv:2507.08616}, 2025.
\newblock URL \url{https://arxiv.org/abs/2507.08616}.

\bibitem[Hong et~al.(2024)Hong, Zhuge, Chen, Zheng, Cheng, Zhang, Wang, Wang,
  Yau, Lin, Zhou, Ran, Xiao, Wu, and Schmidhuber]{hong2024metagpt}
Sirui Hong, Mingchen Zhuge, Jiaqi Chen, Xiawu Zheng, Yuheng Cheng, Ceyao Zhang,
  Jinlin Wang, Zili Wang, Steven Ka~Shing Yau, Zijuan Lin, Liyang Zhou, Chenyu
  Ran, Lingfeng Xiao, Chenglin Wu, and J{\"u}rgen Schmidhuber.
\newblock {MetaGPT}: Meta programming for a multi-agent collaborative
  framework.
\newblock In \emph{International Conference on Learning Representations}, 2024.
\newblock URL \url{https://arxiv.org/abs/2308.00352}.

\bibitem[Hounie et~al.(2023)Hounie, Chamon, and Ribeiro]{hounie2023automatic}
Ignacio Hounie, Luiz F.~O. Chamon, and Alejandro Ribeiro.
\newblock Automatic data augmentation via invariance-constrained learning.
\newblock In \emph{Proceedings of the 40th International Conference on Machine
  Learning}, volume 202 of \emph{Proceedings of Machine Learning Research},
  pp.\  13410--13433. PMLR, 2023.
\newblock URL \url{https://proceedings.mlr.press/v202/hounie23a.html}.

\bibitem[Hu et~al.(2022)Hu, Shen, Wallis, Allen-Zhu, Li, Wang, Wang, and
  Chen]{hu2022lora}
Edward~J. Hu, Yelong Shen, Phillip Wallis, Zeyuan Allen-Zhu, Yuanzhi Li, Shean
  Wang, Lu~Wang, and Weizhu Chen.
\newblock {LoRA}: Low-rank adaptation of large language models.
\newblock In \emph{International Conference on Learning Representations}, 2022.
\newblock URL \url{https://arxiv.org/abs/2106.09685}.

\bibitem[H{\"u}ttenrauch et~al.(2019)H{\"u}ttenrauch, {\v{S}}o{\v{s}}i{\'c},
  and Neumann]{huettenrauch2019deep}
Maximilian H{\"u}ttenrauch, Adrian {\v{S}}o{\v{s}}i{\'c}, and Gerhard Neumann.
\newblock Deep reinforcement learning for swarm systems.
\newblock \emph{Journal of Machine Learning Research}, 20\penalty0
  (54):\penalty0 1--31, 2019.
\newblock URL \url{https://jmlr.org/papers/v20/18-476.html}.

\bibitem[Ji et~al.(2026)Ji, Chen, Chen, Zhu, Xu, Gro{\ss}, Zhou, Cao, and
  Zhao]{ji2026genswarm}
Wenkang Ji, Huaben Chen, Mingyang Chen, Guobin Zhu, Lufeng Xu, Roderich
  Gro{\ss}, Rui Zhou, Ming Cao, and Shiyu Zhao.
\newblock {GenSwarm}: Scalable multi-robot code-policy generation and
  deployment via language models.
\newblock \emph{npj Robotics}, 4\penalty0 (1):\penalty0 5, 2026.
\newblock \doi{10.1038/s44182-025-00065-w}.
\newblock URL \url{https://doi.org/10.1038/s44182-025-00065-w}.

\bibitem[Kang et~al.(2025)Kang, Jeong, Lee, Cho, and Hwang]{kang2025distilling}
Minki Kang, Jongwon Jeong, Seanie Lee, Jaewoong Cho, and Sung~Ju Hwang.
\newblock Distilling {LLM} agent into small models with retrieval and code
  tools.
\newblock In \emph{Advances in Neural Information Processing Systems},
  volume~38, pp.\  106501--106538. Curran Associates, Inc., 2025.
\newblock \doi{10.52202/085713-3553}.
\newblock URL
  \url{https://proceedings.neurips.cc/paper_files/paper/2025/hash/99263f9bf46874e2b8b7f104b5063864-Abstract-Conference.html}.

\bibitem[Li et~al.(2023)Li, Hammoud, Itani, Khizbullin, and
  Ghanem]{li2023camel}
Guohao Li, Hasan Abed Al~Kader Hammoud, Hani Itani, Dmitrii Khizbullin, and
  Bernard Ghanem.
\newblock {CAMEL}: Communicative agents for ``mind'' exploration of large
  language model society.
\newblock In \emph{Advances in Neural Information Processing Systems}, 2023.
\newblock URL \url{https://arxiv.org/abs/2303.17760}.

\bibitem[Li et~al.(2026{\natexlab{a}})Li, Zhang, Bo, Dai, Li, Wen, and
  Chen]{li2026raps}
Rui Li, Zeyu Zhang, Xiaohe Bo, Quanyu Dai, Chaozhuo Li, Feng Wen, and Xu~Chen.
\newblock Towards adaptive, scalable, and robust coordination of {LLM} agents:
  A dynamic ad-hoc networking perspective.
\newblock \emph{arXiv preprint arXiv:2602.08009}, 2026{\natexlab{a}}.
\newblock URL \url{https://arxiv.org/abs/2602.08009}.

\bibitem[Li et~al.(2025)Li, Liu, Zhao, Li, Li, Jiang, Xu, Zhao, Fan, and
  Liang]{li2025swarmsys}
Ruohao Li, Hongjun Liu, Leyi Zhao, Zisu Li, Jiawei Li, Jiajun Jiang, Linning
  Xu, Chen Zhao, Mingming Fan, and Chen Liang.
\newblock {SwarmSys}: Decentralized swarm-inspired agents for scalable and
  adaptive reasoning.
\newblock \emph{arXiv preprint arXiv:2510.10047}, 2025.
\newblock URL \url{https://arxiv.org/abs/2510.10047}.

\bibitem[Li et~al.(2026{\natexlab{b}})Li, Liu, Wen, Zhang, and
  Pan]{li2026assemblecrew}
Shiyuan Li, Yixin Liu, Qingsong Wen, Chengqi Zhang, and Shirui Pan.
\newblock {Assemble Your Crew}: Automatic multi-agent communication topology
  design via autoregressive graph generation.
\newblock In \emph{Proceedings of the AAAI Conference on Artificial
  Intelligence}, volume~40, pp.\  23142--23150, 2026{\natexlab{b}}.
\newblock \doi{10.1609/aaai.v40i28.39481}.
\newblock URL \url{https://doi.org/10.1609/aaai.v40i28.39481}.

\bibitem[Liu et~al.(2024)Liu, Zhang, Li, Liu, and Yang]{liu2024dynamic}
Zijun Liu, Yanzhe Zhang, Peng Li, Yang Liu, and Diyi Yang.
\newblock A dynamic {LLM}-powered agent network for task-oriented agent
  collaboration.
\newblock In \emph{First Conference on Language Modeling}, 2024.
\newblock URL \url{https://openreview.net/forum?id=XII0Wp1XA9}.

\bibitem[Luo et~al.(2026)Luo, Jin, Yu, Zhang, Kumar, Li, Xu, Chen, and
  Wang]{luo2026agentark}
Yinyi Luo, Yiqiao Jin, Weichen Yu, Mengqi Zhang, Srijan Kumar, Xiaoxiao Li,
  Weijie Xu, Xin Chen, and Jindong Wang.
\newblock {AgentArk}: Distilling multi-agent intelligence into a single {LLM}
  agent.
\newblock \emph{arXiv preprint arXiv:2602.03955}, 2026.
\newblock \doi{10.48550/arXiv.2602.03955}.
\newblock URL \url{https://arxiv.org/abs/2602.03955}.

\bibitem[{OpenAI}(2024)]{openai2024swarm}
{OpenAI}.
\newblock {Swarm}: Educational framework for lightweight multi-agent
  orchestration.
\newblock GitHub repository, 2024.
\newblock URL \url{https://github.com/openai/swarm}.

\bibitem[Park et~al.(2022)Park, Biza, Zhao, Van De~Meent, and
  Walters]{park2022symmetric}
Jung~Yeon Park, Ondrej Biza, Linfeng Zhao, Jan-Willem Van De~Meent, and Robin
  Walters.
\newblock Learning symmetric embeddings for equivariant world models.
\newblock In \emph{Proceedings of the 39th International Conference on Machine
  Learning}, volume 162 of \emph{Proceedings of Machine Learning Research},
  pp.\  17372--17389. PMLR, 2022.
\newblock URL \url{https://proceedings.mlr.press/v162/park22a.html}.

\bibitem[Qian et~al.(2024)Qian, Liu, Liu, Chen, Dang, Li, Yang, Chen, Su, Cong,
  Xu, Li, Liu, and Sun]{qian2024chatdev}
Chen Qian, Wei Liu, Hongzhang Liu, Nuo Chen, Yufan Dang, Jiahao Li, Cheng Yang,
  Weize Chen, Yusheng Su, Xin Cong, Juyuan Xu, Dahai Li, Zhiyuan Liu, and
  Maosong Sun.
\newblock {ChatDev}: Communicative agents for software development.
\newblock In \emph{Proceedings of the 62nd Annual Meeting of the Association
  for Computational Linguistics}, 2024.
\newblock URL \url{https://arxiv.org/abs/2307.07924}.

\bibitem[Ravanbakhsh et~al.(2017)Ravanbakhsh, Schneider, and
  P{\'o}czos]{ravanbakhsh2017equivariance}
Siamak Ravanbakhsh, Jeff Schneider, and Barnab{\'a}s P{\'o}czos.
\newblock Equivariance through parameter-sharing.
\newblock In \emph{Proceedings of the 34th International Conference on Machine
  Learning}, volume~70 of \emph{Proceedings of Machine Learning Research}, pp.\
   2892--2901. PMLR, 2017.
\newblock URL \url{https://proceedings.mlr.press/v70/ravanbakhsh17a.html}.

\bibitem[Ruan et~al.(2025)Ruan, Huang, Wen, and Sun]{ruan2025swarmbench}
Kai Ruan, Mowen Huang, Ji-Rong Wen, and Hao Sun.
\newblock Benchmarking {LLM}s' swarm intelligence.
\newblock \emph{arXiv preprint arXiv:2505.04364}, 2025.
\newblock URL \url{https://arxiv.org/abs/2505.04364}.

\bibitem[Seeley et~al.(2012)Seeley, Visscher, Schlegel, Hogan, Franks, and
  Marshall]{seeley2012stop}
Thomas~D. Seeley, P.~Kirk Visscher, Thomas Schlegel, Patrick~M. Hogan, Nigel~R.
  Franks, and James A.~R. Marshall.
\newblock Stop signals provide cross inhibition in collective decision-making
  by honeybee swarms.
\newblock \emph{Science}, 335\penalty0 (6064):\penalty0 108--111, 2012.
\newblock \doi{10.1126/science.1210361}.
\newblock URL \url{https://doi.org/10.1126/science.1210361}.

\bibitem[Shaw et~al.(2022)Shaw, Wenzel, Walker, and Sartoretti]{shaw2022formic}
Samuel Shaw, Emerson Wenzel, Alexis Walker, and Guillaume Sartoretti.
\newblock {ForMIC}: Foraging via multiagent {RL} with implicit communication.
\newblock \emph{IEEE Robotics and Automation Letters}, 7\penalty0 (2):\penalty0
  4877--4884, 2022.
\newblock \doi{10.1109/LRA.2022.3152688}.
\newblock URL \url{https://doi.org/10.1109/LRA.2022.3152688}.

\bibitem[Singh et~al.(2019)Singh, Jain, and Sukhbaatar]{singh2019ic3net}
Amanpreet Singh, Tushar Jain, and Sainbayar Sukhbaatar.
\newblock Learning when to communicate at scale in multiagent cooperative and
  competitive tasks.
\newblock In \emph{International Conference on Learning Representations}, 2019.
\newblock URL \url{https://openreview.net/forum?id=rye7knCqK7}.

\bibitem[Strobel et~al.(2024)Strobel, Dorigo, and Fritz]{strobel2024llm2swarm}
Volker Strobel, Marco Dorigo, and Mario Fritz.
\newblock {LLM2Swarm}: Robot swarms that responsively reason, plan, and
  collaborate through {LLM}s.
\newblock In \emph{NeurIPS 2024 Workshop on Open-World Agents}, 2024.
\newblock \doi{10.48550/arXiv.2410.11387}.
\newblock URL \url{https://arxiv.org/abs/2410.11387}.

\bibitem[Sukhbaatar et~al.(2016)Sukhbaatar, Szlam, and
  Fergus]{sukhbaatar2016commnet}
Sainbayar Sukhbaatar, Arthur Szlam, and Rob Fergus.
\newblock Learning multiagent communication with backpropagation.
\newblock In \emph{Advances in Neural Information Processing Systems},
  volume~29, 2016.
\newblock URL
  \url{https://proceedings.neurips.cc/paper/2016/hash/55b1927fdafef39c48e5b73b5d61ea60-Abstract.html}.

\bibitem[Theraulaz \& Bonabeau(1999)Theraulaz and
  Bonabeau]{theraulaz1999stigmergy}
Guy Theraulaz and Eric Bonabeau.
\newblock A brief history of stigmergy.
\newblock \emph{Artificial Life}, 5\penalty0 (2):\penalty0 97--116, 1999.
\newblock \doi{10.1162/106454699568700}.
\newblock URL \url{https://doi.org/10.1162/106454699568700}.

\bibitem[Wu et~al.(2023{\natexlab{a}})Wu, Bansal, Zhang, Wu, Li, Zhu, Jiang,
  Zhang, Zhang, Liu, Awadallah, White, Burger, and Wang]{wu2023autogen}
Qingyun Wu, Gagan Bansal, Jieyu Zhang, Yiran Wu, Beibin Li, Erkang Zhu,
  Li~Jiang, Xiaoyun Zhang, Shaokun Zhang, Jiale Liu, Ahmed~Hassan Awadallah,
  Ryen~W. White, Doug Burger, and Chi Wang.
\newblock {AutoGen}: Enabling next-gen {LLM} applications via multi-agent
  conversation.
\newblock \emph{arXiv preprint arXiv:2308.08155}, 2023{\natexlab{a}}.
\newblock URL \url{https://arxiv.org/abs/2308.08155}.

\bibitem[Wu et~al.(2026)Wu, Lu, Yan, Qiu, Hu, Guo, and Yang]{wu2026dmoa}
Xingjian Wu, Junkai Lu, Siyu Yan, Xiangfei Qiu, Jilin Hu, Chenjuan Guo, and Bin
  Yang.
\newblock Differentiable mixture-of-agents incentivizes swarm intelligence of
  large language models.
\newblock \emph{arXiv preprint arXiv:2605.15706}, 2026.
\newblock URL \url{https://arxiv.org/abs/2605.15706}.

\bibitem[Wu et~al.(2023{\natexlab{b}})Wu, Yu, Chen, Hao, and
  Zhuo]{wu2023models}
Zifan Wu, Chao Yu, Chen Chen, Jianye Hao, and Hankz~Hankui Zhuo.
\newblock Models as agents: Optimizing multi-step predictions of interactive
  local models in model-based multi-agent reinforcement learning.
\newblock \emph{arXiv preprint arXiv:2303.17984}, 2023{\natexlab{b}}.
\newblock \doi{10.48550/arXiv.2303.17984}.
\newblock URL \url{https://arxiv.org/abs/2303.17984}.

\bibitem[Yamashita \& Kameda(1996)Yamashita and Kameda]{yamashita1996computing}
Masafumi Yamashita and Tsunehiko Kameda.
\newblock Computing on anonymous networks: Part {I}---characterizing the
  solvable cases.
\newblock \emph{IEEE Transactions on Parallel and Distributed Systems},
  7\penalty0 (1):\penalty0 69--89, 1996.
\newblock \doi{10.1109/71.481599}.
\newblock URL \url{https://doi.org/10.1109/71.481599}.

\bibitem[Yang et~al.(2025{\natexlab{a}})Yang, Li, Yang, Zhang, Hui, Zheng, Yu,
  Gao, Huang, Lv, Zheng, Liu, Zhou, Huang, Hu, Ge, Wei, Lin, Tang, Yang, Tu,
  Zhang, Yang, Yang, Zhou, Zhou, Lin, Dang, Bao, Yang, Yu, Deng, Li, Xue, Li,
  Zhang, Wang, Zhu, Men, Gao, Liu, Luo, Li, Tang, Yin, Ren, Wang, Zhang, Ren,
  Fan, Su, Zhang, Zhang, Wan, Liu, Wang, Cui, Zhang, Zhou, and
  Qiu]{yang2025qwen3}
An~Yang, Anfeng Li, Baosong Yang, Beichen Zhang, Binyuan Hui, Bo~Zheng, Bowen
  Yu, Chang Gao, Chengen Huang, Chenxu Lv, Chujie Zheng, Dayiheng Liu, Fan
  Zhou, Fei Huang, Feng Hu, Hao Ge, Haoran Wei, Huan Lin, Jialong Tang, Jian
  Yang, Jianhong Tu, Jianwei Zhang, Jianxin Yang, Jiaxi Yang, Jing Zhou,
  Jingren Zhou, Junyang Lin, Kai Dang, Keqin Bao, Kexin Yang, Le~Yu, Lianghao
  Deng, Mei Li, Mingfeng Xue, Mingze Li, Pei Zhang, Peng Wang, Qin Zhu, Rui
  Men, Ruize Gao, Shixuan Liu, Shuang Luo, Tianhao Li, Tianyi Tang, Wenbiao
  Yin, Xingzhang Ren, Xinyu Wang, Xinyu Zhang, Xuancheng Ren, Yang Fan, Yang
  Su, Yichang Zhang, Yinger Zhang, Yu~Wan, Yuqiong Liu, Zekun Wang, Zeyu Cui,
  Zhenru Zhang, Zhipeng Zhou, and Zihan Qiu.
\newblock {Qwen3} technical report.
\newblock \emph{arXiv preprint arXiv:2505.09388}, 2025{\natexlab{a}}.
\newblock URL \url{https://arxiv.org/abs/2505.09388}.

\bibitem[Yang et~al.(2025{\natexlab{b}})Yang, Chai, Shao, Song, Qi, Rui, and
  Zhang]{yang2025agentnet}
Yingxuan Yang, Huacan Chai, Shuai Shao, Yuanyi Song, Siyuan Qi, Renting Rui,
  and Weinan Zhang.
\newblock {AgentNet}: Decentralized evolutionary coordination for {LLM}-based
  multi-agent systems.
\newblock In \emph{Advances in Neural Information Processing Systems},
  2025{\natexlab{b}}.
\newblock URL \url{https://arxiv.org/abs/2504.00587}.

\bibitem[Yao et~al.(2023)Yao, Zhao, Yu, Du, Shafran, Narasimhan, and
  Cao]{yao2023react}
Shunyu Yao, Jeffrey Zhao, Dian Yu, Nan Du, Izhak Shafran, Karthik Narasimhan,
  and Yuan Cao.
\newblock {ReAct}: Synergizing reasoning and acting in language models.
\newblock In \emph{International Conference on Learning Representations}, 2023.
\newblock URL \url{https://arxiv.org/abs/2210.03629}.

\bibitem[Zaheer et~al.(2017)Zaheer, Kottur, Ravanbakhsh, P{\'o}czos,
  Salakhutdinov, and Smola]{zaheer2017deep}
Manzil Zaheer, Satwik Kottur, Siamak Ravanbakhsh, Barnab{\'a}s P{\'o}czos,
  Ruslan Salakhutdinov, and Alexander~J. Smola.
\newblock Deep sets.
\newblock In \emph{Advances in Neural Information Processing Systems},
  volume~30, 2017.
\newblock URL
  \url{https://proceedings.neurips.cc/paper/2017/hash/f22e4747da1aa27e363d86d40ff442fe-Abstract.html}.

\bibitem[Zhang et~al.(2025{\natexlab{a}})Zhang, Yue, Li, Yun, Wan, Wang, Cheng,
  Yu, and Chen]{zhang2025agentprune}
Guibin Zhang, Yanwei Yue, Zhixun Li, Sukwon Yun, Guancheng Wan, Kun Wang, Dawei
  Cheng, Jeffrey~Xu Yu, and Tianlong Chen.
\newblock Cut the crap: An economical communication pipeline for {LLM}-based
  multi-agent systems.
\newblock In \emph{International Conference on Learning Representations},
  2025{\natexlab{a}}.
\newblock URL \url{https://arxiv.org/abs/2410.02506}.

\bibitem[Zhang et~al.(2025{\natexlab{b}})Zhang, Yue, Sun, Wan, Yu, Fang, Wang,
  Chen, and Cheng]{zhang2025gdesigner}
Guibin Zhang, Yanwei Yue, Xiangguo Sun, Guancheng Wan, Miao Yu, Junfeng Fang,
  Kun Wang, Tianlong Chen, and Dawei Cheng.
\newblock {G-Designer}: Architecting multi-agent communication topologies via
  graph neural networks.
\newblock In \emph{Proceedings of the 42nd International Conference on Machine
  Learning}, volume 267 of \emph{Proceedings of Machine Learning Research},
  pp.\  76678--76692. PMLR, 2025{\natexlab{b}}.
\newblock URL \url{https://proceedings.mlr.press/v267/zhang25cu.html}.

\bibitem[Zhang et~al.(2025{\natexlab{c}})Zhang, Xiang, Yu, Teng, Chen, Chen,
  Zhuge, Cheng, Hong, Wang, Zheng, Liu, Luo, and Wu]{zhang2025aflow}
Jiayi Zhang, Jinyu Xiang, Zhaoyang Yu, Fengwei Teng, Xiong-Hui Chen, Jiaqi
  Chen, Mingchen Zhuge, Xin Cheng, Sirui Hong, Jinlin Wang, Bingnan Zheng, Bang
  Liu, Yuyu Luo, and Chenglin Wu.
\newblock {AF}low: Automating agentic workflow generation.
\newblock In \emph{The Thirteenth International Conference on Learning
  Representations}, 2025{\natexlab{c}}.
\newblock URL \url{https://openreview.net/forum?id=z5uVAKwmjf}.

\bibitem[Zhang et~al.(2026)Zhang, Liu, Shan, Huang, Yang, Zhu, Cheng, Liu,
  Zeng, Zhang, and Jiang]{zhang2026silobench}
Yuzhe Zhang, Feiran Liu, Yi~Shan, Xinyi Huang, Xin Yang, Yueqi Zhu, Xuxin
  Cheng, Cao Liu, Ke~Zeng, Terry~Jingchen Zhang, and Wenyuan Jiang.
\newblock {Silo-Bench}: A scalable environment for evaluating distributed
  coordination in multi-agent {LLM} systems.
\newblock In \emph{Proceedings of the 64th Annual Meeting of the Association
  for Computational Linguistics (Volume 1: Long Papers)}, pp.\  29379--29398,
  San Diego, California, USA, July 2026. Association for Computational
  Linguistics.
\newblock \doi{10.18653/v1/2026.acl-long.1354}.
\newblock URL \url{https://aclanthology.org/2026.acl-long.1354/}.

\bibitem[Zhou et~al.(2022)Zhou, He, Ma, Berg-Kirkpatrick, and
  Neubig]{zhou2022prompt}
Chunting Zhou, Junxian He, Xuezhe Ma, Taylor Berg-Kirkpatrick, and Graham
  Neubig.
\newblock Prompt consistency for zero-shot task generalization.
\newblock In \emph{Findings of the Association for Computational Linguistics:
  EMNLP 2022}, pp.\  2613--2626. Association for Computational Linguistics,
  2022.
\newblock \doi{10.18653/v1/2022.findings-emnlp.192}.
\newblock URL \url{https://aclanthology.org/2022.findings-emnlp.192/}.

\bibitem[Zhou et~al.(2026)Zhou, Gan, and Cherian]{zhou2026llawco}
Qinhong Zhou, Chuang Gan, and Anoop Cherian.
\newblock {LLawCo}: Learning laws of cooperation for modeling embodied
  multi-agent behavior.
\newblock In \emph{Proceedings of the 43rd International Conference on Machine
  Learning}, 2026.
\newblock URL \url{https://arxiv.org/abs/2606.28182}.

\bibitem[Zhuge et~al.(2024)Zhuge, Wang, Kirsch, Faccio, Khizbullin, and
  Schmidhuber]{zhuge2024language}
Mingchen Zhuge, Wenyi Wang, Louis Kirsch, Francesco Faccio, Dmitrii Khizbullin,
  and J{\"u}rgen Schmidhuber.
\newblock Language agents as optimizable graphs.
\newblock In \emph{Proceedings of the 41st International Conference on Machine
  Learning}, 2024.
\newblock URL \url{https://arxiv.org/abs/2402.16823}.

\end{thebibliography}
